\documentclass[10pt,journal,compsoc]{IEEEtran}
\ifCLASSOPTIONcompsoc
  \usepackage[nocompress]{cite}
\else
  \usepackage{cite}
\fi

\usepackage[pdftex]{graphicx}
\usepackage{tabularx}

\usepackage{amsmath}
\usepackage{amssymb}

\usepackage{algorithmic}
\usepackage{algorithm}
\usepackage{multirow}

\usepackage{array}

\usepackage{subcaption}

\usepackage{latexsym}
\usepackage{xcolor}
\usepackage{comment}
\usepackage{amsthm}
\usepackage{todonotes}

\def\xv{{\mathbf x}}
\def\yv{{\mathbf y}}
\def\zv{{\mathbf z}}
\def\vv{{\mathbf v}}
\def\wv{{\mathbf w}}
\def\fv{{\mathbf f}}
\def\gv{{\mathbf g}}
\def\hv{{\mathbf h}}
\def\dv{{\mathbf d}}

\def\nv{{\mathbf n}}
\def\mv{{\mathbf m}}
\def\qv{{\mathbf q}}
\def\tv{{\mathbf t}}
\def\Jm{{\mathbf J}}
\def\xiv{{\mathbf \xi}}

\def\g{\hat{f}}
\def\xiv{{\mathbf \xi}}
\def\R{{\mathbb R}}
\def\E{{\mathbb E}}
\def\N{{\mathbb N}}
\def\S{{\mathcal S}}
\def\I{{\mathcal I}}

\def\RM{{\mathcal R}}
\def\v1{{\mathbf 1}}
\def\cc{{\mathbf c}}
\DeclareMathOperator{\Rg}{Rg}

\newtheorem{thm}{Theorem}[section]
\newtheorem{proposition}[thm]{Proposition}

\begin{document}

\title{Deterministic Decoupling of Global Features \\ and its Application to Data Analysis} 

\author{Eduardo~Mart\'{i}nez-Enr\'{i}quez, 
        Mar\'ia del Mar~Gonz\'{a}lez, 
        and~Javier~Portilla
\IEEEcompsocitemizethanks{\IEEEcompsocthanksitem E. Mart\'{i}nez-Enr\'{i}quez and
J. Portilla are with the Instituto de \'{O}ptica, CSIC, Spain.  \protect\\
E-mails: eduardo.martinez@io.cfmac.csic.es, javier.portilla@csic.es
\IEEEcompsocthanksitem M. Gonz\'{a}lez is with Dept. of Mathematics, Universidad Aut\'{o}noma de Madrid and ICMAT, Spain.\\
E-mail: mariamar.gonzalezn@uam.es
}
\thanks{This work has been funded by the Spanish Government grants FIS2016-75891-P, PID2020-118071GB-
I00 and PID2020-113596GB-I00. Additionally, Grant RED2018-102650-T funded by MCIN/AEI/ 10.13039/501100011033, and the ``Severo Ochoa Programme for Centers of Excellence in R\&D'' (CEX2019-000904-S).}
}


\IEEEtitleabstractindextext{%
\begin{abstract}
We introduce a method for deterministic decoupling of global features and show its applicability to improve data analysis performance, as well as to
open new venues for feature transfer.
We propose a new formalism that is based on defining transformations on submanifolds, by following trajectories along the features' gradients.
Through these transformations we define a {\em normalization} that, we demonstrate, allows for decoupling differentiable features. 
By applying this to sampling moments, we obtain a quasi-analytic solution for the {\em orthokurtosis}, a normalized version of the kurtosis that is not just decoupled from mean and variance, but also from skewness.
We apply this method in the original data domain and at the output of a filter bank to regression and classification problems based on global descriptors, obtaining a consistent and significant improvement in performance as compared to using classical (non-decoupled) descriptors.
\end{abstract}

}

\maketitle

\IEEEdisplaynontitleabstractindextext

\IEEEpeerreviewmaketitle

\IEEEraisesectionheading{\section{Introduction}\label{sec:introduction}}

\IEEEPARstart{D}{ata} analysis relies on the statistical distribution of the observed samples.
Usually, some {\em features} (i.e., real functions) are applied to extract relevant information from the observed data. In the Machine Learning field there are two basic scenarios for data analysis: (i) the {\em classical} one, where the set of features is chosen ad-hoc;  
(ii) the Deep Learning scenario, which involves automatically learning the features from the input data.
In all cases, when features are used, {\em the observed dependencies among them come both from the statistical behavior of the data and from the features' joint algebraic structure}. 

To illustrate this problem imagine that we analyze vectors representing 1-D signals, extracting some marginal sample moments and the sample auto-correlation. We will find a strong dependency between skewness and kurtosis, and also between consecutive correlation factors, e.g., $\rho(1)$ and $\rho(2)$, {\em even when the input data are i.i.d. samples}.
The reason is that the two mentioned pairs of features, like many others, are {\em algebraically  coupled}. As a consequence, their joint range is not just the outer product of their marginal ranges: some combinations of (independently) valid feature values are incompatible for the same input data. E.g., a skewness of 10 and a kurtosis of 3, or $\rho(1) = 0.9$ and $\rho(2) = 0$. 

Feature coupling, thus, produces spurious redundancy, a sort of {\em feature entanglement}, regardless of (and in addition to) the redundancy derived from input data statistics. It causes difficulties for analysis, processing, and simulation. Having a joint feature range with a very intricate topology (full of ``holes'' and complex boundaries' structure) is an obstacle to interpreting the role of each feature separately from the others. It also complicates the manipulation of the samples, in case we wanted to average feature values~\cite{PS00}, study the effect of modifying the value of a particular feature independently of the others, or simulating data by imposing some values onto their features~\cite{MMSP2020}. An interesting example of these problems appeared in~\cite{Balas2006}, which addressed the problem of simplifying the set of features used in~\cite{PS00} to visually describe texture samples. 

Despite its large negative impact on data analysis and processing, much less effort has been devoted in the literature to study and reverse deterministic feature coupling compared to statistical data modeling. Note that, when algebraic coupling between features exists (as in the examples mentioned above), conventional techniques such as PCA, ICA, or even more advanced non-linear ICA (see, e.g., \cite{Lee-Verleysen}), do not provide the right tools for disentangling the involved joint feature vector structures. Even in the ideal scenario of perfectly modeling all dependencies, a purely statistical approach applied to the observed features would mix up the two sources of redundancy, namely, statistical and algebraic. It is advantageous to address separately these two redundancy sources, as it is usual to apply the same type of features for diverse statistical distributions (even in ANNs, when doing {\em transfer learning}~\cite{PanYang2010}). Therefore, many different real problems on different data distributions using the same features will benefit from their decoupling.

Here we propose a mathematical and algorithmic framework for decoupling a set of given features, in the sense of finding another set of similar functions with their gradients mutually orthogonal everywhere in the domain - and, as a consequence, with their ranges decoupled. We study the mathematical conditions under which that is possible. We also study less favorable scenarios where only a limited and/or approximated decoupling is possible. 
We demonstrate the practical application of the proposed method to several examples of data analysis, namely, statistical regression and textured image classification.
Some of the seminal ideas and applied results presented has been published in three conference proceedings~\cite{Portilla:ICIP:2018,Martinez:ICASSP:2020,MMSP2020}. In this work we provide a solid framework, unifying, extending, and giving the necessary mathematical rigor to our previous results.

As concrete study cases, here we have focused on marginal moments and on the second-order moments at the output of a set of filters.
Marginal moments are widely used in the signal processing and statistics literature, for analysis tasks such as estimation (e.g., the method of moments), detection, regression, classification, etc. ~\cite{Pareto,Skew_original,Skew_distribution,Weibul_original,Weibul_estimation,Bovik:PAMI:1990}, and also for synthesis-by-analysis~\cite{HB95,PN96,PS00}. Typically they are used either implicitly, as empirical marginal histograms, or in their standardized form, and up to fourth order: sample mean, variance, skewness, and kurtosis. 
Whereas, as shown here, the first three standardized moments are already mutually decoupled, that is not the case for the skewness and kurtosis. The problem of the skewness-kurtosis coupling has been pointed out by several authors~\cite{Blest2003,Pearson16,Sharma2013}, but it had not been fully solved.
In this respect, one of the main contributions of this paper is presenting a normalized version of the fourth-order sample moment, the {\em orthokurtosis}, which is not just decoupled from the sample mean and variance, but also from the skewness. This new statistical function is fully consistent with the previous standardized sample moments (mean, variance and skewness),
that result from applying our decoupling technique to the first three raw moments.
In addition, the orthokurtosis calculation has a modest computational cost.
By using this new fourth-order feature, instead of the classical kurtosis, we obtain a dramatic accuracy gain in several regression problems (see Subsection~\ref{subsec:regression}).
Furthermore, higher-than-four order moments have been very rarely used (see exceptions in, e.g., ~\cite{Pearson1894,Kalai2010}) because of their instability and mutual redundancy.
By decoupling the marginal moments (exactly or approximately) here we are able to exploit very high order decoupled moments (up to 10th order) and demonstrate their positive impact for texture classification  (see Subsection~\ref{subsec:marg_mom}). 

Banks of convolutional filters, on the other hand, are a classical tool in signal processing, 
with a huge field of application, including early human vision modelling and image/audio analysis, processing and synthesis. In addition, they have also been incorporated~\cite{Fukushima1980, Lecun1998} into artificial neural networks (ANNs) for signals having spatial dimensions (image, video, 3-D, etc.) with tremendous impact.
In neural science, they have long been used to model the responses at early stages of animal and human visual and auditory systems~\cite{Hubel-Wiesel1959,Daugman1988CompleteD2,Navarro1997ImageRW,Dau1997}.
The latter image/audio representations 
share the feature of being redundant, thus avoiding the artifacts that plague critically-sampled linear transformations (e.g., orthogonal or bi-orthogonal wavelets~\cite{Daubechies1992TenLO}). 
However, redundancy in non-orthogonal linear representations demands paying a high price, namely, the algebraic coupling of undecimated sub-bands (outputs of the filters).
Here we address the problem of deterministically decoupling the second-order moments at the output of a filter bank, with direct application, besides analysis, to transfer~\cite{SIIMS2022} and synthesis~\cite{MMSP2020}. 
In Section~\ref{sec:applications} we 
show how the gradients of the resulting {\em decoupled} features are virtually orthogonal for white noise samples and very close to orthogonal for photographic textured image patches.
Furthermore, we demonstrate how approximately decoupling not just variance, but also higher-order moments, at the output of a bank of filters, results in an important performance boost in texture classification (Subsection~\ref{subsec:dec_TIL}).

This paper is organized as follows. Section~\ref{sec:decoupling} sets the mathematical foundations of the method, that allow, in favorable cases, to transform a given feature by decoupling it from a set of other given features. Section~\ref{sec:NeNs} proposes algorithms (based on the Nested Normalization concept, NeN) to obtain a hierarchically ordered set of mutually decoupled features and to transfer features from one observed sample to another.
Section~\ref{sec:study_cases} addresses in detail two study cases of features for being decoupled, namely marginal moments, and the second-order moments at the output of a filter bank. Section~\ref{sec:local_decorrelation} addresses analytically the local de-correlation effect of feature decoupling,
and why this improves parameter discrimination, in regression problems.   
Section~\ref{sec:applications} is devoted to showing how the proposed method actually decouples the studied features, and its practical impact when it is applied to data analysis (regression and classification). Section~\ref{sec:conclusion} concludes the paper. In addition, some technical and/or very detailed contents have been encapsulated in appendices, for readability and reproducibility sake.

\section{Deterministic Decoupling of Global Features}
\label{sec:decoupling}

In this section we propose a method for, given a set of features ${\cal S}$ and another unrelated feature $g$, finding a transformed feature $\hat g$, 
identical to $g$ on a high dimensional submanifold, that is {\em decoupled} to every feature in ${\cal S}$. 

\subsection{Preliminary Concepts}
\subsubsection{Global Shift-Invariant Features}
In this paper we associate a finite discrete signal made of $N$ samples with a vector  $\xv\in \R^N$, possibly lexico-graphically reordered, if the signal support is a multi-dimensional array. We will term a {\em feature} $f$ of that vector $\xv$ a differentiable real function $f: \overline\Omega \subset \R^N\rightarrow \R$ for a domain $\overline\Omega$.
We define here {\em global feature} a feature that depends on all vector's coefficients.

In this paper we will focus on \emph{shift-invariant} features, a special case of global features\footnote{The only non-global shift-invariant features are the trivial functions $f(\xv) = c$, where $c$ is a real constant.}.
Within them, we will exemplify the application of our method to features of the form:
\begin{equation}
    \label{eq:gf}
    f(\xv) = \frac{1}{N} \sum_{n=1}^{N}{[\mv(\xv)]_n},
\end{equation}
with $\mv: \overline\Omega \rightarrow \R^N$ being a differentiable shift-equivariant (i.e., commuting with shift operations), or shift-invariant, mapping, assuming a shift operation with boundary conditions (e.g., circular) has been defined.
This kind of functions, being averages, play the role of sample statistics, like marginal moments, correlation coefficients, moments at the output of filters, etc.

\subsubsection{Decoupled Features}
We say that two features $f_i$ and $f_j$ are {\em algebraically decoupled} (from now on just {\em decoupled}) on a subset of $\overline\Omega$ 
iff
\begin{equation}
\label{eq:dec_fea}
\nabla f_i(\xv) \cdot \nabla f_j(\xv) = 0, \quad\text{for all } \xv \text{ in that subset}.
\end{equation}
We extend this concept to a set of features ${\cal S} = \{f_j, j = 1\dots M\}$ by terming that the features of a set are decoupled, iff they are mutually decoupled, i.e., iff all possible pairs of features within that set $\{(i,j)\,:\, i,j \in \{1..M\}, i\neq j\}$ are decoupled. Similarly, we say that a feature is decoupled to a (decoupled or not) set of features iff it is decoupled to each of the features in that set.

It is worth pointing out two special cases, namely, when features are trivially decoupled and when they are trivially coupled.
We term \emph{trivially decoupled} features those for which there exists at least one orthogonal basis where they have disjoint supports\footnote{Note that they can not have disjoint supports in the original domain if they are both global.} (e.g. Fourier, orthogonal wavelets, etc.). Here we refer to support, in a given domain, as the subset of vector indices the feature depends on.
On the other extreme, a feature map $\fv$ is \emph{trivially coupled} iff it exists at least one non-degenerate function $F:\R^M\rightarrow\R$ such that $F(\fv(\xv))=0,\forall \xv\in\overline\Omega$.
In this paper, we present methods for decoupling features assuming none of those situations happens (for which decoupling is either unnecessary or impossible, respectively). 

\subsubsection{Normalization map}\label{subsection:normalization-definition}

The construction of a normalization will be key in our decoupling process. 

Let ${\cal S} = \{f_j:\overline \Omega\rightarrow\R, j = 1\dots M\}$ be a set of features, $\Omega$ a subset of $\overline \Omega$,  $\hat\xv_\S(\xv):\Omega\rightarrow\Omega$ be a continuous non-constant mapping, and $\vv^{ref}$ a vector made of $\{v_j^{ref},j=1\dots M\}$  (the {\em reference values}), some jointly compatible given reference values for the features in $\S$.
We say that $\hat\xv_\S(\xv;\vv^{ref})$ is a {\em normalization w.r.t. $\S$} and $\vv^{ref}$ in $\Omega$ iff it holds that 
\begin{itemize}
\item[(i)] $\{f_j(\hat\xv_\S(\xv;\vv^{ref})) = v_j^{ref}\}_{j=1}^{M}$; 
\item[(ii)] if $\{f_j(\xv) = v_j^{ref}\}_{j=1}^{M}$ then $\hat\xv_\S(\xv;\vv^{ref}) = \xv$. 
\end{itemize}
Note that previous conditions imply that every normalization is idempotent:
\begin{equation}\label{idempotent}
\hat\xv_\S(\hat\xv_\S(\xv;\vv^{ref});\vv^{ref}) = \hat\xv_\S(\xv;\vv^{ref}).
\end{equation}

We now set up some notation for this set of reference values which will be useful in our exposition. Let $\fv_\S:\overline\Omega\rightarrow \R^M$ be the vector transformation made of the ordered features in $\S$, $[f_1(\xv)\dots f_M(\xv)]$, and set $\vv^{ref}$ to be an $M$-dimensional vector.
We define a {\em reference manifold} as 
$\RM_\S(\vv^{ref}) = \fv_\S^{-1}(\vv^{ref})$,
i.e., the set of vectors $\xv$ 
such that $\fv_\S(\xv)=\vv^{ref}$ .

For being a valid set of reference values for the features, $\vv^{ref}$ must be made of jointly compatible values of the functions in an algebraic sense, i.e., $\{\xv: \fv_\S(\xv) =\vv^{ref}\}\neq\emptyset$.
In addition, we will assume a \emph{non-degeneracy} hypothesis, under which all feature gradients are linearly independent at every point (this will be condition C1 in Subsection \ref{subsection:invariant-mapping}). This is a stronger condition than the set of features not being trivially coupled, and it implies that the dimension of the reference manifold is {\em everywhere} $N-M$.

\subsection{Motivating example: decoupling two features}\label{normalization-one-feature}

In order to motivate the general algorithm, let us explain the method in the case of two features.

\subsubsection{Gradient systems}
\label{subsection:gradient-systems}

Let us fix one feature $f$, defined in a connected open set $\overline\Omega$.  We study the  trajectories $\xv(t)$ of the initial value problem
\begin{equation}\label{grad-system}
\left\{\begin{aligned}
&\frac{d\xv}{dt}=-\nabla f(\xv),\\
&\xv(0)=\xv_0.
\end{aligned}\right.
\end{equation}
This ODE is known as a gradient system. It is clear that moving along a (non-constant) trajectory in the $+t$ (resp. $-t$) direction will strictly decrease (resp. increase) the value of the function $f$ until it reaches the minimum (resp. maximum) value or stabilize at a critical point of $f$ (see the reference \cite[Section 9.3]{HSD} for a discussion of this type of systems).

Thus, in order to study the set of values that $f$ may take along a trajectory, one needs to impose some constraints on its equilibrium points. The precise set of conditions are given in Appendix~\ref{appendix:grad-systems}, and can be summarized in:
\begin{itemize}
\item[B1.] Maxima and minima are all global extrema, not just local.
\item[B2.] The set made of the basins of attraction of all saddle points, denoted by $\Lambda$, is of lower dimension.
\end{itemize}
The first condition ensures that the trajectory will not stop at a local non-global extreme. The second condition 
guarantees that saddle points are non-degenerate and thus, essentially unstable, so it allows to circumvent them by adding small perturbations on the initial condition $\xv_0$ (see Section~\ref{sec:perturbation}).

We denote the trajectory that passes through the point $\xv_0$, or equivalently, the integral manifold of the system  \eqref{grad-system},  by $\I(\xv_0,f)$. The main property we will need is that all possible $f$ values are reachable from any initial point $\xv_0$ by moving along the gradient. As a consequence, fixed a reference value $v_f^{ref}$ for $f$,  conditions B1-B2 guarantee that, for $\xv\not\in\Lambda$, each trajectory $\I(\xv,f)$ reaches (only once) this value. See Appendix \ref{appendix:grad-systems} for a technical discussion.

\subsubsection{Decoupling via normalization}

Assume that we are given  two features $\{f, g\}$. We would like to replace $g$ by a ``similar'' feature $\hat g$ that is decoupled from $f$ in the sense given by \eqref{eq:dec_fea}, via normalization.

More precisely, fixed one feature $f$, we would like to construct a normalization map $\hat \xv_f(\xv)$ as defined in Section \ref{subsection:normalization-definition}. A possibility is to choose a point in  the trajectory $\I(\xv,f)$ that attains some reference value $v^{ref}$ of the feature $f$. Thus it is natural 
to define the normalization by $\hat{\xv}_f(\xv;v^{ref})$, this is, as the map that sends a point $\xv$ to the point where the trajectory $\I(\xv,f)$ crosses the reference manifold $\RM_f(v^{ref})$, which is unique by our previous discussion on gradient systems.

Now, given another feature $g$, we define $\hat g$ by
\begin{equation}\label{normalize-g}
\hat g(\xv)=g(\hat \xv_f(\xv)).
\end{equation}
Then $f$ and $\hat g$ are decoupled, this is, their gradients are orthogonal. To show this,   apply the chain rule to Eq.~\eqref{normalize-g}, to obtain $\nabla \hat{g}_f(\xv) = \Jm_{\hat{\xv}_f}^T(\xv) \nabla g(\hat \xv_f(\xv))$, where $\Jm_{\hat \xv_f}$ represents the Jacobian matrix of the map $\hat\xv_f$.
 Now, by pre-multiplying both terms by $\nabla f(\xv)^T$, we obtain:
\begin{equation}\label{orthogonal-gradients}
    \nabla f(\xv)\cdot\nabla \hat{g}(\xv) 
     =  (\Jm_{\hat \xv_f}(\xv) \nabla f(\xv))^T \nabla g(\hat \xv_f(\xv)).
\end{equation}
Finally, note that
\begin{equation*}
\Jm_{\hat \xv_f}(\xv) \nabla f(\xv)=\left.\frac{d}{dt}\right|_{t=0} \hat{\xv}_f({\alpha}(t)),
\end{equation*}
where $\alpha(t)$ is the integral curve of \eqref{grad-system} starting at the point $\xv$. By construction, the map $\hat \xv_f$ is constant along this curve and, thus, the above expression vanishes. This shows that the two gradients in expression \eqref{orthogonal-gradients} are orthogonal, as desired. Figure~\ref{fig:manifold} graphically illustrates these concepts for two given features $f_1,f_2$.

\begin{figure*}[t]
\begin{center}
\hspace{-2.5mm}
\includegraphics[width=0.85\textwidth]{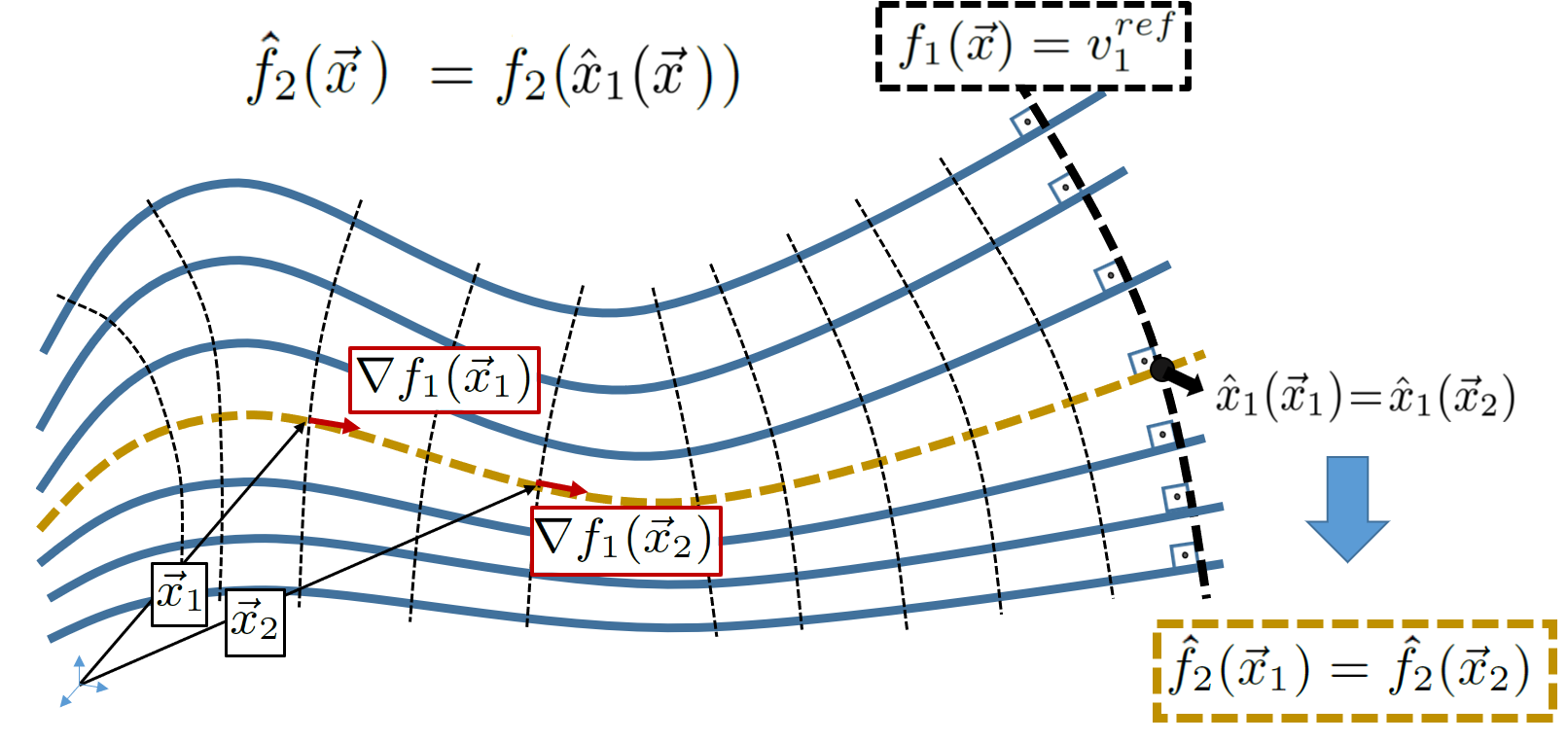} 
\end{center}
\vspace{-4mm}
\caption{Decoupling two features through normalization.  The proposed normalization consists in finding the intersection of the invariance submanifolds passing by $\xv$ (parallel curved thick lines) with the reference manifold $\RM_1 = \{ \xv : f_1(\xv) = v_1^{ref}\}$ (black dashed thick line). All vectors belonging to the same invariance submanifold (e.g., the mustard dashed curve), like  $\xv_1$ and $\xv_2$, have the same normalized vector $\hat{x}_1$ and, thus, the same  $\g_2$ value. Therefore, $\nabla \g_2$ must be locally orthogonal everywhere to the invariance submanifolds (iso-level sets of $\g_2$), and, as a consequence, also to $\nabla f_1$.}
\label{fig:manifold}
\end{figure*}

\subsection{Decoupling features from a given set: multi-feature normalization}

Now we consider the problem of decoupling a feature $g$ from a given set $\S= \{f_i:\overline\Omega\subset\R^N \rightarrow \R, i = 1\dots M \}$ of $M$ features. The method follows the normalization scheme explained in the two-feature case, by following trajectories given by the gradients of all the $f_i$. Unlike the simple case of gradient systems, in order to build integral manifolds from multiple feature gradients, additional conditions must be fulfilled.  Indeed, we will give a necessary and sufficient condition for this method to apply (Proposition \ref{prop:iff} below). 

It is also clear that, in order to construct a normalization, we need to restrict to a subdomain $\Omega$, obtained from $\overline\Omega$ by removing its
critical points corresponding to the set of given features $\S =\{f_j,j = 1 ...M\}$.

\subsubsection{Invariant mapping with respect to a set of features}\label{subsection:invariant-mapping}

We first introduce the notion of a mapping being invariant with respect to a set of features $\S= \{f_i:\overline\Omega\subset\R^N \rightarrow \R, i = 1\dots M \}$, a concept that will greatly facilitate the decoupling of an arbitrary feature $g$ from this set.

 We say that a non-constant, differentiable mapping $\yv_\S:\Omega\rightarrow \Omega$ is \emph{invariant} w.r.t. $\S$ iff
\begin{equation}
    \label{eq:invariant_mapping_condition}
    \Jm_{\yv_\S}(\xv) \nabla f_i(\xv) = {\bf 0}, \forall f_i\in\S,\forall \xv\in\Omega,
\end{equation}
where $\Jm_{\yv_\S}$ represents the Jacobian matrix of $\yv_\S$.

\begin{proposition}\label{prop:decoupling}
{\em Obtaining decoupled features from invariant mappings.}
Let $g$ be an arbitrary feature $g:\overline\Omega\rightarrow\R$, and $\yv_\S$ an invariant mapping w.r.t. a set of features $\S$. 
From them we construct a new feature:
\begin{equation}
    \label{eq:feature_of_invariant_mapping}
    \hat{g}_\S(\xv) = g(\yv_\S(\xv)).
\end{equation}
Then it holds that $\nabla \hat{g}_\S(\xv) \cdot \nabla f_i(\xv)=0,\, \forall f_i\in\S,\forall \xv\in\Omega$, i.e., the new feature $\hat{g}_\S$ is decoupled from 
all features in $\S$.
\end{proposition} 
\begin{proof}
Mimicking the calculation in Eq. \eqref{orthogonal-gradients}, we have
\begin{eqnarray}
    \nabla f_i(\xv)\cdot\nabla \hat{g}_\S(\xv) 
    & = & (\Jm_{\yv_\S}(\xv) \nabla f_i(\xv))^T \nabla g(\yv_\S(\xv)) \nonumber \\
    & = & 0,\,\forall f_i\in\S,\forall \xv\in\Omega
\end{eqnarray}
where the last equality holds because of Eq.~\eqref{eq:invariant_mapping_condition}.
\end{proof}

Now that the significance of having an invariant mapping has been established, let us consider the problem of existence. We will need to assume that:
\begin{itemize}
\item[C1.] The gradients $\{\nabla f_i(\xv)\}$ are linearly independent at every point $\xv$; and
\item[C2.] they satisfy the Frobenius condition, which is an integrability condition for several gradients $\nabla f_i$. The related technicalities are addressed in Appendix \ref{appendix:Frobenius}.
 \end{itemize}
 
 We will always assume condition C1 in order to avoid redundancy in the set of features $\S$, even if not explicitly stated.
 
 Now we show that the Frobenius condition C2 is necessary and sufficient for an invariant mapping to apply. This gives a criterion for the possibility of exact decoupling.

\begin{proposition}\label{prop:iff}
Given $\S$ as above, there exists an invariant mapping $ \yv_\S$ w.r.t. $\S$ iff the gradients $\{\nabla f_i,\,i=1\dots M\}$ satisfy the Frobenius condition C2 at each point.
\end{proposition}

This proof will be given in two steps. The  {\em only if} part will be postponed to the Appendix (Section \ref{appendix:iff}) because it is rather technical and not relevant to our study.  Here we will concentrate in the {\em if} statement, which will be treated in Subsection \ref{subsec:normalizaton_Frobenius} below. Our proof is explicit, giving a precise construction of the invariant map via normalization, which is the crucial ingredient.

\subsubsection{Invariance submanifolds}
\label{subsection:Frobenius}

Given a set of features $\S= \{f_i:\overline\Omega \rightarrow \R, i = 1\dots M \}$, the \emph{invariance submanifold}  $\I(\xv_0,\S)$ is an $M$-dimensional submanifold passing through $\xv_0$ whose tangent planes at each point are spanned by the gradients $\{\nabla f_i,\,i=1\dots M\}$. 

To ensure that the invariance submanifold exists we need to assume conditions C1 and C2 above, as it is explained in Appendix \ref{appendix:Frobenius}. Indeed, Frobenius condition C2 is a compatibility condition on the second derivatives of different $f_i$ which is needed for multi-feature integrability. Moreover, it is  is vacuous if $M=1$,  as we only integrate along the gradient of a single feature, see Eq. \eqref{grad-system}.

In addition, Frobenius theorem states that 
$\Omega$ is foliated by invariance submanifolds.

\subsubsection{Normalization of several features}   \label{subsec:normalizaton_Frobenius}

Here  we give the construction of a normalization of multiple features generalizing the  approach in Section \ref{normalization-one-feature} for the decoupling of two features, and show that this normalization indeed yields an invariant mapping.

 For this, we associate a single vector $\yv_\S(\xv)$ to each invariance submanifold $\I(\xv,\S)$, this is,
\begin{equation}
    \label{eq:invariant_mapping_and_invariance_manifold}
    \yv_\S(\zv) = \yv_\S(\xv_0),\,\,\forall\zv\in\I(\xv_0,\S),
\end{equation}
and 
ensure that such mapping $\yv_\S(\xv)$ is continuous and differentiable.  
\label{subsection:normalization}
The remaining question is, then, how to choose a representative $\yv_\S(\xv_0)$ of each invariance submanifold $\I(\xv_0,\S)$. In the case of a normalization, we choose the vector belonging to $\I(\xv_0,\S)$ that attains some reference values $\vv_\S^{ref}$ in its features, values that we know $\I(\xv_0,\S)$ may take, for all $\xv_0\in\Omega$. 

The following proposition states that, after removing a lower-dimensional subset $\Lambda$ from $\Omega$, we can attain a valid set of reference values by moving along the invariance submanifold: 
\begin{proposition}
Let $\vv^{ref}$ be any jointly compatible set of values of $\fv_\S$. Then
 for all  $\xv_0\in\Omega\setminus\Lambda,$ there exists $\zv\in \I(\xv_0,\S)$  that satisfies $\fv_\S(\zv)=\vv^{ref}$, and it is unique in the connected component of $\Omega\setminus\Lambda$  where $\xv$ belongs to.
Thus, the solution set $\I(\xv_0,\S)\bigcap\RM_\S(\vv^{ref})$ contains exactly one point in this connected component. 
\label{prop:all_v_S_are_reachable}
\end{proposition}
\begin{proof} By our assumptions on the critical points, the basin of attraction of critical points that are not global maxima or minima is lower dimensional. Note that the gradient flow of each $f_i$ starting at $\xv_0$ is fully contained in $\mathcal I(\xv_0,\S)$, and thus, the $f_i$'s take all possible values along the flow unless $\xv_0$ belongs to the basin of attraction of a saddle. 

Next, there cannot be  more than one point in $\I(\xv_0,\S)$ with exactly the same reference values since, under our  assumptions on critical points, the flow of each gradient always strictly decreases (resp. increases) the value of the corresponding $f_i$.
\end{proof}

Thanks to the previous Proposition, for $\xv\in\Omega\setminus\Lambda$ we can define the normalization
\begin{equation}\label{def-maping}
\hat{\xv}_\S(\xv;\vv_\S^{ref}) = \I(\xv,\S)\bigcap\RM_\S(\vv^{ref}). 
\end{equation}

\begin{proposition}\label{prop:normalization-invariant}
The normalization $\hat{\xv}_\S$ constructed in  \eqref{def-maping} has the following properties:
\begin{itemize}
    \item[i.] $\hat{\xv}_\S$ is an invariant mapping. 
    \item[ii.] The Jacobian $\Jm_{\hat{\xv}_\S}$, when evaluated on the reference manifold $\RM_\S$, is an orthogonal projection map on $\RM_\S$.  Moreover, it is non-degenerate, i.e., $rank(\Jm_{\hat{\xv}_\S}(\xv))=N-M$ $\forall \xv\in\Omega$. 
\item[iii.] The pair $\left(\fv_\S(\xv),\hat{\xv}_\S(\xv;\vv^{ref})\right)$ carries the same information as $\xv$. In particular, $\xv$ can always be recovered from it by reversing the normalization,
i.e., $\xv = \hat{\xv}_\S(\hat{\xv}_\S(\xv;\vv^{ref});\fv_\S(\xv))$, $\xv\in\Omega\setminus\Lambda$. 
\end{itemize}
\end{proposition}

\begin{proof}
The fact that $\hat{\xv}_\S$ is an invariant mapping is obvious from the construction. Indeed, by definition of Jacobian matrix,
$$\Jm_{\hat{\xv}_\S}(\xv)\nabla f_i(\xv)=\left.\frac{d}{dt}\right|_{t=0} \hat{\xv}_\S({\alpha}_i(t))$$
for a curve $\alpha_i(t)$ satisfying $\alpha_i(0)=\xv$ and $\alpha_i'(0)=\nabla f_i (\xv)$. Since this curve can be taken fully contained inside the invariant submanifold $\I(\xv,\S)$ (following, for instance, the gradient flow of $f_i$), then $\hat{\xv}_\S(\alpha_i(t))$  is a constant function in $t$ and thus, its derivative vanishes. 
 
For the second statement note first that, by applying the chain rule in the condition of Eq. \eqref{idempotent} it immediately yields that
\begin{equation*}
\Jm_{\hat \xv_\S}(\hat\xv_\S(\xv)) \Jm_{\hat \xv_\S}(\xv)  =
\Jm_{\hat \xv_\S}(\xv),\,\,\forall\xv\in\Omega.
\label{idempotent2}
\end{equation*}
Now, recall that for $\xv$ in the reference manifold $\RM_\S$ we have $\hat \xv_\S(\xv)=\xv$, so the previous equation reduces to
$$(\Jm_{\hat \xv_\S})^2(\xv)=\Jm_{\hat \xv_\S}(\xv)\quad\text{for all }\xv\in\RM_\S.$$
Moreover, since $\hat \xv_\S$ is an invariant mapping, the rows of $\Jm_{\hat\xv_\S}$ are orthogonal to gradients $\{\nabla f_j, j = 1\dots M\}$, and then the projection is on the space orthogonal to the linear span of the previous gradients. That is, on the local tangent space in $\RM_\S$.

In addition, the non-degeneracy of the Jacobian follows from a classical result in linear algebra that states that the dimension $N$ is the  sum of the dimension of the kernel ($M$ in our case) plus the rank of the matrix.

The last statement is a consequence of $\Omega\setminus\Lambda$ being foliated by invariance submanifolds.
\end{proof}

\section{The Nested Normalization Method}
\label{sec:NeNs}

So far we have proposed a method for decoupling a new feature with respect to a given set of features. Here we apply the results of the previous analysis in a particular hierarchical fashion, and study how and under which conditions we can obtain a set of mutually decoupled features. 

\subsection{Analysis}\label{subsec:Analysis}

Let us consider a set $\S$ of $M$ ordered non-trivially coupled global features $\{f_i(\xv), i = 1..M\}$. We propose here a sequential algorithm that, starting by taking the first original feature $\g_1 = f_1$ unchanged, aggregates at each step $k$ a new feature $\hat f_{k+1}$, as shown in Algorithm~\ref{alg:hier_approach}:
\begin{algorithm}
\begin{algorithmic}[1]
\REQUIRE Coupled features
$\{f_j, \,\,j = 1,\dots, M \}$
\STATE {\bf Initialization:} $\g_1=f_1$
\FOR {$k=1$ to $M-1$}
    \STATE $\g_{k+1} \leftarrow \text{decouple}(f_{k+1},\{\g_i, i = 1\dots k\})$   \label{alg_step:multifeature_decouple}
\ENDFOR
\RETURN Decoupled features $\{\g_j, j = 1,\dots, M \}$
\end{algorithmic}
\caption{NeN: A hierarchical decoupling approach.}
\label{alg:hier_approach}
\end{algorithm}

Our particular strategy involves constructing suitable normalization maps following this sequential aggregation scheme. For this, we  use hierarchically nested reference manifolds:
\[
\RM_{M-1} \subset \dots \subset \RM_1 \subset \RM_0 = \Omega\subset \R^N,
\]
where, in the notation of Section \ref{subsection:normalization-definition}, $\RM_k = \hat \fv_{k}^{-1}(\vv^{ref})$, being $\hat \fv_{k}$  a map made from the ordered set of features $\hat \S_k=\{\hat f_1,\dots,\hat f_k\}$, and a corresponding set of reference values $\vv_k^{ref}$.   At each step $k$, we obtain a normalization map $\hat \xv_k$ with respect to $\hat\S_k$ and this is, precisely, what defines and gives its name to the \emph{Nested Normalization} (NeN) method.

The proposed nested structure has some consequences. First, although each normalization onto $\RM_k$ imposes a new reference value to the  $k$ feature, it  respects the previously normalized values for the features $j = 1\dots k-1$. 
Second, it implies that  
\begin{equation}\label{ref-manifold}\RM_k = \hat\fv_k^{-1}(\vv_k^{ref}) = \fv_k^{-1}(\vv_k^{ref}),
\end{equation}because $\hat \fv_{k}(\xv) = \fv_{k}(\xv)$ when  $\xv\in\RM_{k-1}$.  This property shows that, under these constraints, one can define reference values with respect to $\fv_k$ or $\hat \fv_k$, interchangeably.

We present two variants of the NeN algorithm:  \emph{Broad} and \emph{Narrow} paths, presented in Subsections   \ref{subsection:broad-path} and \ref{subsec:NeNs}, respectively, depending on the choice of the integration path.

\subsubsection{A broad path to normalization}
\label{subsection:broad-path}

This scheme is precisely explained in  Algorithm~\ref{alg:nested_normalizations_broad_path}. At the  $k$-th step in the Algorithm, we start with a set of features $\hat \S_k=\{\hat f_1,\ldots,\hat f_k\}$, constructed inductively. Assume that these satisfy the 
Frobenius condition C2. Then the broad path scheme   yields  a normalization $\hat\xv_k$ with respect to the $\hat \S_k$ by integrating along trajectories inside the invariance submanifold of $\hat S_k$. Note that all trajectories made of linear combinations of the features' gradients imposing the desired normalization values provide the same normalization result, as they belong to the same integral manifold (which tells us that the order of  integration does not change the output normalization).  Therefore, this method provides us with valuable degrees of freedom for choosing convenient integration paths. 

More formally, we look for suitable combinations of coefficients $\alpha_{j,k}$,
such that the initial value problem:
\begin{equation}
    \frac{d\yv_k(t)}{dt} = \sum_{j=1}^{k}{\alpha_{j,k}(t) \nabla \hat f_j(\yv_k(t))},
    \label{eq:ODE}
\end{equation}
with $\yv_k(0) = \xv$,
can be integrated in $\yv_k(t,\vec \alpha_k(t))$.

We can write, taking into account Eq. \eqref{ref-manifold} for the choice of the reference values,
\begin{equation}\begin{split}
t_s & =  \arg_t \left\{\fv(\yv_k(t,\vec{\alpha}_k(t)))=\vv_k^{ref}\right\},  \\
\hat\xv_k(\xv; \vv_k^{ref})  &=  \yv_k(t_s,\vec{\alpha}_k(t_s)),
\label{eq:analytical_normalization}
\end{split}
\end{equation}
with certainty that such a solution exists and is unique in a connected domain, as it only depends on the reference values of the adjusted features, and not on the choice of the $\alpha$ coefficients. In any case, $\alpha_j$ coefficients need to respect two constraints: (i) having all the sign of $(v_j^{ref} -  f_j(\xv))$, in order to go coordinately in the direction of imposing the reference values to the features; and (ii) 
 they should not introduce any additional stationary solutions apart from the already discussed admissible critical points of the features. Under these constraints, each feature can be adjusted in its full range.

\begin{algorithm}
\begin{algorithmic}[1]
\REQUIRE $\{f_j,v_j^{ref}, \,\,j = 1,\dots, M \}$
\STATE {\bf Initialization:}  $\g_1(\xv) = f_1(\xv)$ 
\FOR {$k=1$ to $M-1$}
    \STATE $\fv_k(\xv) = [f_j(\xv)],\, \vv_k^{ref} = [v_j^{ref}],\, j = 1,\dots,k$
	\STATE $\RM_k = \fv_k^{-1}(\vv_k^{ref})$
    \STATE $\hat\S_k\, = \,\{\hat f_j,j=1\dots k\}\,\,(*)\,\,$    
      \STATE Compute (ODEs) \,$\hat{\xv}_k(\xv; \vv_k^{ref}) = \I(\xv,\hat\S_k)\cap\RM_k\,\,(*)\,\,$      
    \STATE $\g_{k+1}(\xv) = f_{k+1}(\hat{\xv}_k(\xv))$   
    \label{decoupled_from_normalized}
\ENDFOR
\RETURN $\{\g_j, j = 1,\dots, M \}$

 (*) Modifications for the broad path relaxation in Section \ref{subsec:vs}: Substitute $\S_k$ by $\hat\S_k$ and $f_j$ by $\g_j$ in Steps 5 and 6, for the relaxed version of the NeN broad path.
\end{algorithmic}
\caption{Nested Normalization, Analysis - broad path.}
\label{alg:nested_normalizations_broad_path}
\end{algorithm}

The proposed Algorithm \ref{alg:hier_approach} (and its particular realization Algorithm \ref{alg:nested_normalizations_broad_path}) produces a new set of features $\hat f_1,\dots,\hat f_M$ such that each $\hat f_{k+1}$  is decoupled from the previous ones $\hat f_1,\ldots,\hat f_k$. Unfortunately,  it has several drawbacks that often make its implementation  difficult in practice.

A first obstacle in the method is the fact that the ODEs involved in computing Step 6 of Algorithm \ref{alg:nested_normalizations_broad_path} are typically difficult to solve analytically. In fact, lacking an analytical solution for the normalization at the $k$-th iteration translates into not being able to obtain the expression of the decoupled feature for the $k + 1$ iteration (Step 7), and beyond.

However, a more significant drawback of this scheme is that a new  decoupled feature is defined upon the previous ones. As a consequence, the loop stops after a single decoupled feature no longer fulfills the requirements, meaning that the next ``decoupled features” in the loop simply do not exist. In particular, Frobenius condition C2 is rather stringent, besides being usually hard to verify since the gradients of the new features $\hat f_1,\ldots,\hat f_M$ will tend to have very convoluted mathematical expressions. 

In next subsection we develop a second version of the algorithm (\emph{narrow path}) that provides: i) an alternative method for feature decoupling that does not require analytical calculations; and, ii) a way to demonstrate that it is possible to relax the normalization at each step $k$ (also in the broad path algorithm), by making it w.r.t. to the original feature set $\S_k$, instead of w.r.t. the decouple features set, $\hat \S_k$.  
In Subsection \ref{subsec:vs} below we will discuss when both approaches are equivalent.

\subsubsection{The narrow path algorithm}
\label{subsec:NeNs}

Here we propose Algorithm  \ref{alg:nested_normalizations_narrow_path} to construct a normalization $\hat \xv_k$ with respect to the features in $\S_k=\{f_1,\ldots,f_k\}$.  
Such normalization is obtained, at each step $k$,  by moving along the gradient of each feature, projected over the previous reference manifold. Thus our normalization $\hat \xv_k$ is made of a sequence of 1-D integral manifolds, whose only requirement is being each free from critical points (conditions B1-B2, see Subsection~\ref{subsection:gradient-systems}). 

Assuming that Frobenius condition C2 holds for the gradients of $\{f_1,\dots,f_k\}$, then the normalization map $\hat \xv_k(\xv)$ we obtain from the Algorithm~\ref{alg:nested_normalizations_narrow_path} is an invariant mapping with respect to the features in  $\S_k$. This fact follows from  Proposition \ref{prop:normalization-invariant}, since we have just provided an admissible integration path inside the invariance submanifold $\I(\xv,\S_k)$ to reach  $\hat \xv_k(\xv)$ as defined in  \eqref{def-maping}. Another consequence of this Proposition (statement ii.) is that, for $\xv\in\mathcal R_{k}$, 
\begin{equation}\label{eq:projection}
\nabla \hat f_{k+1}(\xv)=P_{\mathcal R_{k}} (\nabla f_{k+1}(\xv)),
\end{equation}
where we have denoted by $P_{\mathcal R_{k}} $ the (orthogonal) projection map on the reference manifold. This in particular implies that $\nabla \hat f_{k+1}$ is orthogonal to the linear span of the gradients  $\nabla f_1,\dots,\nabla f_k$.

If, on the contrary, Frobenius condition does not hold for the original gradients, we can still follow the direction of those projected gradients (termed $\gv_k$ in Algorithm ~\ref{alg:nested_normalizations_narrow_path}) until we reach the desired feature values. This will not  give a new orthogonal gradient $\nabla \hat f_{k+1}$. However, the algorithm still yields interesting results since it produces approximate decoupling and improves performance in different applications (an example is shown in Section~\ref{sec:applications}).

\begin{algorithm}[t]
\begin{algorithmic}[1]
\REQUIRE
$\xv \in \Omega\setminus\Lambda$, \,
$\{f_j, v_j^{ref}, \,\,j = 1,\dots, M \}$
\STATE {\bf Initialization:} $\g_1 = f_1$; ${\cal R}_0 = \Omega\setminus\Lambda$; $\hat \xv_0(\xv) = \xv$
\FOR {$k=1$ to $k=M-1$}
    \STATE Compute $\gv_k = P_{\RM_{k-1}}(\nabla f_k)$ \label{step:project_gradient}
	\STATE $\yv_k(0) = \hat \xv_{k-1}(\xv)$
	\STATE Follow $\gv_k$ until $f_k(\yv_k(t)) = v_k^{ref}$
	\STATE $\hat \xv_k(\xv; \vv_k^{ref})=\yv_k(t)$ \label{step:reference_value_reached}
	\STATE $\g_{k+1}(\xv)=f_{k+1}({\hat \xv_k})$ \label{step:decoupled_from_normalized}
\ENDFOR
\RETURN $\{\g_j(\xv), j = 1,\dots, M \}$, $\hat x_{M-1}(\xv)$
\end{algorithmic}
\caption{Nested Normalization - narrow path.}
\label{alg:nested_normalizations_narrow_path}
\end{algorithm}

Figure~\ref{fig:NeNs_2levels} illustrates the NeN algorithm in its narrow-path version, for three dimensional vectors, defining two nested normalization levels. 

\begin{figure*}[t]
\begin{center}
\hspace{-2.5mm}
\includegraphics[width=0.65\textwidth]{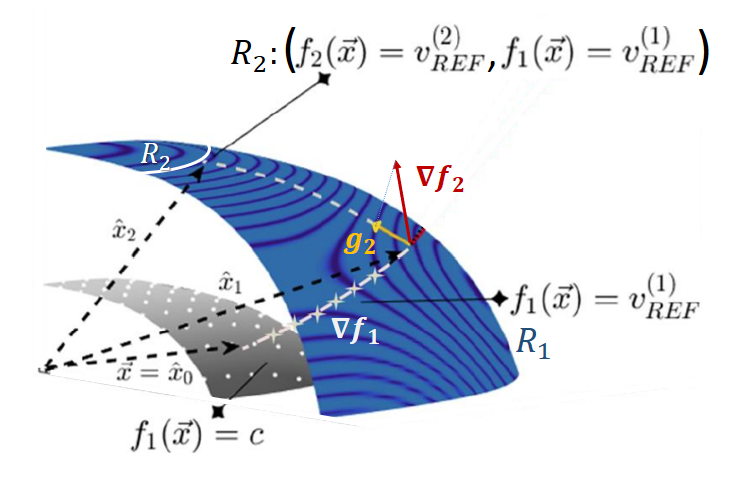} 
\end{center}
\vspace{-4mm}
\caption{Illustration of the NeN algorithm, in its narrow path version. The original vector $\xv$ is first normalized w.r.t. $f_1$, that is, it is modified along the gradient of $f_1$ until reaching $\RM_1$, the reference manifold with all its vectors having $f_1(\xv) = v_{REF}^{(1)}$. That vector is $\hat\xv_1(\xv)$, and there we can evaluate $\hat f_2(\xv) = f_2(\hat\xv_1(\xv))$. From there we follow the projection of the gradient of $f_2$ onto the local hyperplane tangent to $\RM_1$ until reaching $\RM_2$, the set of vectors having $(f_1(\xv) = v_{REF}^{(1)}, f_2(\xv) = v_{REF}^{(2)})$. 
That vector is $\hat\xv_2(\xv)$, the normalization of $\xv$ w.r.t. both $f_1$ and $f_2$. There we can evaluate $\hat f_3(\xv) = f_3(\hat\xv_2(\xv))$ (and so on).}
\label{fig:NeNs_2levels}
\end{figure*}

\subsubsection{From a narrow path to a broad relaxation}
\label{subsec:vs}

As we have seen in the previous subsection, the narrow path yields a  normalization by moving along the gradients of the original features, whereas in the broad path we use the gradients of the modified features. These approaches are equivalent if  Frobenius condition C2 holds on the gradients of the modified features. In fact that, under  this condition, given ${\cal S}_k = \{f_j, j = 1\dots k\}$ and  $\hat{\cal S}_k = \{\hat f_j, j = 1\dots k\}$ (the latter obtained with the Narrow Path Algorithm \ref{alg:nested_normalizations_narrow_path}), then 
\begin{equation}\label{S=hatS}
\hat \xv_{{\cal{S}}_k}(\xv) = \hat \xv_{ \hat{\cal{S}}_k}  (\xv).
\end{equation}
We see that, in this favorable setting, the obtained features in ${\hat \S}_k$ are mutually decoupled. Moreover, since our inductive scheme produces a $\hat f_{k+1}$ that is decoupled from the previous ones, then the features in $\hat{S}_{k+1}$ will also be mutually decoupled.

The proof of \eqref{S=hatS} is a consequence of  our construction, since in Algorithms~\ref{alg:nested_normalizations_broad_path} and \ref{alg:nested_normalizations_narrow_path} the reference manifolds are the same thanks to Eq. \eqref{ref-manifold}. In addition, we recall  Eq. \eqref{eq:projection} that compares the gradients of the original and the modified features when being on the reference manifolds. Thus, the solution constructed by the narrow path algorithm is both a valid concatenation of 1-D integration paths for the decoupled gradients (as we are following them), and for the original gradients, as projected gradients are linear combinations of original gradients.

As a corollary, given that the normalization result (and, thus, also the resulting set of decoupled features) is unique, if it exists, for a given ordered set $\S$ and their corresponding reference values $\vv^{ref}$, then in the broad path Algorithm \ref{alg:nested_normalizations_broad_path} we can simply  substitute $\hat \S_k$ by $\S_k$ in its Steps 5 and 6. By doing that we make it totally equivalent to the narrow path algorithm. 
We call this change a {\em relaxation} of the broad path algorithm, and refer to this modified version as its {\em relaxed version}.
Besides being much easier to implement, the broad path in its relaxed version (like the narrow path and unlike the broad path in its original version), still provides useful results (but not strictly decoupled) when the Frobenius condition does not hold on the output features, as discussed below. 

More generally, our previous discussion yields decoupling in the case when 
a subset of the output features obtained using Algorithm 3 (or equivalent) have gradients fulfilling the Frobenius condition:  
\begin{proposition}
\label{prop:partial_frobenius}
If the set ${\cal S}_k = \{f_j, j = 1\dots k\}$ and the subset $\hat{\cal S}_k = \{\hat f_{j}, j = 1\dots k\}$, $k\leq M$ (the latter obtained from $\hat{\cal S}$ with Algorithm \ref{alg:nested_normalizations_narrow_path} or equivalent, for a given set of reference values $\vv^{ref}$) have gradients fulfilling Frobenius in ${\cal S}_k$ and $\hat {\cal S}_k$, respectively, 
then all pairs $(\hat f_{i},\hat f_j)$, $i=1\dots k$, $j=1\dots M$, $i\neq j$, are mutually decoupled in the whole domain $\Omega$. 
\end{proposition}

 Note that Frobenius condition is vacuous if $k=1$ and consequently, $(\hat f_1,\hat f_j)$, for $j=2\dots M$, are decoupled unconditionally.

If Frobenius condition holds for the gradients of the original features, but not for those of the transformed features, their corresponding gradients will still be orthogonal on their corresponding reference manifolds, i.e.,
$\nabla \g_i(\xv) \cdot \nabla \g_j(\xv) = 0,i\neq j,$ for all $\xv \in \RM_{k-1}$, where $k = \max\{i,j\}$. The proof of this fact follows from the nested structure. Indeed, assume without loss of generality that $i<j$. Then, in $\RM_{j-1}$, $\nabla \hat f_j$ is the orthogonal projection over $\RM_ {j-1}\subset \RM_{i-1}$  of the original gradient, while $\nabla \hat f_i$ is orthogonal to $\RM_{i-1}$  by Eq.~\eqref{ref-manifold}. The interest of this comes from the fact that many times, even if  fully decoupling is not possible, one can still have mutual decoupling over high dimensional manifolds in $\Omega$. 
In such situations, and as a practical consequence, when dealing with probability distributions it is convenient to choose reference values that are the expected values of the density function. This favors obtaining gradients close to mutually orthogonal (see Figs.~\ref{fig:figAE_MSM} and \ref{fig:figAE_VF}, panels (a) and (d), in Subsection~\ref{subsec:gradients_orthogonality}), as samples will be close to the reference manifolds, where exact orthogonality holds. 

Comparing narrow and broad path versions of the NeN algorithm, the principal advantage of the broad path (especially in its much more convenient relaxed form) is that it provides closed-form solutions in some favorable cases. 
This usually translates on its solutions being easier to analyze and faster to compute. On the other hand, narrow path is simpler and more systematic at the implementation level, since it only requires explicit functions for the original features' gradients, and it just relies on numerical integration along 1-D trajectories.

\subsubsection{Normalization of homogeneous features}

A key characteristic of the NeN algorithm as we have presented it so far, is to establish a sequential order among the features to be decoupled. 
While there are cases, such that of marginal moments, for which it is natural to establish a hierarchical order for normalizing the features, there exist other situations where this does not apply.
An example is, given a bank of scaled and rotated filters, to obtain features by applying some functions to the filters' outputs. In this case there is no reason to establish a hierarchy among the features coming from the filters at the same scale, just rotated in different angles. For these situations, a combined graph for the extracted features, having both sequential and parallel nodes seems much more appropriate than a purely sequential scheme.

Fortunately, the tools we have presented so far can be readily applied for (i) ``simultaneously'' 
 (the sequential order chosen for the integration does not affect the result) changing an input sample by forcing a set of $P_k$ features to have their reference values (i.e., normalizing the vector w.r.t. that set of features), and (ii) obtaining new features that are functions of the normalized vector (applying Equation~\eqref{normalize-g}, and  Section~\ref{subsec:normalizaton_Frobenius}), which we know that will be decoupled to those $P_k$ parallel features.
Therefore, we can easily adapt our algorithms just by changing some of the sequentially adjusted features $f_k$ by $P_k$-dimensional features $\overline f_p^{(k)}$ (not to be confused with the aggregated maps $\fv_k$ from $j=1$ to $k$), without changing the underlying logic.
It must be noted, though, that this procedure does not {\em mutually} decouple these parallel features~\footnote{Mutual decoupling can be obtained by imposing a sequential order to these homogeneous features. Fully parallel (symmetrical) decoupling is a hard problem requiring different techniques from the ones presented here.}.

Algorithm~\ref{alg:normalizations_parallel} shows our strategy to ``simultaneously'' impose a set of reference values to a set of homogeneous features, i.e., features having all the same functional expression, but different parameters $\{\vec\lambda_j, j = 1\dots P_k\}$.
The underlying idea is simple: to express analytically the solution of a single sequential ODE integration, one ODE excursion for each of the homogeneous features, and then obtain an analytical expression concatenating all these excursions. The time values $t_j$ are left as variables, that are computed numerically in order to fulfill the normalization (or de-normalization).
Note the difference with Eqs.~\eqref{eq:analytical_normalization}, where the normalization was obtained by solving a single non-linear equation at a time (scalar $t$, instead of a vector $\tv$, like now). However, note as well that the solution of Algorithm~\ref{alg:normalizations_parallel} can also be expressed in terms of those equations, by choosing $\{\alpha_j(t)\}$'s that activate sequentially single-gradient combinations, at times $\{t_j\}$.
Finally, it must be also noted that a variant of the previous algorithm can be used as well for the case of having homogeneous features not in parallel, but hierarchically ordered (e.g., second-order moment at the output of a bank of filters). In that case we can apply Algorithm~\ref{alg:normalizations_parallel} to hierarchically normalize nested subsets of features, as a particular procedure for computing Step 6 in  Algorithm~\ref{alg:nested_normalizations_broad_path}.
\begin{algorithm}
\begin{algorithmic}[1]
\REQUIRE
$\xv_0 \in \Omega\setminus\Lambda, \,$
${\cal \tilde S}_k = \{f_{j,k}(\xv) = \tilde f_k(\xv;\vec\lambda_j)\}$, 
$\{v_{j,k}^{ref}\}, \,\,j = 1,\dots, P_k$
\STATE Solve for generic ODE  $\yv_k(t;\vec\lambda,\xv)$ (analytically)
\STATE {\bf Initialization:} $\xv_{0,1} = \xv$
\FOR {$j=1$ to $P_k-1$}
    \STATE $\xv_{0,j+1} = \yv_k(t_j;\vec\lambda_j,\xv_{0,j})$
\ENDFOR
\STATE $\tv = [t_1,\dots,t_{P_k}]$, ${\cal L} = \{\vec\lambda_1,\dots,\vec\lambda_{P_k}\}$ 
\STATE Solve for $\tilde\yv_k(\tv;{\cal L},\xv) = \xv_{0,P_k}$ (analytically)
\STATE Solve for $\hat \tv = \arg_{\tv} \{f_{j,k}(\tilde\yv_k(\tv;{\cal L},\xv_0))  =  v_{j,k}^{ref}\}_{j=1}^{P_k}$
\RETURN $\hat\xv_{{\cal \tilde S}_k}(\xv_0) = \tilde\yv_k(\hat\tv;{\cal L},\xv_0)$
\end{algorithmic}
\caption{Normalization of homogeneous features}
\label{alg:normalizations_parallel}
\end{algorithm}
In Subsection~\ref{subsec:VF_study_case} we study the different alternatives to approximately decouple the second-order moments at the output of a set of hierarchically ordered filters, among which Algorithm~\ref{alg:normalizations_parallel} provides the most systematic approach still providing (partially) analytical solutions. 

\subsection{Feature Transfer}
An essential characteristic of the integration along the direction of one or several gradients is its reversibility. First, the adjustment of a set of feature values, under the given constraints and assumptions, is always possible for every $\xv\in\Omega\setminus\Lambda$,  whenever the set of desired values are algebraically compatible (see Proposition~\ref{prop:all_v_S_are_reachable}). It is also true, in particular, that we can {\em de-normalize} a normalized vector (whenever the reference manifold itself is also contained in $\Omega\setminus\Lambda$), not just for recovering the original vector (as pointed out in the property (iii) from Proposition \ref{prop:normalization-invariant}), but also for imposing whatever new feature values we may aim for.
Thus, the good properties of the NeN analysis methodology presented so far allow us to change the role of vector transformation (by integrating the gradient flows) from instrumental to the main goal, and, as such, changing the focus from analysis to feature transfer or even synthesis.

However, before going into how to do that, it is important to realize that a set of mutually decoupled features have their joint range decoupled, as shown next. Given two features $f_i$ and $f_j$ it is immediate to assess that the decoupling condition of Eq.~\eqref{eq:dec_fea} implies that a local change in $\xv$ along the gradient of one of the features does not affect the value of the other feature. More precisely, let us assume $\hat{\S}$ is a set of decoupled features defined in a set $\hat \Omega\setminus\hat \Lambda$. Since we can modify each feature within its whole range by navigating along its one-dimensional flow, independently of the values of the other features, we obtain as a corollary that
\begin{equation}\label{decoupling-range-property}
    \Rg(\hat\fv)=\Rg(\hat{f}_1)\times \ldots\times \Rg(\hat{f}_M)
\end{equation}
in the set $\hat{\Omega}\setminus\hat{\Lambda}$.
 
The problems of finding the largest admissible domain $\hat\Omega$ for the decoupled set of features (which implies knowing the location and subsequently excluding all its critical points, C1 condition), and the set  $\hat\Lambda$ of basins of all their saddles, as well as the range of each decoupled feature, are not trivial, and depend on the particular set of features.
Thus, we leave that analysis for  Section~\ref{sec:study_cases} and 
the Appendix C (\ref{subsec:critical_points_moments} and \ref{subsec:critical_etc_variance}), where the decoupling of two particular sets of features is studied in detail. 

\subsubsection{Peeling the onion and covering it back with new layers}
\label{subsec:transfer_via_denormalization}

Now, the decoupling range property \eqref{decoupling-range-property} 
allows to modify a signal by enforcing arbitrary desired values (each within its valid range) for the decoupled features without iterative corrections, opening up an unprecedented scenario. 

We show below how to achieve this by means of our NeN algorithm. All we need is applying the following 3-step procedure: (i) obtaining a set of desired decoupled features $\vv^{des}$ we want to transfer, either by applying any of the NeN analysis algorithms we have presented so far to some {\em target} vector data $\yv$, or simply by choosing any combination of decoupled features' values within their valid range; (ii) normalizing the {\em source} vector data $\xv$ (the one we want to transform) up to the $M-1$-level; and (iii) de-normalizing the previously normalized data, in reverse order, until achieving for each decoupled feature $\g_k$ the same value $v_k^{des}$ previously measured/chosen in step (i). More formally:
\begin{equation*}
\begin{split}
\text{(i)}&\quad\vv^{des}  =  \hat\fv(\yv)  \\
\text{(ii)}&\quad\hat\wv_{M-1}(\xv)  =  \hat\xv_\S(\xv;\vv_{M-1}^{ref})  \\
\text{(iii)}&\quad\zv_f(\xv;\vv^{des})  =  \check{\xv}_\S(\hat\wv_{M-1}(\xv); \vv^{des}),
\end{split}
\end{equation*}
where we have used the symbol {\Huge $ _{\check{}}$} to indicate de-normalization.
Thanks to the reversibility of the whole process (based on the reversibility of each ODE integration), the resulting transformed data vector $\zv_f$ shares the same exact decoupled features $\vv^{des}$ with the target vector, and the same normalized {\em kernel} $\hat\wv_{M-1}$ with the source vector $\xv$:
\begin{eqnarray*}
\hat\wv_{M-1}(\zv_f;\vv_{M-1}^{ref}) & = & \hat\wv_{M-1}(\xv;\vv_{M-1}^{ref}) \\
\hat\fv(\zv_f) & = & \hat\fv(\yv),
\end{eqnarray*}
Figures~\ref{fig:feature_transfer} and~\ref{fig:denormalization_scheme}  illustrate this process.
\begin{figure*}
    \begin{tabular}{ll}
    \includegraphics[width=0.7\textwidth]{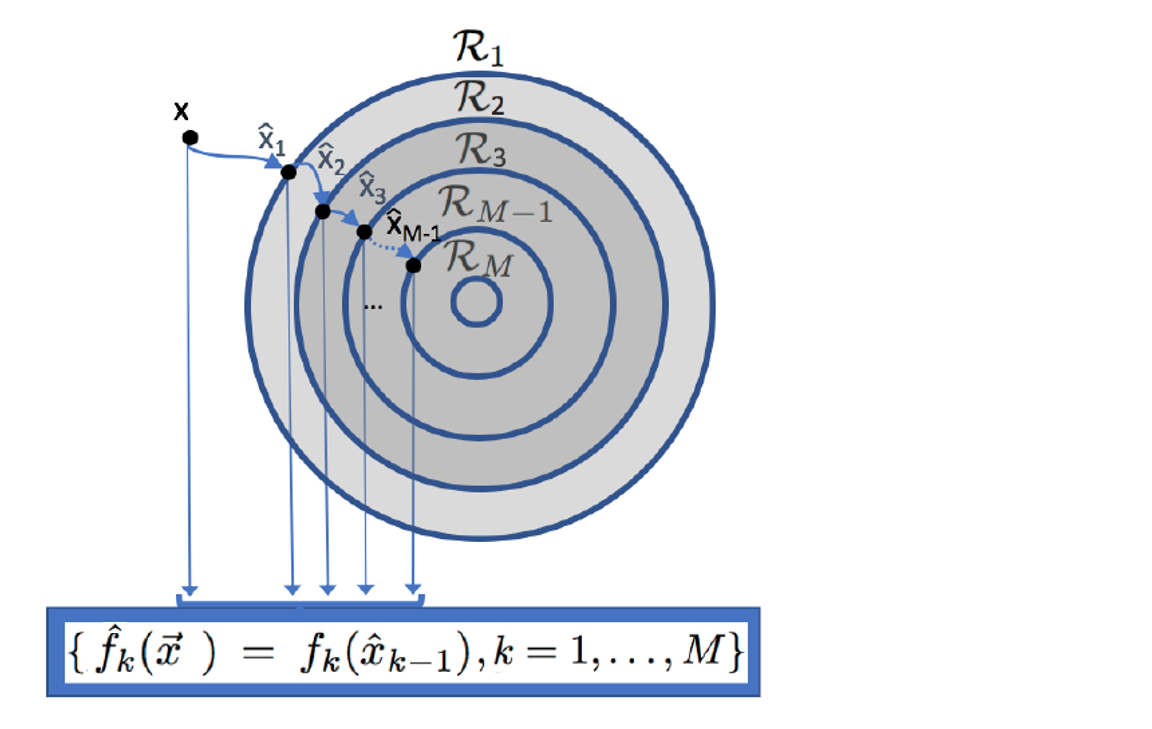} &
    \hspace{-4cm}
    \includegraphics[width=0.7\textwidth]{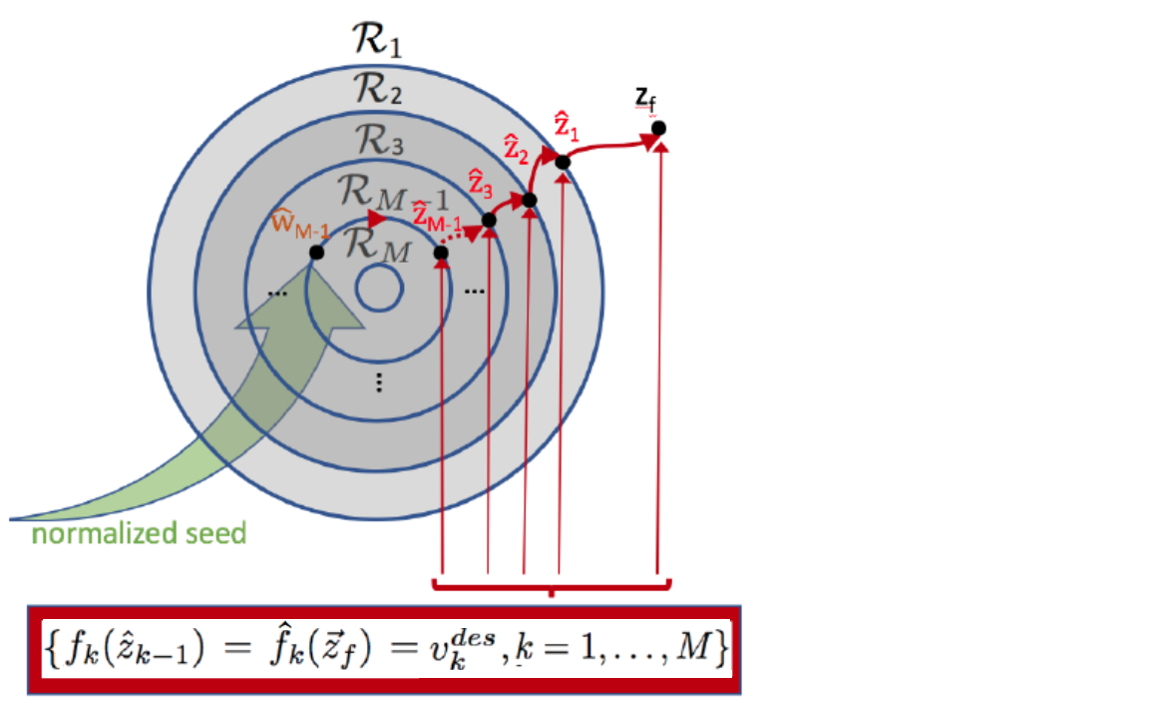}\\
    Analysis via Normalization & Transfer via De-normalization 
    \end{tabular}
    \caption{Feature transfer can be done as a sequence of (i) extracting decoupled features in a target vector data, (ii) normalization of the source data (both (i) and (ii) represented on the left), and (iii) de-normalization of the data normalized in (ii), by imposing the features $\vv^{des}$ measured in (i), in reverse order (on the right). See Algorithm~\ref{alg:NeNsSynthesis}.}
    \label{fig:feature_transfer}
\end{figure*}
\begin{figure*} 
  \centering
  \includegraphics[width=0.7\textwidth]{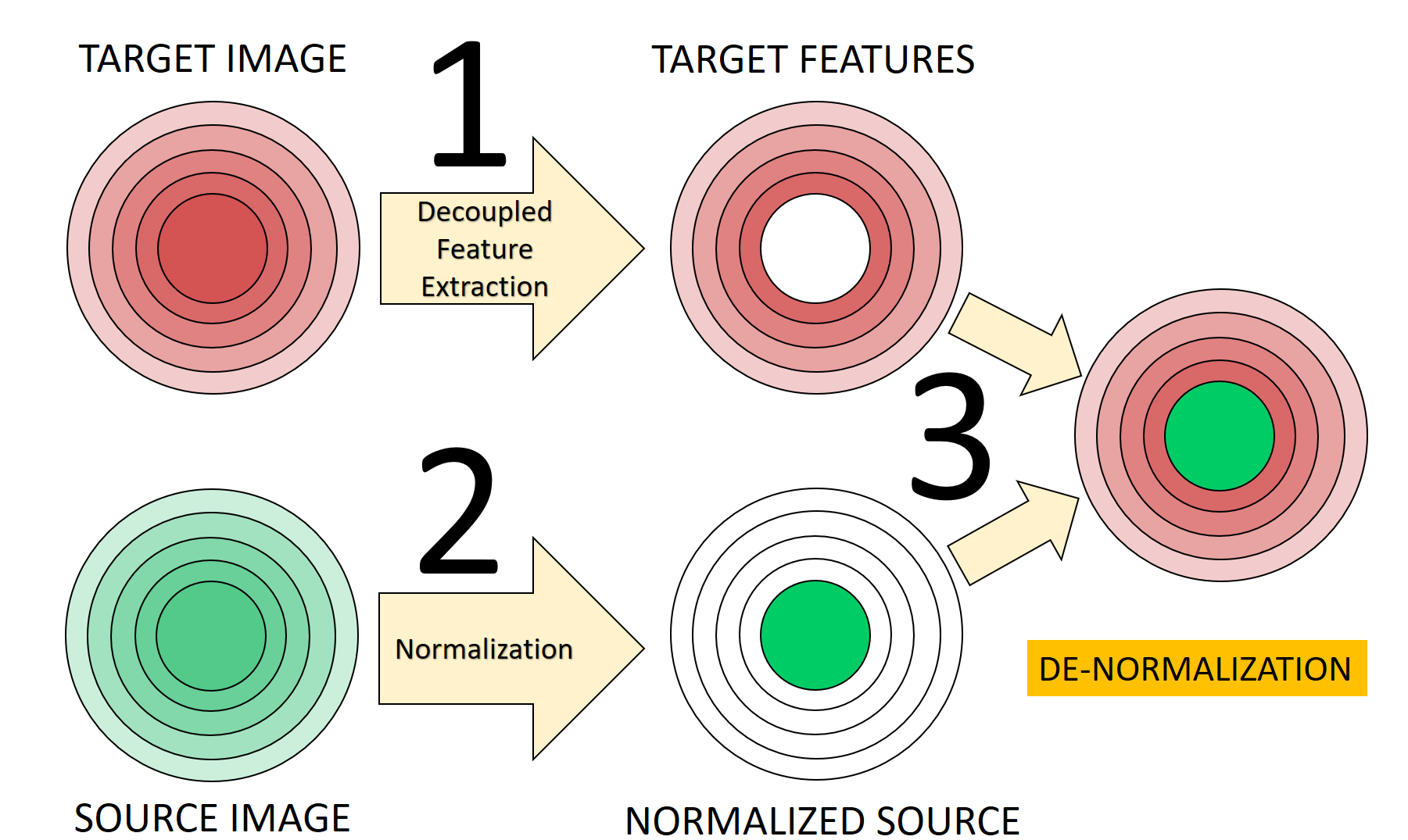}
  \caption{Three-step feature transfer using de-normalization, within the Nested Normalization framework. The left half of the figure represent the normalization+analysis process (first and second steps, in parallel), and the right half of the figure (converging arrows) is the de-normalization+transfer process (third step).}
  \label{fig:denormalization_scheme}
\end{figure*}
In Algorithm~\ref{alg:NeNsSynthesis} 
we describe the de-normalization steps using the (reversed) narrow-path algorithm.

This approach to feature transfer, termed Controlled Feature Adjustment in~\cite{MMSP2020}, and applied to photo-realistic style transfer in~\cite{SIIMS2022}, opens up a new catalogue of transfer and synthesis possibilities, as explained in those references.

\begin{algorithm}
\begin{algorithmic}[1]
\REQUIRE ${\hat \wv}_{M-1}\in \Omega\setminus\Lambda$, $\fv^{ref}\in \Rg(\fv), \fv^{des}\in \Rg(\hat\fv)$
\STATE {\bf Initialization:} $\hat z_M = \hat w_{M-1}$
\FOR {$k=M$ to $1$}
	\STATE $\yv(0) = \hat z_{k}$ 
    \STATE Compute $\gv_k = P_{\RM_{k-1}}(\nabla f_k)$
	\STATE Follow $\gv_k$ until $f_k(\yv_k(t)) = v_k^{des}$
	\STATE $\hat z_{k-1}=\yv(t)$
\ENDFOR
\RETURN $\zv_f = \hat z_0$
\end{algorithmic}
\caption{NeN, De-normalization (narrow path).}
\label{alg:NeNsSynthesis}
\end{algorithm}

\subsection{Critical points and perturbations}

\subsubsection{Critical points}
\label{subsec:critical_points}

In this subsection we look at the conditions (assuming B1, B2, C1, C2 hold) for the critical points of the decoupled features obtained using the NeN algorithm. This will be needed in order to establish the range of the newly constructed features.

The critical points of a feature $\g_k$ are those points $\xv_k^*$ satisfying $\nabla \g_k(\xv_k^*) = 0$. As, when using the Nested Normalization algorithm, we have that  $\nabla \g_k = P_{\RM_k}(\nabla f_k)$ within $\RM_{k-1}$, the critical point condition corresponds to the orthogonality of $\nabla f_k$ to the reference manifold $\RM_{k-1}$. Because, by the definition of reference manifold, the local tangent space at $\xv\in\RM_{k-1}$ is the orthogonal complement of the linear span ${\cal L}(\{\nabla f_j\}_{j=1}^{k-1})$, having a null projection on that local plane implies in this case that $\nabla f_k\in {\cal L}(\{\nabla f_j\}_{j=1}^{k-1})$, i.e., that there exist $ \{\lambda_{j,k}\in \R\}_{j=1}^{k-1}$ not all zero such that:
\begin{equation}
    \nabla f_k(\xv_k^*) = \sum_{j=1}^{k-1} \lambda_{j,k} \nabla f_j(\xv_k^*).
    \label{eq:crit_point_condition_NeNs}
\end{equation}
(Note also that a trivial solution of Eq.~\eqref{eq:crit_point_condition_NeNs} comes from having common critical points of the original features for orders $j<k$, for which both sides of the equation vanish.)  
This is precisely the same condition as C1.

By introducing the (known) structure of the gradients, we obtain a general expression for the critical points (which, by the structure of the NeN method, are not isolated, but entire submanifolds). 
In addition, by using the previous equation plus the constraints derived from $\xv_k^*\in\RM_k$
we obtain the expression of the critical points of $\g_k$ on $\RM_k$
(see the corresponding calculations in the two study cases, in Section~\ref{sec:study_cases}).

Finally, it is important to note that decoupled features at level $j$ {\em are not defined} at critical points of decoupled features at level $i$, for $i<j$. The reason is that iso-level sets of the feature $j$ all cross orthogonally the iso-level sets of the feature $i$, until converging all to a critical point (maximum or minimum, a source or a drain of the gradient field of feature $i$), and therefore the feature $j$ is not defined there. 
This is a direct consequence of the normalization $\hat\xv_i(\xv)$ requiring the existence of a non-null gradient $\nabla f_i(\xv)$ for initiating the ODE trajectory that adjusts the $i$-th feature to its reference value, a previous step for computing $\g_{i+1}(\xv) = f_{i+1}(\hat{\xv}_i(\xv))$. 
Figure~\ref{fig:orthokurtosis} (comparing the orthokurtosis to skewness and kurtosis) may help to understand the involved concepts.

\subsubsection{Introducing perturbations to avoid spurious basins}\label{sec:perturbation}
We have imposed that the domain in which the features in $\S$ are defined, $\Omega$, is free from critical points.
A practical way to choose the common domain $\Omega$ for a set of features $\S$ is to find the largest admissible set, consisting of the intersection of the domains on which each feature is defined, and then to remove from it all critical points.

We have also assumed B1-B2 conditions, i.e., that the basins of attraction of critical points other than the absolute maxima and minima are a set $\Lambda$ of submanifolds with a joint dimension lower than that of $\Omega$. In that case, it is well known that a small perturbation of a point inside these basins will take it back into $\Omega\setminus\Lambda$ with probability one.\footnote{Although it is very easy to fabricate ad-hoc counter-examples, these counter-examples will never happen by chance using continuous perturbation densities in $\R^N$.} 
We should also verify that for {\em almost all} $\xv \in \Omega\setminus\Lambda$, the perturbed $(\xv+\epsilon)$ will remain in $\Omega\setminus\Lambda$\footnote{Idem previous footnote.}.
At which stage should we add a perturbation, then?
The simplest solution, because it does not even require to check if $\xv\in\Lambda$, is to  {\em always} add the perturbation just before performing gradient integration. 

In addition, in the context of a real application, there are two more concerns on the practical impact of the above constraints and assumptions:
(i) usually, real-world signals are quantized, and, thus, they are not in $\R^N$, but in a finite, numerable subset. In particular, within a quantized representation, it is no longer true that the probability of {\em falling} on a lower dimension basin of attraction of a spurious critical point is zero. Actually, quite the opposite, signal quantization will reduce the entropy and favor the symmetry.  
(ii) Whereas being directly on a spurious basin of attraction (of a saddle) will make our method to eventually get stuck, that is not the only problem. Close-distance neighbor points, although outside those basins of attraction, may be affected negatively in terms of the speed of convergence of the differential equations, and should be avoided. Same can be said about the adjustment of features to values too close to their absolute extrema, especially when using numerical integration because of lacking analytical solutions.

\subsubsection{Choosing a suitable perturbation}
\label{subsec:perturbation}
Critical points, as we will see in the study cases (Section~\ref{sec:study_cases}), are typically points presenting low entropy configurations (high symmetry, a few repeated values/patterns, etc.).
The role of the added perturbation is to increase the entropy of the signal in a suitable way, without affecting its relevant (e.g., perceptual) features.
Here is a list of desirable characteristics of a perturbation $\epsilon$ on a digital signal:
\begin{itemize}
    \item it should not cause a direct loss of information, i.e., $\qv(\xv+\epsilon) = \qv(\xv) = \xv$, where here $\qv(\cdot)$ represents the quantization already present in $\xv$,
    \item it should be reproducible,
    \item it should increase the entropy (break the symmetry),
    \item it should not affect the operational conditions (e.g., not noticeable under human observation).
\end{itemize}
Among previous characteristics, the third and fourth ones discourage us from naively using noise-like perturbations (i.i.d. pseudo-random coefficients), because they may produce (i) some samples having very close values (just by chance); and (ii) noticeable (e.g., visual, auditory, etc.) artifacts, because of introducing unnecessarily large differences among neighbors.
In appendix subsections \ref{subsec:perturbation_moments} and \ref{subsec:critical_etc_variance} we propose ad-hoc perturbations for two different decoupling problems. 
In both studied cases the perturbation aims at increasing the entropy of the signal, by increasing the number either of the distinct signal values or of the active frequencies in the Fourier domain. 
As such, these perturbations expand the decoupled feature ranges to their theoretically broadest possible intervals.

\section{Two study cases}
\label{sec:study_cases}

\subsection{Marginal Moments}
\label{subsec:marginal_moments}

We study here the problem of hierarchically decoupling the first $M$ sample moments:
$$\S=\left\{f_j(\xv) = \frac{1}{N} \sum_{n=1}^{N}{x_n^j}\right\}_{j=1}^M.$$
First of all, we check that in this case the Frobenius condition C2 holds (see Appendix~\ref{appendix:Frobenius} for a proof). Therefore, the corresponding set of gradients defines a proper invariance submanifold and we can use the broad path from Subsection~\ref{subsection:broad-path}.
This fact will be crucially exploited for computing in a quasi-explicit way the fourth-order decoupled moment, that we termed the {\em orthokurtosis}.
In Section \ref{sec:applications} we will show how the pair skewness-kurtosis is clearly inferior for several analysis tasks than the couple skewness-orthokurtosis. 

To fix notations, we will use for the normalization the moments of a zero-mean uni-variate Gaussian distribution: 
$$
v_j^{ref}  = 
\begin{cases}
 (2j)!/(2^j j!) & \text{ for}\,j\,\text{even},  \\
  0, & \text{ for}\,j\,\text{odd.}
 \end{cases}
 $$

\subsubsection{Analytic solutions: a path to the {\em orthokurtosis}} 

First, the gradient of the sample mean (the sample mean is both $f_1$ and $\g_1$) is:
\[
\nabla \g_1(\xv) = \frac{1}{N} \v1,
\]
and the normalization comes from solving the ODE
\[
\frac{d\yv_1(t)}{dt} = \frac{1}{N}\v1
\]
starting from $\yv_1(0) = \xv$ and finding  $t_1 = \arg_t \{f_1(\yv_1(t))=0\}$ (recall that $\mathcal R_1=\fv_1^{-1}(\vv_1^{ref})=\{\zv\in\Omega\,:\, f_1(\zv)=0\}$). The  normalization will then be  $\hat\xv_1(\xv)=\yv_1(t_1)$.
In this case the solution is straightforward:
\begin{eqnarray*}
\yv_1(t) & = & \xv + \frac{1}{N} t \nonumber \\
f_1(\yv_1(t)) & = & f_1(\xv) + \frac{1}{N} t \nonumber \\
t_1 & = & -N f_1(\xv)  \nonumber \\
\hat\xv_1(\xv) & = & \xv - f_1(\xv)\v1,
\end{eqnarray*}
which corresponds to the original vector with the sample mean subtracted to every sample.
Now we obtain the next decoupled feature, $\g_2(\xv) = f_2(\hat\xv_1(\xv))$:
\[
\g_2(\xv) = f_2(\xv - f_1(\xv)\v1),
\]
which is a (biased) version  of the classical sample variance.
Now we compute its gradient,
\begin{eqnarray}
\nabla \g_2(\xv) & = & \frac{2}{N} \hat\xv_1(\xv) \nonumber \\
& = & \frac{2}{N} (\xv - f_1(\xv)\v1).
\end{eqnarray}
Because of the irrelevance of the factor $2/N$ for the subsequent calculations, we drop it.
Now we can modify $\g_2$ by moving along its gradient without leaving $\RM_1$, until reaching $\RM_2 = \fv_2^{-1}(\vv_2^{ref}) = \{\zv\in\Omega: f_1(\zv)=0,\,f_2(\zv)=1\}$.
Now $\yv_2(0) = \hat\xv_1(\xv)$,
\[
\frac{d\yv_2(t)}{dt} =  \left(\yv_2(t) - f_1(\yv_2(t))\v1\right).
\]
This ODE simplifies by noting that, when the gradient has zero sample mean, the sample mean can not change when integrating it (the resulting curve belongs to $\RM_1$).
As the initial value has zero mean, $f_1(\yv_2(t))=0\,\,\forall t$,  the ODE simplifies to:
\[
\frac{d\yv_2(t)}{dt} = \yv_2(t),
\]
whose solution $\log(\yv_2(t)) + C = t \v1$ results in
\[
\yv_2(t) = \exp(t) \hat\xv_1(\xv) ,
\]
by enforcing $\yv_2(0) = \hat\xv_1(\xv)$. We see that $f_2(\yv_2(t)) = f_2(\hat\xv_1(\xv))\exp(2t)$ and the $t$ value intersecting with $\RM_2$ is $t_2 = - 1/2 \log(f_2(\xv_1(\xv)))$.
Then
\[
\hat\xv_2(\xv) = \yv_2(t_2) = \hat\xv_1(\xv)/\sqrt{\g_2(\xv)}.
\]
We see that this second normalization is the standardization of $\xv$.
Now we can compute the next decoupled moment $\g_3(\xv) = f_3(\hat\xv_2(\xv))$:
\begin{equation}
\g_3(\xv) = f_3\left(\frac{\xv - f_1(\xv)\v1}{\sqrt{f_2(\xv - f_1(\xv)\v1)}}\right),
\label{eq:skew}
\end{equation}
which is the sample skewness.
Here again, we compute $\nabla \g_3(\xv)$ on $\RM_2$ by differentiating in Eq.~(\ref{eq:skew}). 
We obtain:
\begin{equation}
    \nabla \g_3(\xv) = \frac{3}{N} \left( \hat\xv_2(\xv)^{\odot2} - \g_3(\xv)\hat\xv_2(\xv) -\v1 \right),
\end{equation}
where we use the symbol ``$\odot$'' for representing a pointwise scalar operation for the vector coefficients (here a power).

Unlike in previous cases, now the resulting ODE equation $\frac{d\yv_3(t)}{dt} =     \nabla \g_3(\yv_3(t))$ has no known closed-form solution.
However, we note two facts. First one, as mentioned above, the original gradients fulfill the C1-C2 conditions and, thus, define a proper invariance submanifold. Second, in this case we observe that the linear span of $\{\nabla \g_j(\xv)\}_{j=1}^{3}$ is the same as the linear span of $\{\nabla f_j(\xv)\}_{j=1}^{3}$, $\forall \xv \in \Omega$. As a consequence, both sets of features produce the same invariance submanifolds. Thus to compute the normalization, it is equivalent to  use both broad path (Algorithm 
\ref{alg:nested_normalizations_broad_path}) or its relaxed version. In particular, using the gradients of the original features instead of the decoupled ones allows to find a closed-form solution for the normalization.
More precisely, we proceed as follows: first find a solution by following $\nabla f_3$ until achieving zero skew, and then standardizing that result, i.e., imposing also zero-mean and standard deviation one. The latter adjustments are a shift and re-scaling that correspond to moving along $\nabla f_1$ and $\nabla f_2$, admissible operations within the submanifold, that do not affect the zero-skew condition.
Thus, we can pose the much easier Ricatti ODE equation obtained moving along the $\nabla f_3$, $\frac{d\zv_3(t)}{dt} = \nabla f_3(\zv_3(t))$\footnote{This is not the only possibility for obtaining analytic solutions for the normalization, although it seems the best, in this case. Note that other integrable gradients, such as $\zv_3(t)^2-1$ (giving rise to an hiperbolic tangent solution) have a reduced range and converge to spurious stationary solutions.}:
\[
\frac{d\zv_3(t)}{dt} = \frac{3}{N} (\zv_3(t))^2,
\]
with $\zv_3(0) = \hat{\xv}_1(\xv)$, whose solution is (because of its irrelevance for the next calculations, we ignore the $\frac{3}{N}$ factor): 
\begin{equation}
\label{eq:skew_adjustment}
\zv_3(t) = \hat{\xv}_1(\xv)\odot/(\v1 - t \hat{\xv}_1(\xv)).
\end{equation}
Then we find a numerical solution for
\begin{equation}
\label{eq:zero_skew}
t_0(\xv) = \arg_t \{\g_3(\hat{\xv}_1(\xv)\odot/(\v1 - t \hat{\xv}_1(\xv)))=0\},    
\end{equation}
and we obtain the third-order normalization by standardizing $\zv_3(t_0(\xv))$\footnote{It is easy to check that Eq.~\eqref{eq:skew_adjustment} always has solution within the open interval $(1/\min(\hat{\xv}_1(\xv))),1/\max(\hat{\xv}_1(\xv)))$ (note that $sign(\min(\hat{\xv}_1(\xv)))\neq sign(\max(\hat{\xv}_1(\xv)))$ for all $\xv\neq {\mathbf 0}$).}:
$$\hat{\xv}_3(\xv) = \hat{\xv}_2\left(\hat{\xv}_1(\xv)\odot/(\v1 - t_0(\xv) \hat{\xv}_1(\xv))\right).$$
Finally, this allows us to define the fourth-order decoupled moment as $\g_4(\xv) = f_4(\hat{\xv}_3(\xv)).$
We have termed this new decoupled sample moment $\g_4$ the {\em orthokurtosis}, a function that, unlike the classical standardized fourth-order sample moment (the kurtosis) is not just decoupled from the mean and variance, but also from the skewness.

Note that the computation of the orthokurtosis includes a non-explicit function, namely $t_0(\xv)$.
Although we could apply a similar strategy for obtaining closed-form solutions (up to the integration parameter value) for decoupling higher-order moments by integrating separately along integer power gradients, that would not provide us with efficient solutions, because we would still need to numerically find the parameter $t_0$ for which the reference value for the decoupled moment is reached. For instance, for computing the fifth order decoupled moment we need to normalize the orthokurtosis. This implies evaluating every time in a loop this function, which, in turn, requires the normalization in loop of the skewness. In summary, once there are no fully analytic expressions, computationally expensive nested search loops appear, and a piece-wise 1-dimensional ODE integration strategy (as the one described in Algorithm~\ref{alg:nested_normalizations_narrow_path}) is preferable.   

Finally, it is worth emphasizing how the first three decoupled moments obtained using the NeN algorithm are precisely the classical standardized moments: mean, variance and skewness. This clearly reflects how our method has captured the pre-existing natural intuitions about decoupling features through normalization. However, the next standardized moment, the kurtosis, breaks the pattern of being decoupled from all previous standardized moments, as it is algebraically coupled to skewness. This algebraic coupling has been previously noted, and some solutions have been proposed (see, e.g., \cite{Blest2003}). Previous efforts have not aimed at orthogonalizing the involved gradients, and the few proposed modifications of the kurtosis lack a theoretical foundation and have proven inferior in their practical application to the solutions presented here (see Fig.\ref{fig:results} in Section \ref{sec:applications}). 

Figure~\ref{fig:orthokurtosis} illustrates the orthogonalization of the kurtosis with respect to the skewness, giving rise to the orthokurtosis.  It shows the actual iso-level curves, for the case of $N=4$ (for visualization purposes, a dimension has been removed, namely forcing that the solutions belong to the hyperplane $\mu(\xv)=0$). Each trajectory shown in the orthokurtosis representation has been actually computed by integrating the projected gradients, starting from a randomly perturbed maximum of the skewness (a perturbation is necessary, because the new function is not defined at the skewness' critical points) and finishing in one minimum.

\begin{figure}[ht]
\begin{center}
\begin{tabular}{c}
 \includegraphics[width=0.3\textwidth]{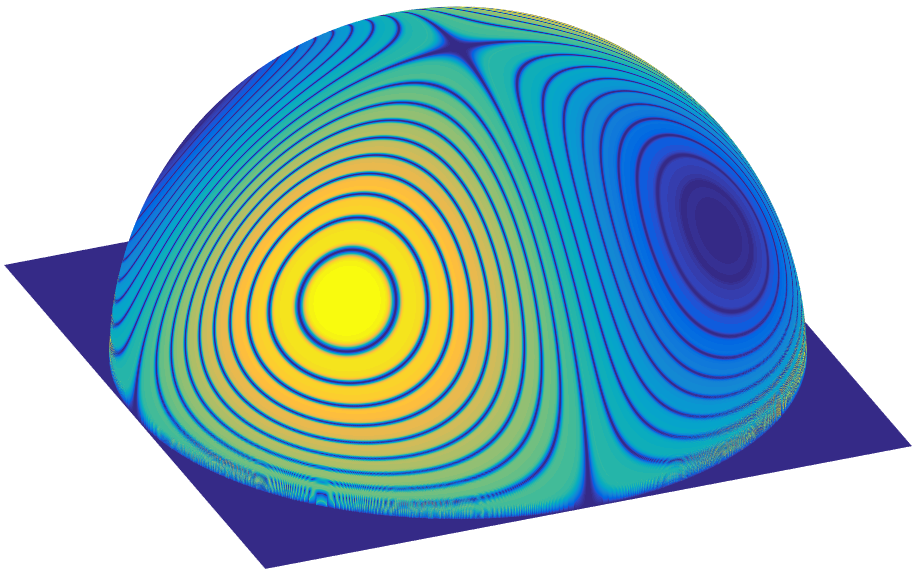} \\ 
 \includegraphics[width=0.3\textwidth]{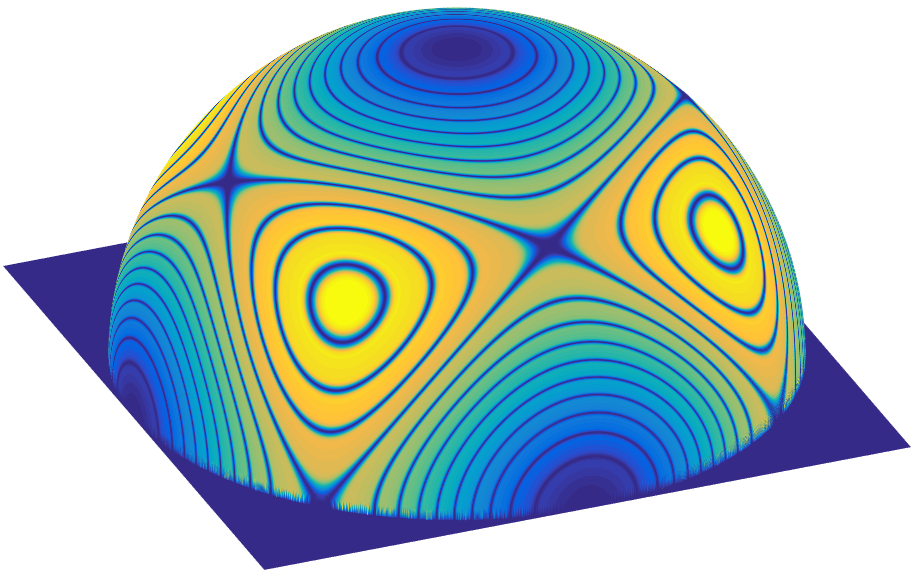} \\
 \includegraphics[width=0.3\textwidth]{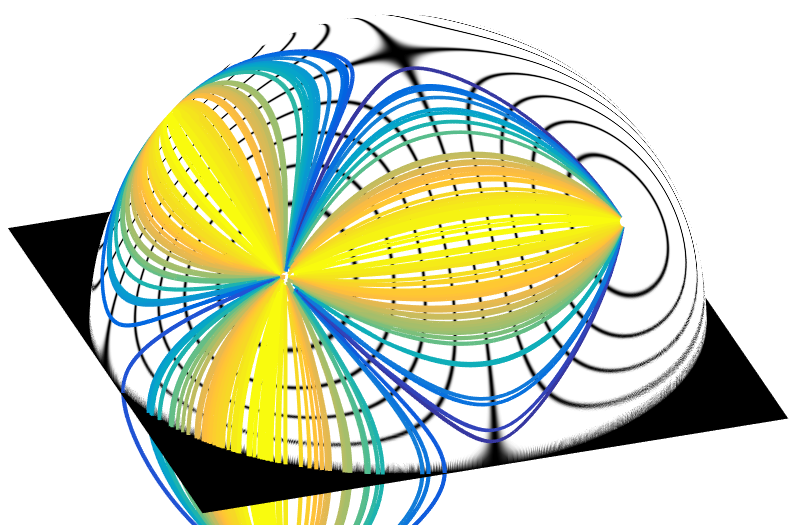} \\
\end{tabular}
\end{center}
\vspace{-2mm}
\caption[t]{Visual represention of sample skewness (up), kurtosis (middle) and {\em orthokurtosis}, the new fourth-order normalized moment (bottom). Dark curves (at all three panels), and coloured curves (at the bottom), are actual iso-level curves. Yellow/bright represents high values, and blue/dark low values. At the bottom panel, the iso-orthokurtosis colored curves are drawn over the iso-skewness curves (black) to show mutual orthogonality. Note how the orthokurtosis is not defined at the extrema of the skewness.}
\label{fig:orthokurtosis}
\end{figure}

\subsubsection{Beyond orthokurtosis: higher-order approximately decoupled moments}
\label{subsec:higher_order_moments}

As explained in Section 3, universal and exact feature decoupling is only possible when the gradients obtained by any of the proposed algorithms fulfill the Frobenius condition C2.
The gradient of the orthokurtosis, 
together with the gradients of its preceding decoupled moments (sample mean, variance and skewness), no longer fulfill C2 condition.
Therefore, an exact unconstrained hierarchical decoupling solution is not possible beyond four-order moments.
Nevertheless,  as shown in Section~\ref{subsec:vs}, gradients orthogonality can still hold exactly within the reference manifolds and, as shown in Section~\ref{subsec:gradients_orthogonality}, approximately outside of them.

Thus, in a looser sense of ``decoupling", the lack of analytic solutions fulfilling Frobenius beyond the fourth order is not an insurmountable obstacle for computing  higher-order approximately decoupled moments. 
In fact, we have seen (Proposition~\ref{prop:normalization-invariant} and posterior discussion, in Subsection~\ref{subsec:normalizaton_Frobenius}) how easy is to compute the gradient of a decoupled feature $\nabla \g_k(\xv)$ if we constrain $\xv$ to belong to the reference manifold $\RM_{k-1}$.
In that case it holds $\nabla \g_k(\xv) = \gv_k(\xv) = P_{\RM_{k-1}}(\nabla f_k(\xv))$, where $P_{\RM_{k-1}}$ is the projection operator on the local tangent plane to $\RM_{k-1}$ in $\xv$, $\xv\in\RM_{k-1}$. As the orthogonal space to that tangent plane is the linear span of the previous gradients $\nabla f_j, j=1\dots k-1$, the projection can be computed by finding a linear combination of all gradients (including $\nabla f_k$) that is orthogonal to the previous gradients. Because in our case the original gradients are made of ``monomial vector" of increasing orders, it is easy to solve the triangular system of equations resulting from equating their inner products to zero \cite{Portilla:ICIP:2018}, yielding the projected gradients (a set of orthogonal vectors):
\begin{eqnarray}
\label{eq:ng1}
\gv_1(\xv) & = & \frac{1}{N}({\v1}) \nonumber \\
\label{eq:ng2}
\gv_2(\xv) & = & \frac{2}{N}(\xv -f_1{\v1}) \nonumber \\
\label{eq:ng3}
\gv_3(\xv) & = & \frac{3}{N}\left(\xv^{\odot 2}  - f_2{\v1} - a_{2,3} \gv_2(\xv)\right) \nonumber  \\
& \vdots & \nonumber \\
\gv_k(\xv) & = & \frac{k}{N}\left(\xv^{\odot k-1} -\sum_{j=1}^{k-1} {a_{j,k} \gv_j(\xv)}\right),
\end{eqnarray}
with
\begin{eqnarray*}
c_{i,j}^{(\ell)} & = & 
\begin{cases}
 f_{i+j} - f_i f_j & \text{ if }\ell=1, \\
  c_{i,j}^{(\ell-1)} c_{\ell-1,\ell-1}^{(\ell-1)} - c_{\ell-1,i}^{(\ell-1)} c_{\ell-1,j}^{(\ell-1)} & \text{ if } \ell>1,
  \end{cases}\\
a_{j,k} & = & 
\begin{cases}
 f_{k-1} & \text{ if }j=1,\\
  c_{j-1,k-1}^{(j-1)} / c_{j-1,j-1}^{(j-1)} & \text{ if } j>1.
\end{cases}
\end{eqnarray*}
We remind the reader that, in general, 
$\gv_k = \nabla \g_k$ only for points belonging to their corresponding reference manifolds (the $k$-th gradient is computed and applied in the $\RM_{k-1}$ manifold). Some of them may get simpler expressions when imposing the corresponding reference values (particularly, lower than  order $k$ odd moments vanish if we take as reference the moments of even symmetric pdf's, e.g., Gaussian).
They have the cross-invariance property, i.e., by integrating a curve along the $k$-th decoupled gradient we do not change the previous featuress $f_j,\g_j\,j= 1\dots k-1$.  
Although in this case we obtained a closed-form (recursive) solution for the gradient projection, in case of lacking close-form expressions we can always apply a purely numerical orthogonalization method to the gradient vectors of the original features (like Gram-Schmidt).

In our practical examples of applying decoupled moments to signal analysis in Section~\ref{sec:applications}, we demonstrate the usefulness of higher than four order decoupled moments. This is, we believe, a relevant result, as the original moments of so high order are very rarely used in the literature due to their instability and high redundancy.

\subsection{second-order moments at the output of a set of filters}
\label{subsec:VF_study_case}

We study here the problem of hierarchically decoupling second-order moments measured at the output of a set of $M$ linearly independent band-pass (zero DC-response) filters:
$$\S=\left\{f_j(\xv) = \frac{1}{N} \sum_{n=1}^{N}{[\xv * \hv_j]_n^2}\right\}_{j=1}^M.$$
Such a set of features provides an economical description of the auto-correlation of $\xv$. We will use for the normalization 
$\{v_j^{ref} = f_j(\wv)\}_{j=1}^M$, 
being $\wv = N{\mathbf \delta} $, a scaled Kronecker delta.
This is equivalent to taking for reference values the expected value of $f_j(\yv)$ for $\yv$ a vector made of i.i.d. zero-mean and unit variance coefficients.  
This choice for reference values, being the values of the functions applied to a given vector $\wv$, guarantees the algebraic compatibility of $\vv^{ref}$, regardless of the chosen set of filters $\{\hv_j\}$. In addition, it makes the normalization to {\em whiten} the input.

Same as in previous Subsection~\ref{subsec:marginal_moments}, we first check that in this case the Frobenius condition C2 holds for the original gradients (see Appendix~\ref{appendix:Frobenius} for a proof), and, therefore, the corresponding set of gradients defines a proper invariance submanifold.
This allows us to apply Algorithms \ref{alg:nested_normalizations_broad_path} (in its relaxed form),  \ref{alg:nested_normalizations_narrow_path} or \ref{alg:normalizations_parallel} for trying to decouple these features, and see {\em a posteriori} if Frobenius condition also holds on the gradients of the obtained features. 

\subsubsection{General analytical approach}

We first obtain the feature gradients and study their integrability. We have:
\[
\nabla f_j(\xv) = \frac{2}{N} \xv * \hv_j*\tilde\hv_j,
\]
where $\tilde{\hv}_j(\nv) =\hv_j(-\nv)$.
To simplify this expression and its subsequent integration, we express it in the Fourier domain, by doing the DFT of $\xv$ and $\hv$: \begin{equation}
\label{eq:gradient_variance}
G_j(X(\xiv)) = {\cal F}\{\nabla f_j(\xv)\} = \frac{2}{N} |H_j(\xiv)|^2 X(\xiv),
\end{equation}
where $\xiv$ represents the (possibly vectorial, for $n$-D signals, $n>1$) discrete frequencies, 
and, as usual, upper case letters represent the Fourier transforms of their original lower-case counterparts (except for $G$, that corresponds to the Fourier transform of the gradient of the features).

In order to normalize the first $k$ features, following the broad path algorithm, we can write the relaxed version of Eq.~(\ref{eq:ODE}) in the Fourier domain as:
\[
\frac{dY_k(\xiv,t)}{dt} = \frac{2}{N}\left(\sum_{j=1}^{k}{\alpha_{j,k}|H_j(\xiv)|^2}\right) Y_k(\xiv,t),
\]
for some convenient choice of the integration path, encoded by  the coefficients $\vec\alpha_k$. 
Setting $Y_k(\xiv,0) = X(\xiv)$ and calling $L(\xiv;\vec\alpha_k)$ the filter of the sum in brackets in previous equation, the integration of the initial value problem for computing the signal normalization is straightforward: 
\begin{equation}
    \label{eq:fix_variance}
    Y_k(\xiv,t;\vec\alpha_k) = X(\xiv) \exp\Big(\frac{2}{N}L(\xiv;\vec\alpha_k)t\Big).
\end{equation}
To obtain $\hat\xv_k(\xv,\vv^{ref})$, we first normalize $\alpha_{1,k} = 1$ and then solve for
\begin{equation}
\label{eq:ref_condition_VF}
    (\vec\alpha_k^{ref},t_k^{ref}) =\displaystyle \underset{\vec\alpha,t}\arg
    \Big\{\sum_{\xiv}{|H_j(\xiv)|^2 \,|Y_k(\xiv,t;\vec\alpha_k)|^2} = v_j^{ref}\Big\}_{j = 1}^{k}
   \end{equation}
Finally, $\hat{X}_k(\xiv) = Y_k(\xiv,t_k^{ref};\vec\alpha_k^{ref})\}$, and $\hat\xv_k(\xv,\vv^{ref}) = {\cal F}^{-1}\{\hat{X}_k(\xiv)\}$.
Solving the non-linear system of $k$ equations and $k$ unknowns of Eq.~\eqref{eq:ref_condition_VF} will require a numerical computation in a general case.
However, it may also have simplified solutions in some special cases, as we will see in next subsection.

An equivalent possibility for computing $\hat\xv_k(\xv,\vv^{ref})$ consists in following the Algorithm~\ref{alg:normalizations_parallel}: calculating analytically the result of a generic ODE excursion using a single gradient $\nabla f_j$ from a given point, with $\alpha_{j,k}=1$, as a function of $t_{j,k}$ (initially left as unknown), Step 1. Then we concatenate these 1-D trajectories (Steps 2-5) into a final analytical solution depending on $\tv_k = [t_{j,k}], j = 1\dots k$ (Step 7). It is easy to see that such analytical solution has the same form of Eq.~\eqref{eq:fix_variance} (but depending on a set of consecutive times $t_{j,k}$, instead of $\alpha_{j,k}$), and that the Step 8 of that algorithm, which enforces the solution achieving all $k$ reference values in $\fv_k(\xv)$, is equivalent to Eq.~\eqref{eq:ref_condition_VF}. 

\subsubsection{A case including two complementary filters}

Let us consider a Parseval frame representation~\cite{Kovacevic2008} using two filters, $H_1(\xiv),H_2(\xiv)$ fulfilling $|H_1(\xiv)|^2+|H_2(\xiv)|^2 = 1$, e.g., a low and high-pass kernels of a redundant wavelet transform. Here the features are simply $f_j(\xv) = \sigma_j^2(\xv)$, $j=1,2$. Since the Euclidean metric is preserved, we have $\sigma_1^2 + \sigma_2^2 = \sigma^2$, the total signal variance (for simplicity sake we assume here zero mean).

For normalizing the first feature, we integrate its corresponding gradient, obtaining:
\[
Y_1(\xiv,t) = X(\xiv) \exp\Big(\frac{2}{N}|H_1(\xiv)|^2t\Big).
\]
As usual, we then find the $t$ parameter providing us the desired reference value:
\[
t_1^{ref}(\xv) = \underset{t}\arg \Big\{\sum_{\xiv} |H_1(\xiv)|^2\,|Y_1(\xiv,t)|^2 = v_1^{ref}\Big\}
\]
and we obtain $\hat\xv_1(\xv) = {\cal F}^{-1}\{Y_1(\xiv,t_1^{ref}(\xv))\}$, from which we obtain the decoupled feature $\g_2(\xv) = f_2(\hat \xv_1(\xv))$.

If we wanted to normalize also the second feature (e.g., in order to add more decoupled features), we could move along the gradient of the second feature, but, at the same time, control that the first feature does not change its value (this is equivalent to project $\nabla f_2$ onto $\mathcal R_1$).
We can achieve that by dividing by the square root of the first feature evaluated for each $t$. Such an adjustment is valid in this case because it corresponds to moving along the sum of the two gradients (which in this case corresponds to simply applying a scale factor to the vector):
\begin{equation}
\label{eq:filters}
Y_2(\xiv,t) = \frac{\hat X_1(\xiv) \exp(\frac{2}{N}|H_2(\xiv)|^2t)}{\sqrt{(v_1^{ref})^{-1}\sum_{\xiv}|H_1(\xiv)|^2\,|\hat X_1(\xiv) \exp(\frac{2}{N}|H_2(\xiv)|^2t)|^2}},
\end{equation}
thus enforcing that $\sum_{\xiv}|H_1(\xiv)|^2|Y_2(\xiv,t)|^2 = v_1^{ref}$.
Same as before, we solve for the $t$ value that achieves the desired normalization:
\[
t_2^{ref}(\xv) = \underset{t}\arg \Big\{\sum_{\xiv} |H_2(\xiv)|^2|Y_2(\xiv,t)|^2 = v_2^{ref}\Big\}
\]
and we obtain $\hat\xv_2(\xv) = {\cal F}^{-1}\{Y_2(\xiv,t_2^{ref}(\xv))\}$, from which we could obtain another decoupled feature from any arbitrary (non-trivially redundant) feature $g(\xv)$; indeed, $\hat g(\xv) = g(\hat \xv_2(\xv))$.

Note also that Eq. \eqref{eq:filters}, although looking quite different from  Eq. \eqref{eq:fix_variance}, is a particular case of the latter with $\alpha_{1,2} = t_1^{ref} - \nu$, $\alpha_{2,2} = t_2^{ref} - \nu$, $\nu$ being the logarithm of the square root in the denominator of Eq.~\eqref{eq:filters}. No matter the adjustment method applied here may seem arbitrary, the fulfillment of Frobenius conditions on the gradients of the original features guarantees, jointly with the additional constraints explained in Section \ref{sec:decoupling}, that the result of such normalization exists and is unique, as it only depends on the reference values of the adjusted features, and not on the choice for the coefficients $\alpha$'s in the linear combinations of the feature gradients, in the ODEs.

Finally, although not mathematically proven here\footnote{Explicit expressions of the second decoupled feature's gradient can be obtained through implicit derivation.}, it turns out that the resulting gradients of the new features obtained in this case do not fulfill the Frobenius condition. This implies that strict gradients' orthogonality only holds for all pairs within their reference manifolds, and in the whole domain for the pairs $(\g_1,\g_j),j=1\dots M$, as explained in Subsection~\ref{subsec:vs}.  Nevertheless, we have obtained gradients that are very close to being mutually orthogonal also in the other cases when applying a set of bar and edge detectors both to white noise and to textured patches of photographic images (see Fig.~\ref{fig:figAE_VF} in Section~\ref{subsec:gradients_orthogonality}).

\subsection{A summary of feature-decoupling scenarios}

Table~\ref{tab:scenarios} summarizes the different situations one may encounter when trying to apply the decoupling framework proposed here to decouple a given set of ordered features.
From less favorable to more favorable, the first scenario is when we do not have an explicit expression of the original features gradients. The most extreme case would be that each of our features is a black box function. In that case, a very computationally costly numerical procedure (such as described in Section~\ref{subsec:gradients_orthogonality}, see Eq.~\eqref{eq:partial_der}) is the only option for approximately computing the gradients. A better situation is when using artificial neural networks  (ANNs) with  automatic differentiation for computing the gradients of the cost functions.  
In that case we can evaluate the gradients with a reasonable cost (and apply numerical integration, using the narrow-path algorithm), but we may not be able to assess the fulfillment of the Frobenius condition for the original features beyond the first layer. 
As such, we should not expect a strict and universal decoupling beyond the second feature (the second layer, if we decouple in a layer-wise fashion). For those features, it becomes an empirical matter to test how far new gradients typically are from being mutually orthogonal.

Second scenario corresponds to knowing the explicit expressions of our features and their gradients, and knowing that they do not jointly fulfill the Frobenius condition. In that case we can still apply the narrow-path version of NeN and, again, test empirically how far the resulting gradients of the (approximately) decoupled features are from being orthogonal.
This corresponds, for instance, to the case of higher-than-two order moments at the output of a set of filters.

Third scenario is when we know the analytic expressions of the gradients of the original features and they fulfill Frobenius. In that case the definitions of Section 2 for the decoupled features apply for finding decoupled features {\em to the original ones}, and we may end up finding explicit expressions for their gradients. 
However, in this scenario these gradients turn out not fulfilling Frobenius condition. 
First,  we recall that the first decoupled feature is just the first feature $\g_1(\xv) = f_1(\xv)$, as always, and, for that reason ${\cal S}_1 = \hat{\cal S}_1$, so the broad-path algorithm applied to obtaining the second decoupled feature is the same as its relaxed version. Furthermore, the second feature is exactly and universally decoupled, as Frobenius condition becomes vacuous in 1-D. However, if we are able to compute both gradients and they do not fulfill the Frobenius condition, we then know that no added third or following features will be universally and exactly decoupled from the two previous ones. Still, we may find exact constrained decoupling solutions: 
the output features will be exactly decoupled within the corresponding reference manifolds, but only approximate outside them. Again, finding how close to be orthogonal are the gradients outside those manifolds requires an empirical measurement. An example of this situation is when decoupling the second-order moments at the output of a set of filters overlapping in the Fourier domain. 

Finally, the most favorable scenario corresponds to being able to obtain explicit expressions for the input (coupled) and output features (by using the broad-path NeN algorithm, Algorithm 2, either in its original form or in its relaxed version) and their corresponding gradients, and that the two sets of gradients fulfill Frobenius condition. In this case we obtain the full decoupling solution, which we know is unique, universal and exact. In addition, in this case there is a joint equivalence relationship between the set of original features and the obtained decoupled set. E.g., when decoupling the first three marginal moments, the decoupled result (sample mean, variance and skewness) jointly carries exactly the same information as the original set (first, second and third-order moments). When adding the fourth-order moment we obtain another exact and universally decoupled feature (the orthokurtosis). However, because of the gradient of the orthokurtosis no longer fulfills Frobenius condition jointly with the lower order gradients, then: (1) the joint equivalence relationship between coupled and decoupled moments up to order four does not hold anymore, and (2) we can not further obtain exact universal decoupling for any added feature to this set (and, in particular, for any higher-order moments).

\begin{table*}
	\begin{center}
	\caption{\label{tab:scenarios} Four feature  decoupling scenarios.}
   \scalebox{1.1}{
	
	\begin{tabular}{|cc||c|c|c|c|}
    \multicolumn{2}{c}{  }
    & \multicolumn{4}{c}{\# S C E N A R I O}\\
	\cline{3-6}
    \multicolumn{2}{c|}{  } & 1 & 2 & 3 & 4 \\
    \hline
    \multirow{2}{*}{Original Feat. Gradients:}
    & Analytic Expression & Non-Explicit & Yes & Yes & Yes \\
	& Frobenius & ? & No & Yes & Yes \\
	\hline
    \multirow{2}{*}{Output Feat. Gradients:}
    & Analytic Expression & ? & No & Yes (2) & Yes \\
	& Frobenius & ? & No & No & Yes \\
	\hline
	\multicolumn{2}{|c||}{Joint Equivalence Coupled/Decoupled Set} & No & No & No & Yes \\
	\multicolumn{2}{|c||}{Does accept one additional decoupled feature?} & No & No & No & Yes \\
	\hline
    \multirow{3}{*}{NeN Applicability:}
    & Computation & Heavy/? & Medium & Medium & Light \\
	& Algorithm & 3 & 3 & 3,4 & 2,3,4 \\
	& Decoupling & 1 layer & Approx/Constr & Approx/Constr & Exact\&Univ. \\
	\hline
	\multicolumn{2}{|c||}{Example (See Table~\ref{tab:feat} for acronyms' description) }& ANN & MF ($p>2$) & MF ($p=2$) & MM ($p\leq3$) \\
	\multicolumn{2}{|c||}{See Secs.\& Refs.}& Future Work
	& \ref{subsec:texture_classification}, \ref{subsec:dec_TIL} \&\cite{Portilla:ICIP:2018,Martinez:ICASSP:2020} & \ref{subsec:VF_study_case}, \ref{subsubsection:empirical-gradientsVF} \&\cite{MMSP2020} & \ref{subsec:marginal_moments}, \ref{subsec:marg_mom}, \ref{subsubsection:empirical-gradientsMM}, \ref{subsec:regression} \\
    \hline
    \end{tabular}
	}
	\end{center}
\end{table*}

\section{Deterministic Decoupling and Local De-Correlation}
\label{sec:local_decorrelation}

Here we show how features' decoupling removes local covariance in the feature space, 
and how this improves discrimination. 

\subsection{Covariance-free ``balls" in the feature space}

Let $\xv \in \mathbb{R}^N$ be a random vector made of $N$ i.i.d. samples obeying a probability distribution $p(x)$.
Let us assume a feature set ${\cal S}$ of $M$ global features $\left\{ f_j \right\}$. 
Define a vector $\cc \in \R^M$ containing the expected value of the features $f_j(\xv)$ for different realizations of $\xv$, i.e., $c_j=\mathbb{E}\{f_j(\xv)\}, j=1\dots M$.
Define a manifold ${\cal A}(\cc;{\cal S}) = \fv^{-1}(\cc)$, i.e., the manifold containing the set of all vectors $\xv$ having the same $\cc$ (note that it is non-degenerate by our initial assumption).
Describe vector samples as $\xv_i = \xv_{0i} + \dv_i$, where $\xv_{0i} \in {\cal A}$, and $\dv_i$ is a (relatively small) sampling fluctuation, with $\mathbb{E}\{d_{i,k} d_{i,l}\} = 0$, $k,l \in [1,\dots, N]$, for all $k$-th and $l$-th components of the vector $\dv_i$. 

\begin{proposition}
\label{propcorr}
\emph{(Decoupled features are locally uncorrelated).}
Under previous assumptions, for $N$ large, decoupled features will have uncorrelated deviations from their expected values, i.e., $\mathbb{E}\{(\g_n(\xv)-c_n)(\g_m(\xv)-c_m)\} \approx 0$, $n,m \in [1,\dots, M]$. 
\end{proposition}

\begin{proof}

For $N$ large, features' values, as they behave like sample statistics (see Eq.~\eqref{eq:gf}), will not deviate much from their expected values and thus vector samples $\{\xv_i\}$ will be located in the vicinity of $\cal A$.
Thus, a first-order local approximation can be applied, which gives
\begin{equation}
\label{eq:balls}
f_j(\xv_i) \approx f_j(\xv_{0i}) + \nabla f_j(\xv_{0i}) \cdot \dv_i = c_j + \nabla f_j(\xv_{0i}) \cdot \dv_i. 
\end{equation}
Therefore, $\mathbb{E}\{(f_n(\xv)-c_n)(f_m(\xv)-c_m)\} \approx \mathbb{E}\{( \nabla f_n(\xv_0) \cdot \dv(\xv_0))(\nabla f_m(\xv_0) \cdot \dv(\xv_0) )\}$, yielding the covariance:
\begin{equation}
Cov(f_n,f_m)(\xv) \approx \sigma_d^2 \,\nabla f_n(\xv_0) \cdot \nabla f_m(\xv_0),
\label{eq:covariance}
\end{equation}
where $\sigma_d^2$ is the expected quadratic dispersion of the features fluctuations. 
\end{proof}

In the decoupled features case gradients are mutually orthogonal, and thus vector differences for the different features will be uncorrelated.  
In contrast, when using coupled features, $\dv_i$ is projected onto non-orthogonal directions, leading to 
correlated sampling fluctuations in the feature space, as illustrated in Figure~\ref{fig:manifold_perturbation_correlation}.

\begin{figure}
\begin{center}
\includegraphics[width=0.45\textwidth]{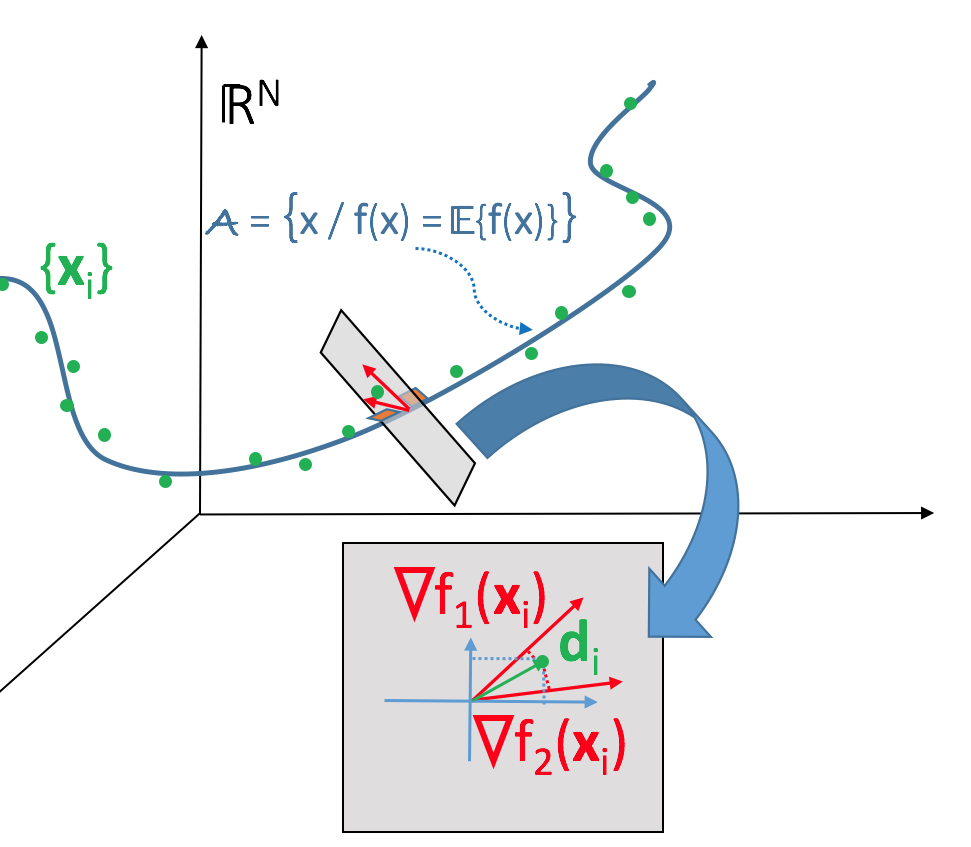}\\
\end{center}
\vspace{-4mm}
\caption{A representation of the expected feature vector manifold in the input space domain, showing some input samples, and the influence of feature gradients' correlation on the correlation in the feature domain.}
\label{fig:manifold_perturbation_correlation}
\end{figure}

\subsection{Features' covariance and discriminability}

Let us assume now that our pdf depends on a parameter $\theta$, $p(x;\theta)$.
Consider also a global feature transformation $\fv(\xv):\mathbb{R}^N \rightarrow \mathbb{R}^M$ meant to be applied to vectors $\xv(\theta)\in\R^N$ made of samples from $p(x;\theta)$.
How well can we discriminate samples coming from similar values of $\theta$, based on $\fv(\xv(\theta))$?
For studying this problem it is convenient to represent the samples $x(\theta)$ using an intermediate stochastic sample $x_0\sim p(x;\theta_0)$ that does not depend on $\theta$; then we obtain the final sample by applying a deterministic invertible mapping $s_\theta:\mathbb{R} \rightarrow \mathbb{R}$ of $x_0$ depending on $\theta$: $x(\theta)=s_\theta(x_0)$, such that $x(\theta)\sim p(x;\theta)$ (re-parametrization trick~\cite{Williams1992,Kingma2014}. 

This allows us to study the dependency of the expected feature vector $\bar\fv$ on $\theta$, by expressing: 
\begin{eqnarray}
\frac{d \bar\fv(\theta)}{d\theta} & = & \frac{d \mathbb{E}\{\fv(\xv(\theta))\}}{d\theta} \nonumber \\
& = & \mathbb{E}\left\{{\bf J}_{\fv} \frac{d \xv(\theta)}{d\theta}\right\} \nonumber \\
 & = & \mathbb{E}\left\{ {\bf U}_{\fv}{\bf S}_{\fv}{\bf V}^T_{\fv} \frac{d \xv(\theta)}{d\theta} \right\},
\label{eq:feature_dependency}
\end{eqnarray}
where ${\bf J}_{\fv}$ is the Jacobian matrix of $\fv$ and ${\bf U}_{\fv}{\bf S}_{\fv}{\bf V}_{\fv}^T$ is its singular value decomposition, SVD (we have omitted here their dependency on $\xv(\theta)$ for brevity). 
On the other hand, from~Eq. \eqref{eq:covariance} we can write the expected local covariance matrix ${\bf C}(\theta)$ of the features fluctuations, as: 
\begin{eqnarray}
{\bf C}(\theta) & \approx & \sigma_d^2\, \mathbb{E}\{{\bf J}_{\fv} {\bf J}_{\fv}^T\} \nonumber \\
& = & \sigma_d^2 \mathbb{E}\{{\bf U}_{\fv} {\bf S}_{\fv}^2 {\bf U}_{\fv}^T\} .
\label{eq:covariance_full}
\end{eqnarray}
Here it is crucial to notice that, under the assumptions made in previous and current subsections, whereas ${\bf J}_{\fv}(\xv(\theta))$ will heavily depend on $\xv$, ${\bf J}_{\fv}(\xv(\theta)){\bf J}_{\fv}^T(\xv(\theta))$ will be much less sensitive to $\xv$, as it only depends on the inner products of the different features' gradients (see ~Eq. \eqref{eq:covariance}). Furthermore,  in Subsection~\ref{subsec:gradients_orthogonality} we show how these inner products (at least their correlation factor, which depends only on their relative angle) are fairly stable,
especially when inputs are samples from pdf's.  Therefore, the ${\bf U}_{\fv}(\xv(\theta))$ and ${\bf S}_{\fv}(\xv(\theta))$ matrices, on their average behavior, will determine both the direction of change of $\bar \fv$ when changing $\theta$ (Eq. \eqref{eq:feature_dependency}) and the dominant direction of ${\bf C}(\theta)$ (Eq. \eqref{eq:covariance_full}), which, thus, will tend to coincide.  
Our observations indicate that the eigenvalues of ${\bf C}(\theta)$ (the diagonal terms of ${\bf S}_{\fv}^2$) are fairly concentrated in the studied cases.
This implies that the features' coupling actually causes a {\em worst case scenario} for discriminating between similar $\theta$'s: the features' pdfs become (i) elongated (due to eigenvalues' concentration), and (ii) locally aligned with the $\bar \fv(\theta)$ curve. This causes strong overlapping of the pdf's having close $\theta$ values, and, as a result, poor discrimination. 
Fig.~\ref{fig:balls}(left) illustrates this phenomenon in a real experiment with real data. 
Fig.~\ref{fig:balls}(right) shows the effect of decoupling the kurtosis from the skewness (\textit{orthokurtosis}). We used 128 random vectors of 1024 i.i.d. samples each, from $x(\theta)=x_0^\theta$, being $x_0\sim U(0,1)$. Ellipses correspond to a Mahalanobis radius of 2, and $\theta=5$ (black), 6 (blue), and 7 (red). 
Expected error probabilities, using a bi-variate Gaussian model, are 12.4\% (coupled) vs. 4.5\% (decoupled). 
\begin{figure}
\begin{center}
\begin{tabular}{cc}
\hspace{-4mm}
\includegraphics[width=0.25\textwidth]{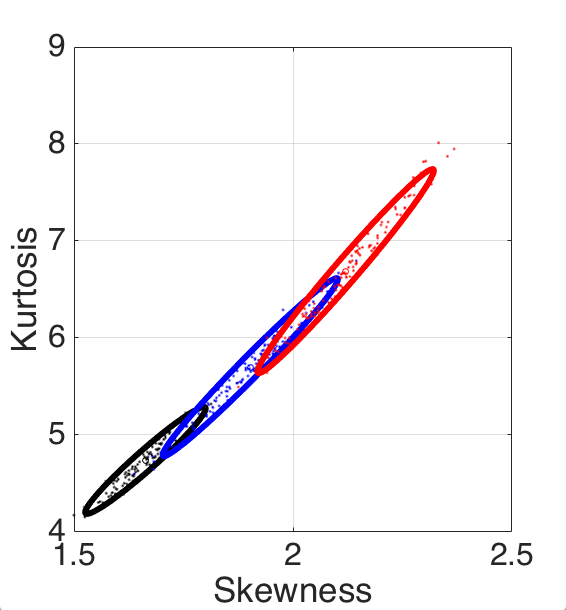}& 
\hspace{-6mm}
\includegraphics[width=0.25\textwidth]{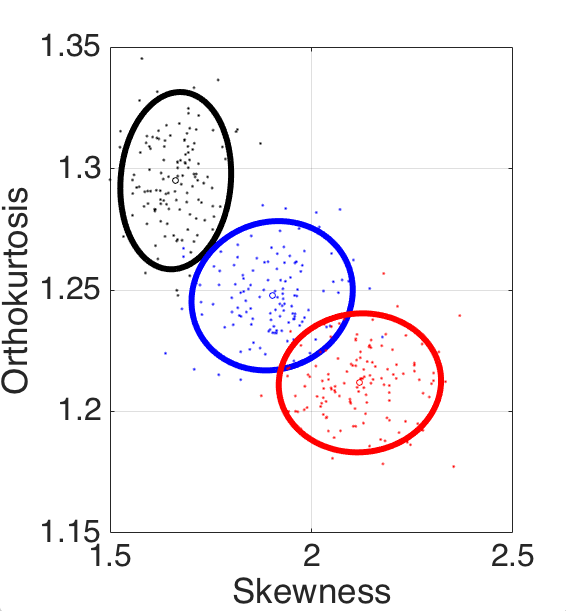}\\
\end{tabular}
\end{center}
\vspace{-4mm}
\caption{Comparing Gaussian classes in the original feature space (left) and in the decoupled feature space (right). See text for details.}
\label{fig:balls}
\end{figure}

To conclude this section, the techniques presented here attack the core of the poor discrimination problem due to using coupled features, by orthogonalizing the feature gradients, which has the effect of approximately diagonalizing the local covariance matrix ${\bf C}(\theta)$. 
It is crucial to note that this diagonalization is effective because it is {\em local}. A {\em global} diagonalization (such as the classical Principal Component Analysis, PCA) would be useless for reducing the pdfs' overlapping  corresponding to close $\theta$ values in the feature space, as such overlapping is insensitive to global affine transformations of that space. In contrast, a global linear correlation (as the one shown in Figure~\ref{fig:balls}(b) ) can be trivially removed, if needed, by applying PCA {\em after} feature decoupling.

\section{Experiments and Applications}
\label{sec:applications}

\subsection{Using two families of global features}

In this section we define the two features' families (sets) that will be studied in the experiments, namely: (i) marginal moments of arbitrary order $p$ (MM); (ii) $p$-th order moment at the output of a 2-D filter bank defined from a filter $h$ (MF). Specifically, to extract the MF features, we first applied the Translation Invariant Laplacian separable (TILs) representation~\cite{Portilla:TIP:2015}, a tight frame acting as bar and edge detector that provides nine subbands, each with the same number of coefficients as pixels in the image. Then, the $p$-th order moment was obtained for every subband. 
To obtain the corresponding decoupled sets of the MM and MF families ($\text{DF}_{\text{MM}}$ and $\text{DF}_{\text{MF}}$ respectively), we used the Nested Normalization-narrow path (Algorithm \ref{alg:nested_normalizations_narrow_path}). For the $\text{MM}$ set of features, we used as reference values the $p$-th moment of a standardized Gaussian distribution (i.e., $(p-1)!!$ for even $p$, 0 for odd \cite{Winkelbauer2012}). For the $\text{MF}$ set of features, reference values corresponded to moments obtained by convolving zero-mean univariate white Gaussian noise with a kernel $h$, which, in our case, using use the set of kernels $\{h_j, j=1\dots 9\}$ from the TILs representation, are the same function of $p$ as for the \text{MM} family in all subbands.
Table \ref{tab:feat} shows the original features for \text{MM} and \text{MF}. It also shows their corresponding gradients expressions (ignoring scaling factors which do not influence the result) and indicates in which cases the set of original gradients fulfills the Frobenius condition.

\begin{table}[h]
    \caption{Families of features (marginal moments, MM, and moments at the output of filters, MF), and their gradients.}
    \begin{tabular}
     {|c||c|c|c|}
    \hline 
    Family & $ f(\xv)$  & $ \partial f(\xv)/\partial x_i$ & Frobenius \\ 
    \hline
    \hline              
    MM & $\displaystyle 1/N \sum_{n=1}^N x_n^p $ & $ x_i^{(p-1)} $ & For all $p$\\
    \hline
    MF & $\displaystyle 1/N \sum_{n=1}^N (x\ast h)_{n}^{p} $ & $(x \ast h)^{(p-1)}\ast\tilde{h} $ & Only for $p=2$\\ 
    \hline 
    \end{tabular}
    \label{tab:feat}
\end{table}

\subsection{Measuring the amount of mutual coupling}
\label{subsec:gradients_orthogonality}
In this section we evaluate how close to being mutually orthogonal are the feature's gradients, for two sets of coupled features and their corresponding decoupled sets, namely: (i) a set composed of the first six orders of the classical MM, in its standardized version: mean, variance, and the rest the moments of the standardized sample to zero mean and unit variance (that is, skewness, kurtosis, etc.). We will refer to this classical set of statistical features by ``MSM" (from Marginal Standardized Moments), and $\text{DF}_{\text{MSM}}$ its corresponding decoupled set;  (ii) a set composed of the second-order moments at the output of a filter bank (``VF", a particular case of MF with $p=2$ and assuming zero mean) and its corresponding decoupled set ($\text{DF}_{\text{VF}}$). 

Let $\{f_j(\xv)\}$ represent the original features and $\{\g_j(\xv)\}$ its corresponding decoupled set. Note that $j = 1,\dots, 6$ for MSM and $j = 1,\dots, 9$ for VF (for the 9 subbands of the TILs representation).
Let $\mathbf{x_0}$ represent an $N$-D vector of i.i.d. samples drawn from a Gaussian or an uniform distribution; or an $N$-D vector that represents the pixel values of a texture patch extracted from an image of the Broadtz database~\cite{brodatz}. 
To obtain the gradient of a feature at $\mathbf{x_0}$, we numerically calculated the partial derivatives with respect to the $i$-th variable $x_i \in \xv$ by adding a differential perturbation $\epsilon$ to the $i$-th element of vector $\xv$: 
\begin{equation}
 \frac{\partial f(\mathbf{x_0})}{\partial x_i} =\lim_{\epsilon \to 0} \frac{f(\mathbf{x_0}+\epsilon \mathbf{e_i})-f(\mathbf{x_0}-\epsilon \mathbf{e_i})}{2\epsilon}.
\label{eq:partial_der}
\end{equation}
This expression yields the gradients  for each feature of the MSM, the VF, and their corresponding decoupled sets ($\text{DF}_{\text{MSM}}$ and $\text{DF}_{\text{VF}}$ respectively). 
To evaluate the function $\g_j(\xv)$ in the DFs cases we used the Nested Normalization-narrow path (Algorithm \ref{alg:nested_normalizations_narrow_path}) using equation \eqref{eq:ng3} for a fast analytical computation of the gradient's orthogonal projections.

Then we measured the angle $\alpha$ between pairs of gradient vectors of the different features that belong to the MSM and VF sets $\{f_j(\xv)\}$ and that belong to the $\text{DF}_{\text{MSM}}$ and $\text{DF}_{\text{VF}}$ sets $\{\g_j(\xv)\}$. The deviation from orthogonality (DO) was obtained as the difference between 90 degrees (perfect orthogonality) and the actual calculated angles ($\text{DO}=90-\alpha$). As such, $\text{DO}>0$ indicates acute angles and $\text{DO}<0$ obtuse angles.
The number of samples were $N$=512 for the Gaussian and uniform distributions and $N$=529 (23$\times$23 pixels) for textures. The experiment was repeated $M$=256 times for the Gaussian and uniform distributions. In the case where $\mathbf{x_0}$ came from textures, we used a single patch for each of the $M$=112 different textures 
in the Brodatz database.
Table \ref{tbl:error_results} shows the average of the absolute value of the DO across the different pairs of feature's gradients for the different distributions tested, for the MSMs, VFs, $\text{DF}_{\text{MSM}}$ and $\text{DF}_{\text{VF}}$.  We excluded from the average calculation the mean and the variance in the MSM and $\text{DF}_{\text{MSM}}$ cases, as they are orthogonal by definition.
Figures \ref{fig:figAE_MSM} and \ref{fig:figAE_VF} show the DO results. Panels (a) and (b) show the DO between different pairs of gradient's features for the original set and its decoupled counterpart (Gaussian and Textures cases respectively). Panels (c-f) show the DO, in absolute value, between different pairs of gradients. Panels (c) and (d) show results for the Gaussian case; panels (e) and (f) show the results obtained for the Textures case. Blue color indicates $|\text{DO}|$ close to 0 degrees (orthogonality), while yellow color indicates a deviation from orthogonality close to 90 degrees (angle of 0 degrees). Note that the main diagonal only acts as a reference ($0$ degrees) here.     

\subsubsection{Marginal moments}\label{subsubsection:empirical-gradientsMM} Let us analyze first the case of the marginal moments.
In agreement with the theory (the first three decoupled moments have gradients fulfilling Frobenius condition, see Proposition~\ref{prop:partial_frobenius}) we see that the DO is exactly zero (perfect orthogonality) for all standardized moments combined with orders 1 and 2 (note that the MSM set is already a partially decoupled version of the original MM family set), and for all decoupled moments combined with orders 1, 2 and 3. In addition, it is close to zero on average in the rest of odd-even combinations, in which case it presents a wide variance for MSM, and a much narrower one for $\text{DF}_{\text{MSM}}$. 
For the odd-odd and even-even rest of the cases, MSM presents highly acute angles (strong average coupling) with an extremely low variance, whereas DF provides either perfect (3-5 case) or approximate (4-6 case) orthogonality, also with low variance. 
In summary, the decoupled moments $\text{DF}_{\text{MSM}}$ provide close to orthogonal gradients also for orders greater than 3, especially for the random sampling experiments (samples are close to their expected values, in the reference manifolds, where exact orthogonality holds), but also, to a lesser extent and with higher variability, for photographic images (textures). In contrast, MSM presented in those cases either very high deviations (with low variability) or low-to-moderate average deviations with high variability (especially in real images).
 In Table~\ref{tbl:error_results} we can see that, for real photographic images, the average absolute DO has been reduced in a factor 6, approx., whereas for Gaussian samples it has been reduced 16 times.
 
\subsubsection{Second-order moments at the output of a filter bank}\label{subsubsection:empirical-gradientsVF}
Figure 8 shows the empirical results obtained using the VF features. First, we note that original features are fairly coupled, although not as much as in the marginal moments' case.
Now the new $\text{DF}_{\text{VF}}$ features are exactly and universally decoupled only for the eight pairs $(\g_1,\g_j),j=2\dots 9$, for which the theory tells us that the gradients of the new features achieve perfect orthogonality when Frobenius condition holds for the gradients of the original features (as it happens in this case; see Proposition~\ref{prop:partial_frobenius} and the note about the special case $\hat{\cal S}_1 = {\cal S}_1$). 
 Although this particular condition on the first eight pairs of gradients is difficult to appreciate in Fig. 8(a) and (d), where deviation from orthogonality seems approximately zero for all pairs of features using white Gaussian samples, we measured a RMS value of the deviation from orthogonality for the first eight couples of $1.6\times 10^{-5}$ degrees. This is a negligible  numerical error exclusively due to the numerical computation of the gradient (see Eq.~\eqref{eq:partial_der}). For the rest of feature couples we obtained an RMS of $0.42$ degrees, still small, but four orders of magnitude larger. Aside from the first eight pairs of features, the excellent practical decoupling of the other ones has been favored in this case by using a random distribution (zero-mean uni-variate white Gaussian noise) for the samples, with features whose expected values are precisely the reference values used in the NeN decoupling algorithm. As explained in the theory (Subsection~\ref{subsec:vs}), exact decoupling is also obtained for samples on the reference manifolds, even if Frobenius condition does not hold for the output features. For large enough samples, sample feature values do not deviate much from their expected values, and, as a consequence of the features' smoothness, the gradients of these samples will also be close to orthogonal.   
 To test how the decoupling quality degrades when using real samples instead of pseudo-random ones, we have tested the method, again, with the referred collection of 112 textured $23\times23$ pixel patches. We first note that the amount of mutual coupling between features is almost exactly the same as for the white noise case, a relevant fact that adds support to the assumptions made in Section~\ref{sec:local_decorrelation}: the angle between features is fairly constant for each couple of features, both within a given distribution/collection (see the relatively small amplitude of error bars) and also across different distributions/collections (compare panel (a) with (b) and (c) with (e)). 
 We can say that, in this case, it is especially accurate to say that the local covariance matrix is fairly independent of $\xv$, behaving the Jacobian pretty close to a moving frame. 
 As for the decoupling quality for this collection of real photographic patches, we now observe than in about half of the pairs the decoupling is either perfect (first eight couples) or almost perfect. For the other half, DO is still very moderate and much smaller than for the original pairs, in the great majority of cases. 
 In Table~\ref{tbl:error_results} we can see that, for real photographic images, the average absolute DO has been reduced in a factor 6, approx., whereas for Gaussian samples it has been reduced  65 times.
 
 \begin{table}[t]
	\begin{center}
	\caption{\label{tbl:error_results} Average ($\pm\sigma$) absolute deviation from orthogonality ($|\text{DO}|$), in degrees, between feature's gradients for the MSM, VF, $\text{DF}_{\text{MSM}}$ and $\text{DF}_{\text{VF}}$ sets of features.}
	\begin{tabular}{|c||c|c|}
	\hline 
				           & Gaussian & Textures \\

\hline
\hline

MSM    & 32$\pm$26 & 59$\pm$21\\
  \hline 
$\text{DF}_{\text{MSM}}$   & 2$\pm$3 & 9$\pm$17\\
\hline
\hline 
VF    & 13$\pm$10 & 12$\pm$11\\
\hline 
$\text{DF}_{\text{VF}}$   &0.2$\pm$0.3  & 2$\pm$3\\
\hline 
	\end{tabular}
	\end{center}
	\label{tab:error_angle}
\end{table}

\begin{figure*}[ht]
\begin{subfigure}{.45\textwidth}
  \centering
  \includegraphics[width=1\linewidth]{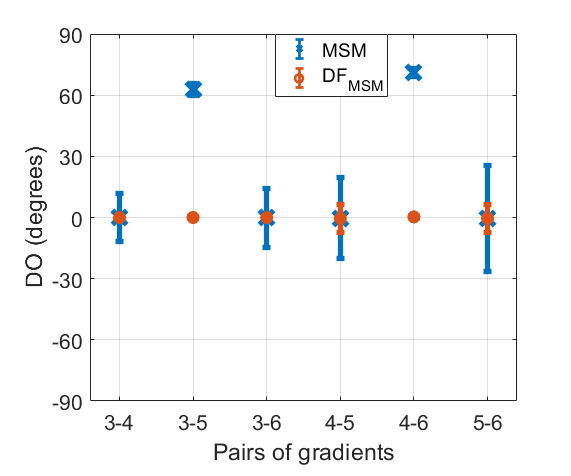}  
  \caption{DO between different pairs of gradients. Gaussian case.}
  \label{figAE:sub-first}
\end{subfigure}
\begin{subfigure}{.45\textwidth}
  \centering
  \includegraphics[width=1\linewidth]{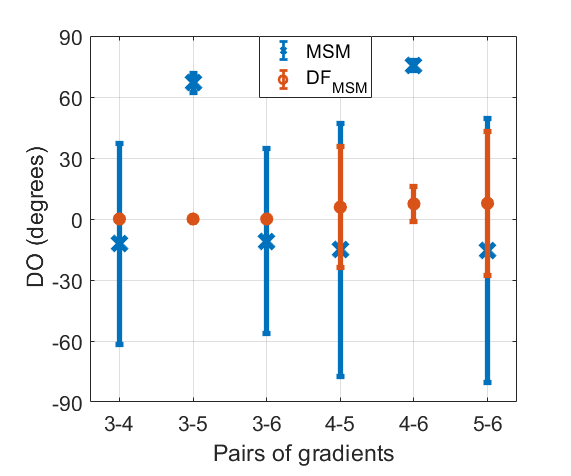}  
  \caption{DO between different pairs of gradients. Textures case.}
  \label{figAE_tex:sub-first_textures}
\end{subfigure}
\begin{subfigure}{.24\textwidth}
  \centering
  \includegraphics[width=1\linewidth]{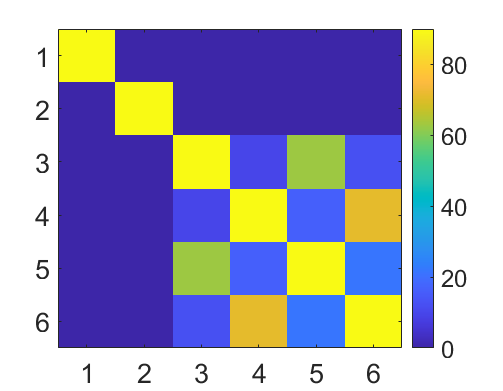}  
  \caption{MSM set. Gaussian case.}
  \label{figAE:sub-second}
 \end{subfigure}
 \begin{subfigure}{.24\textwidth}
  \centering
  \includegraphics[width=1\linewidth]{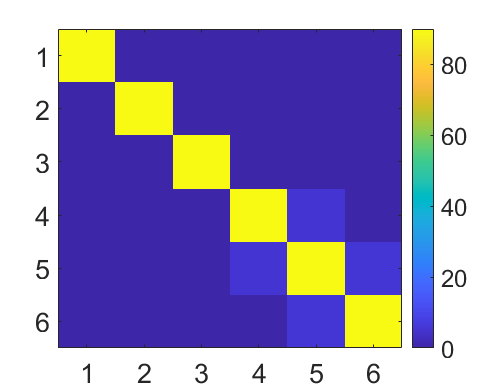}  
  \caption{$\text{DF}_{\text{MSM}}$ set. Gaussian case.}
  \label{figAE:sub-third}
 \end{subfigure}
 \begin{subfigure}{.24\textwidth}
  \centering
  \includegraphics[width=1\linewidth]{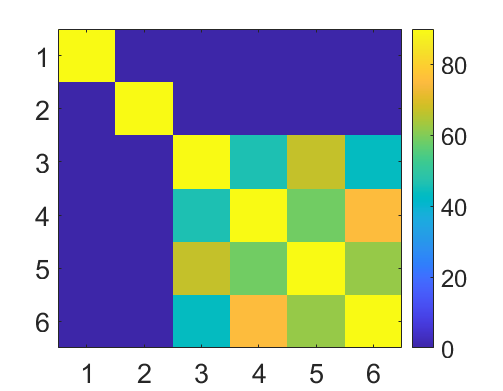}  
  \caption{MSM set. Textures case.}
  \label{figAE_tex:sub-second_textures}
 \end{subfigure}
 \begin{subfigure}{.24\textwidth}
  \centering
  \includegraphics[width=1\linewidth]{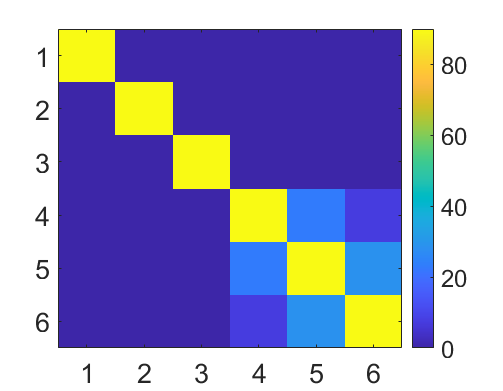}  
  \caption{$\text{DF}_{\text{MSM}}$. Textures case.}
  \label{figAE_tex:sub-third_textures}
 \end{subfigure}
 
\caption{Deviation from orthogonality (DO), in degrees, for the different pairs of gradient's features for the standardized marginal moments (MSM) and the corresponding decoupled sets ($\text{DF}_{\text{MSM}}$). (a), (c) and (d) show results for the Gaussian case, $N=512$. (b), (e) and (f) show results for the Textures case, $N=529$ (23 $\times$ 23 pixels). In (a)-(b) positive values indicate acute angle and negative obtuse. In (c-f) the color level represents the average of the absolute value of DO.} 
\label{fig:figAE_MSM}
\end{figure*}

\begin{figure*}[h]

\begin{subfigure}{\textwidth}
  \hspace{-2cm}
  \includegraphics[width=1.2\linewidth]{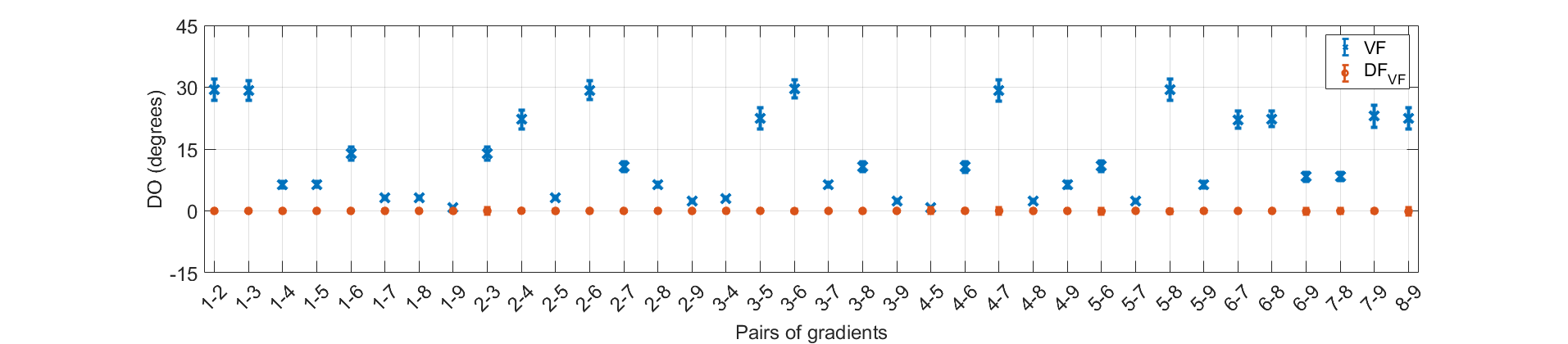}  
  \caption{DO between different pairs of gradients. White Gaussian noise case.}
  \label{figAE:sub-first_tex}
\end{subfigure}

\begin{subfigure}{\textwidth}
  \hspace{-2cm}
  \includegraphics[width=1.2\linewidth]{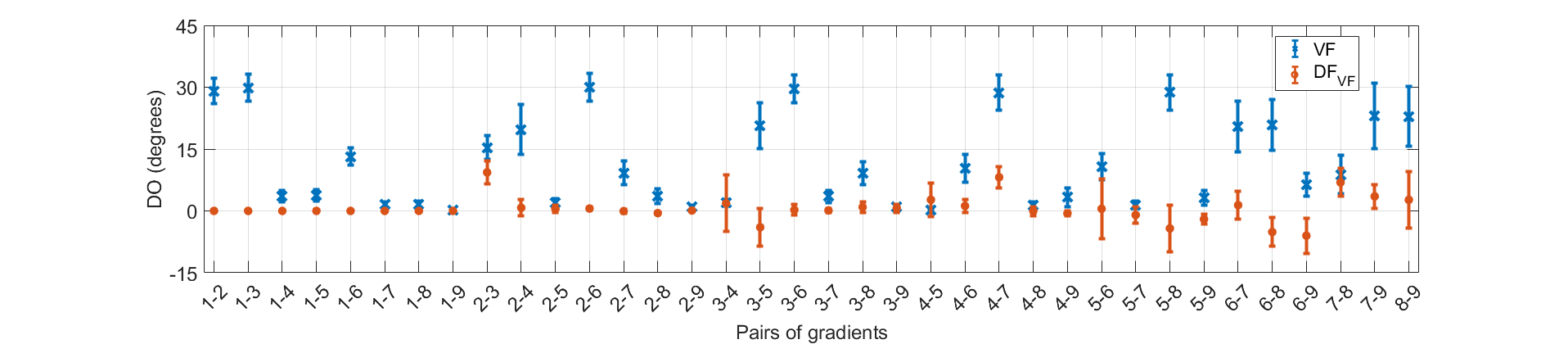}  
  \caption{DO between different pairs of gradients. Textures case.}
  \label{figAE_tex:sub-first_textures_2}
\end{subfigure}

 \begin{subfigure}{.24\textwidth}
  \centering
  \includegraphics[width=1\linewidth]{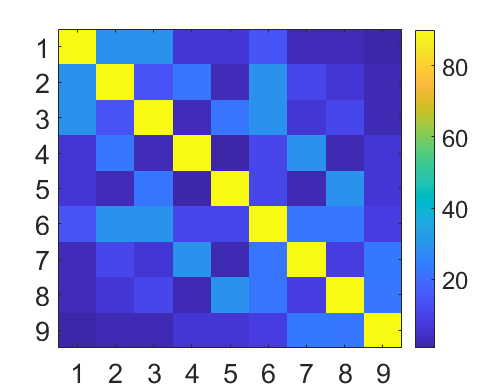}  
  \caption{VF, White Gaussian case.}
  \label{figAE:sub-second_filt}
 \end{subfigure}
 \begin{subfigure}{.24\textwidth}
  \centering
  \includegraphics[width=1\linewidth]{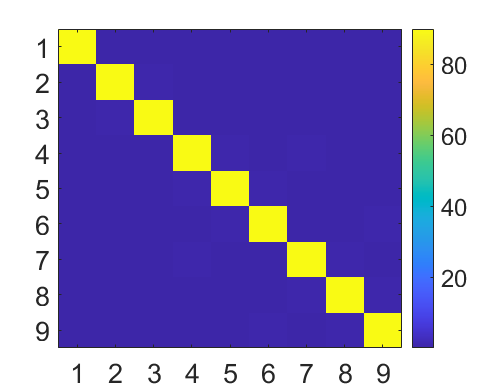}  
  \caption{$\text{DF}_{\text{VF}}$, White Gaussian case.}
  \label{figAE:sub-third_filt}
 \end{subfigure}
 \begin{subfigure}{.24\textwidth}
  \centering
  \includegraphics[width=1\linewidth]{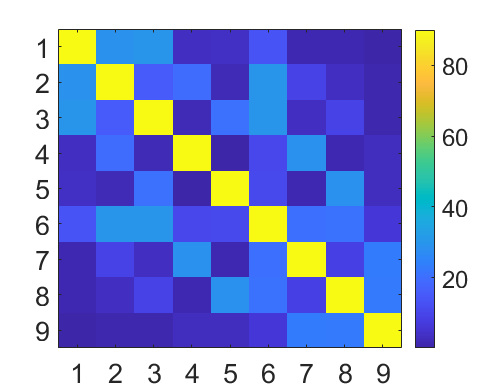}  
  \caption{VF, Textures case.}
  \label{figAE_tex:sub-second_filt_textures}
 \end{subfigure}
 \begin{subfigure}{.24\textwidth}
  \centering
  \includegraphics[width=1\linewidth]{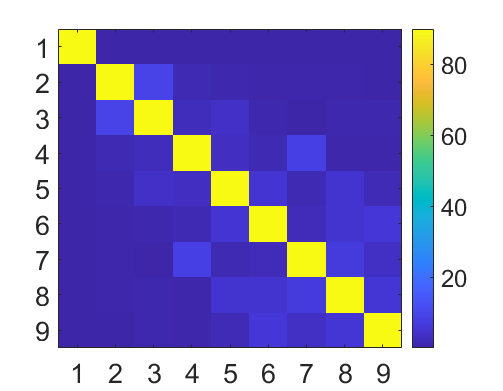}  
  \caption{$\text{DF}_{\text{VF}}$, Textures case.}
  \label{figAE_tex:sub-third_filters}
 \end{subfigure}
\caption{Deviation from orthogonality (DO), in degrees, for the different pairs of gradient's features for the second-order moment at the output of a filter bank (VF), and the corresponding decoupled set ($\text{DF}_{\text{VF}}$). We used  $N=529$ (23 $\times$ 23 pixels). (a), (c) and (d) show results for the Gaussian case. (b), (e) and (f) show results for the Textures case.  In (a)-(b) positive values indicate acute angle and negative obtuse. In (c-f) the color level represents the average of the absolute value of DO.}
\label{fig:figAE_VF}
\end{figure*}

\subsection{Statistical regression}
\label{subsec:regression}
In this section we focus on the estimation of the parameters of a distribution that best describe a dataset\footnote{Some results of this subsection have been presented in~\cite{Martinez:ICASSP:2020}.}. Different approaches have been proposed in the literature for different families of parametric distributions, such as the classical maximum likelihood estimation or the method of moments.    
We approached the estimation task as a regression problem, where the parameters of the distribution are estimated from a set of global features obtained from the observed data. Specifically, we tested, as global features: (i) a set of classical MSM, and (ii) the corresponding decoupled set ($\text{DF}_{\text{MSM}}$). In order to thoroughly compare the descriptive capabilities of both sets of features, we used several regression methods, namely: linear regression models (LRM), regression trees (RT), support vector regression (SVR), Gaussian process regression (GPR), ensembles of trees (ET) and neural networks (NNR). All of these methods are implemented in the Regression Learner App, \textregistered Matlab. For reproducibility purposes, we used the hyper-parameters set by default in the referred app. Specific information about implementation, hyper-parameters selection and methodological details can be found in ~\cite{Matlab}. 

In our experiments we used different statistical distributions, specifically: generalized Gaussian distribution (GGD), Gamma distribution (GMD), and absolute value of a Normal distribution raised to a positive number (GND). The shape of these distributions depends on a shape parameter ($\beta$), and the regression problem consists in estimating  $\beta$ from an observed data set. See Appendix \ref{appendix:reg-exp} for the expressions of the probability density functions of these distributions and their dependence with $\beta$.

Let $\xv$ represent an $N$-D vector of i.i.d. samples drawn from a GGD, GMD or GND distributions (we generated the samples following~\cite{Generalized_Gaussian1,Generalized_Gaussian2} for the GGD, and~\cite{Gamma_sampling} for the GMD, using the \textregistered Matlab function {\em gamrnd.m}), normalized to have zero mean and unitary variance.
For the GGD and GMD cases, given that the kurtosis of these distributions changes very fast for small values of the $\beta$ parameter, we defined $\beta=2^{A}$ and sampled uniformly the exponent $A$ in $[-3, 3]$, resulting in $\beta$ ranging in $[1/8,8]$. In this way, we obtained a quite uniform distribution of kurtosis values. In the GND case, $\beta$ was sampled uniformly in the range $[1, 6]$. 

MSM features and their corresponding decoupled features, $\text{DF}_{\text{MSM}}$, were obtained from the third (skewness) to the sixth order. 
This led to two sets of 4 dimensional predictors $\{f_j(\xv), j = 3,\dots, 6\}$ and $\{\g_j(\xv), j = 3,\dots, 6\}$ for each parameter $\beta$. We compared the $\beta$ prediction accuracy of these two sets. 
We generated $d=2048$ vectors $\xv$ of different lengths $N$ of i.i.d. samples, thus having $2048$ pairs of predictors ($\{f_j\}$ or $\{\g_j(\xv)\}$) and targets (known $\beta$ values). 
We averaged 100 5-fold cross-validation runs to measure the accuracy of the methods in terms of the RMSE in the estimation of the exponent $A=\log_2\text{$\beta$}$ in the GGD and GMD cases, and the estimation of $\beta$, in the GND case. 

\begin{figure*}[ht]
\begin{subfigure}{.33\textwidth}
  \centering
  \includegraphics[width=1\linewidth]{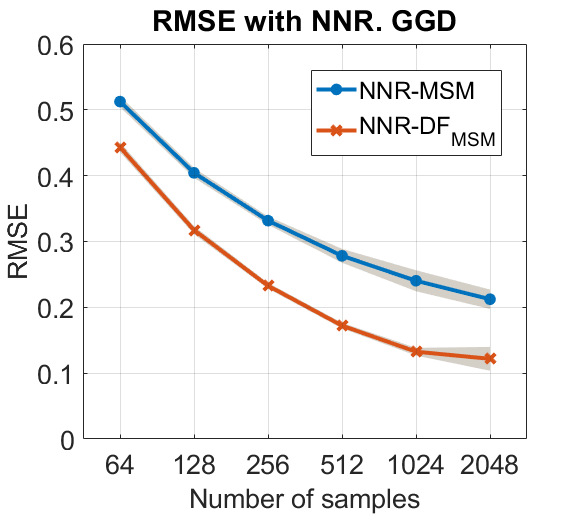}  
  \caption{GGD distribution}
  \label{figRMSE:sub-first}
\end{subfigure}
\begin{subfigure}{.33\textwidth}
  \centering
  \includegraphics[width=1\linewidth]{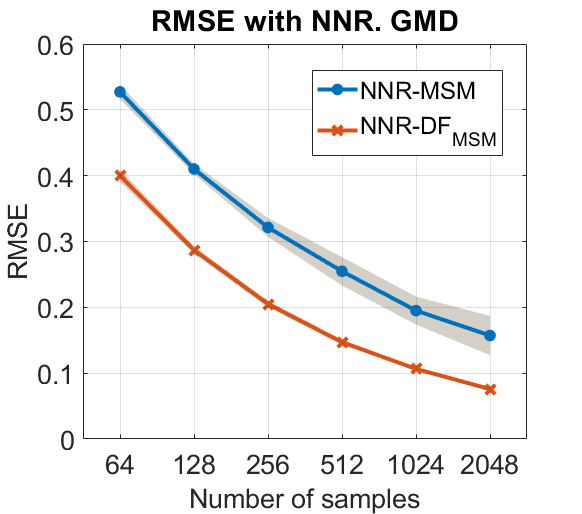}  
  \caption{GMD distribution}
  \label{figRMSE:sub-second_2}
  \end{subfigure}
 \begin{subfigure}{.33\textwidth} 
    \includegraphics[width=1\linewidth]{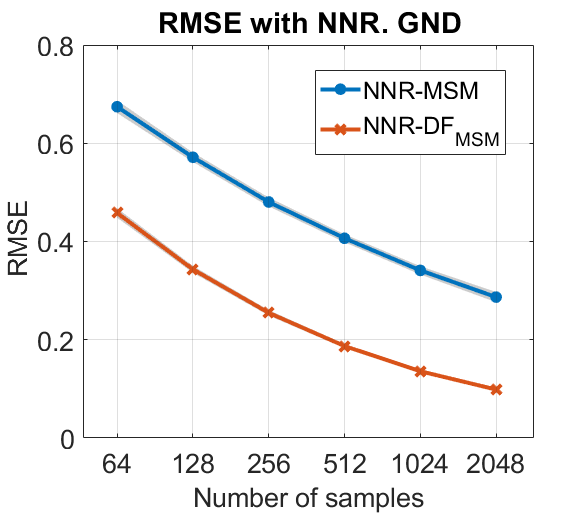}  
  \caption{GND distribution} 
  \label{figRMSE:sub-third_2}
\end{subfigure}
\caption{RMSE of the estimated parameters as a function of the number of samples $N$ for the NNR regression method applied to three density functions (see text for details).}
\label{fig:figRMSE}
\end{figure*}
Figure \ref{fig:figRMSE} shows the RMSE of the results as a function of $N$ for the MSM and $\text{DF}_{\text{MSM}}$ sets of descriptors, by using the NNR regression method, the method providing the best results for almost all the tested cases (see Appendix~\ref{appendix:reg-exp} for more detailed results). Figure \ref{fig:figRMSE}(a) shows the results for GGD, Figure \ref{fig:figRMSE}(b) for GMD and Figure \ref{fig:figRMSE}(c) for GND.
The shadow area represents the standard deviation across 100 repetitions of the experiment.
The proposed $\text{DF}_{\text{MSM}}$ clearly outperformed MSM for all the tested distributions and sample sizes. For instance, in the GGD case, using the proposed $\text{DF}_{\text{MSM}}$ the RMSE was reduced by factors of $0.50$, $0.54$, $0.63$, $0.68$, $0.76$, and $0.86$, for $N=2048, 1024, 512, 256, 128$, and $64$, respectively. 

Tables \ref{tbl:GMD_results}, \ref{tbl:GG_results} and \ref{tbl:GN_results} in Appendix \ref{appendix:reg-exp} show the RMSE obtained using MSM and DF descriptors for the different sample sizes $N$ and regression methods, for the GGD, GMD and GND distributions, respectively. The best regression method for each $N$ is highlighted in bold.
Our decoupled descriptors $\text{DF}_{\text{MSM}}$ outperformed the MSM across all the compared regression methods and sample sizes in 105 out of 108 cases ($97.2\%$), which shows the robustness and generalization ability of our approach.

\subsection{Texture classification}
\label{subsec:texture_classification}
In this section we apply the proposed method for texture classification in two different settings: (i) comparing standardized moments (MSM) with their corresponding decoupled features $\text{DF}_{\text{MSM}}$, both applied to a set of subbands, the output of a filter bank (Section \ref{subsec:marg_mom}); (ii) using features that directly are defined as marginal moments at the output of a filter bank (MF), and their corresponding fully decoupled features, $\text{DF}_{\text{MF}}$ (Section \ref{subsec:dec_TIL}).

\subsubsection{Texture classification based on marginally decoupled moments in a filter bank}
\label{subsec:marg_mom}
Here\footnote{This subsection is a summary of the results published in~\cite{Portilla:ICIP:2018}.} we compare the performance of two classifiers using features derived (i) from MSM of order 2nd to 10th, (ii) from a generalized form of Ref.~\cite{Blest2003}, and (iii) from the corresponding decoupled features $\text{DF}_{\text{MSM}}$, all of them at the output of subbands of the TILs representation~\cite{Portilla:TIP:2015}. We selected 54 textures of 640$\times$640 pixels from the Broadtz \cite{brodatz} database under the criterion of looking homogeneous in 64$\times$64 pixel patches\footnote{The list of selected textures and the \textregistered Matlab code of the experiments is available 
in https://www.researchgate.net/project/Nested-Normalizations-for-Decoupling-Global-Features.}.
Every texture was divided into $10\times 10$ disjoint $64\times64$ patches. The problem consisted in classifying patches in their corresponding textures. 
To extract the features we first applied the TILs representation~\cite{Portilla:TIP:2015} and discarded the low-pass band (having 8 subbands). 
Then, for every subband of every patch, three sets of features were obtained: (i) classical MSM features $\{f_j(\xv), j = 2,\dots, 10\}$ 
(something commonly used to characterize textures, but to a lower order, see e.g.\cite{Mandal96imageindexing}); (ii) modified moment set: same as MSM, except that now even order moments are shift-minimized (${\tilde \mu_n} = \min_{\alpha_n} \E\{({\hat x} - \alpha_n)^n\}$, ${\hat x}$ being the standardized observation)~\cite{Blest2003}; and (iii) proposed marginally decoupled moments $\text{DF}_{\text{MSM}}$ $\{\g_j(\xv), j = 2,\dots, 10\}$. 

In order to quantify the redundancy between features, we estimated the mutual information ($\text{MI}$) between pairs of features~\cite{Peng:PAMI:2005}. Some mean values across subbands are shown in Table~\ref{tbl:MI_results}.
\begin{table}[t]
	\begin{center}
	\caption{\label{tbl:MI_results} Averaged mutual information.}
	\begin{tabular}{|c||c|c|c|c|c|c|}
	\hline 
		Features				           & (3,4) & (3,5) & (4,5) & (3,6) & (4,6) & (5,6)\\
  \hline 
  \hline 
 MSM   & 0.26 &  0.61  & 1.39 & 0.55 & 3.10 &  3.11   \\         
  \hline 
Mod.~\cite{Blest2003}        & 0.23 & 0.61 & 1.39 & 0.55 & 3.07 & 3.03       \\ 
\hline
$\text{DF}_{\text{MSM}}$ & 0.22 & 0.04 & 0.08 & 0.09 & 0.54 & 0.09 \\
\hline
	\end{tabular}
	\end{center}
	\vspace{-0.7cm}	
\end{table}
We see that, except for the (3,4)-th case, the proposed $\text{DF}_{\text{MSM}}$ features present a drastic reduction of redundancy compared to the classical MSM and modified~\cite{Blest2003} ones. 
We recall the reader that the aim of using decoupled features is not to compensate for the statistics of the data (to which the decoupling method is totally transparent), like non-linear ICA aimes for, but rather to avoid adding spurious coupling in the processed data, leaving only the dependencies that are effectively caused by the input data statistics. So, perfectly decoupled features will generally reduce the mutual information among features, but will not necessarily remove it, as that depends on the data statistics (see, e.g., what happened in the example of Fig.~\ref{fig:balls}). 
Figure \ref{fig:classification} shows a 3-D subset of the MSM (left) and proposed $\text{DF}_{\text{MSM}}$ (right) features (i.e., before and after our decoupling method) for two texture classes, to illustrate the decoupling effect on the data distribution in the feature space.  
Shown data correspond to a single subband (number 5 in the representation),
of textures D103 and D111,
selecting just 3 features for the MSM ($f_3(\xv)$, $f_4(\xv)$ and $f_5(\xv)$\footnote{Note that $f_3(\xv)$ and $f_4(\xv)$ are equal to the classical definition of sample skewness and kurtosis, respectively.}) and the proposed $\text{DF}_{\text{MSM}}$ ($\g_3(\xv)$, $\g_4(\xv)$ and $\g_5(\xv)$) sets. We include the projections onto the three orthogonal planes. 
The decoupling between every pair of features is apparent, specially for the (3,5)-th and (4-5)-th order cases.

\begin{figure*}[ht]
\centering
\includegraphics[width=0.96\linewidth]{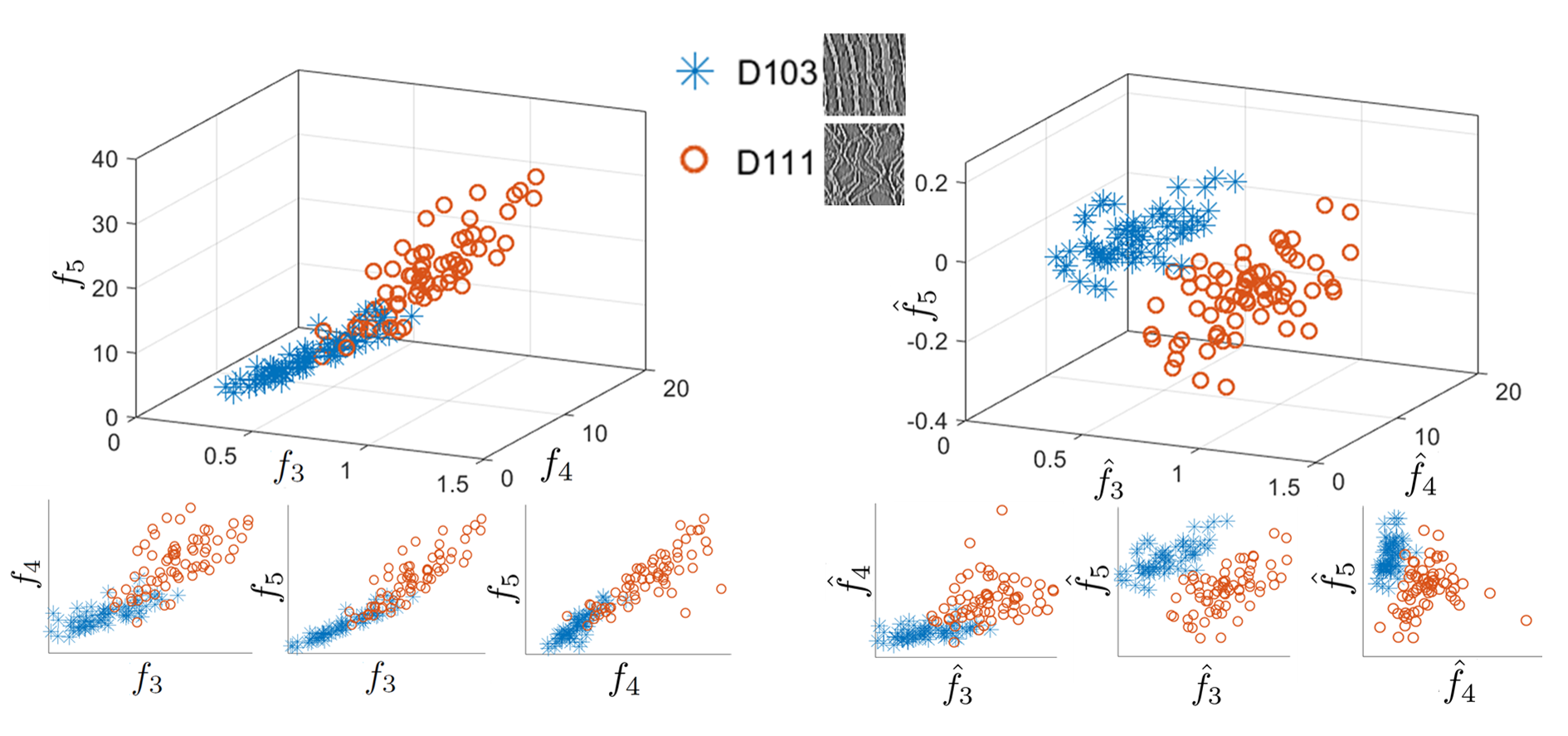}  
\vspace{-0.3cm}
\caption{Illustrating the decoupling effect in the feature space on two classes of texture patches. Left: Using classical MSM features; Right: Using proposed $\text{DF}_{\text{MSM}}$ features. }
\label{fig:classification}
\end{figure*}

We used two classifiers: a na\"\i ve (univariate) Gaussian and a parameter-optimized Support Vector Machine (SVM) using Radial Basis Functions\footnote{We thank the authors of the Pattern Recognition Toolbox (http://covartech.github.io), which we used in preliminary experiments.}. 
We applied cross validation with 4 folds, and averaged 8 runs for each result.
Figure~\ref{fig:results} shows the test classification results as a function of the order of the moments included in the feature's set for the three compared sets and the two classifiers (see legend).
We observe a totally different behavior between using classical MSM and modified~\cite{Blest2003} moments,  vs. the marginally decoupled ones $\text{DF}_{\text{MSM}}$: whereas the former achieve their optima when using only variance ($f_2(\xv)$), skewness ($f_3(\xv)$) and kurtosis ($f_4(\xv)$), roughly achieving 3\% and 2\% error ratios for na\"\i ve and SVM, respectively,  the latter keep on decreasing the error when adding higher-order features, reaching $1.34\%$ and $0.86\%$, respectively, for $n=10$. We see how~\cite{Blest2003} produces just a marginal improvement. 
It is also very significant how, for MSM, the na{\"\i}ve method behaves very differently from SVM, whereas, for $\text{DF}_{\text{MSM}}$, results of SVM and na\"\i ve classifiers run in close parallel.
We believe this is due to the strong assumption made by the na{\"\i}ve method (namely, that features are mutually independent), which holds approximately true {\em after} the feature decoupling (as shown in Table 3), but not before.
\begin{figure}
\centering
\includegraphics[width=1.0\linewidth]{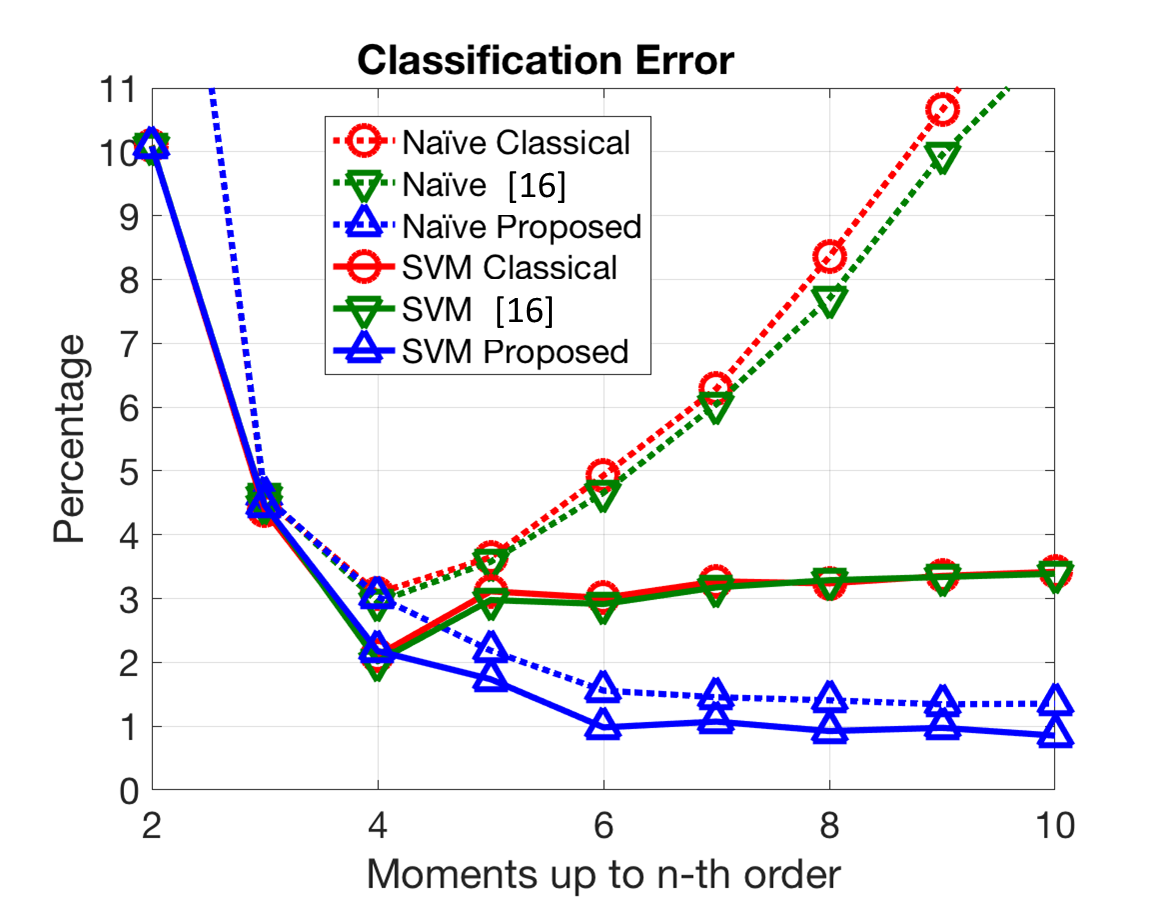}  
\caption{Texture patch classification, 1st experiment. Test classification error (\%) as a function of the order of the moments included in the feature's set.} 
\vspace{-0.1cm}
\label{fig:results}
\end{figure}

\subsubsection{Texture classification based on jointly decoupled moments in a filter bank}
\label{subsec:dec_TIL}

In these experiments we compared the performance of a classifier trained using three different sets of features, namely: (i) MSM measured at the output of a filter bank;
(ii) the decoupled standardized moments $\text{DF}_{\text{MSM}}$ (same as in the previous experiment), without considering the coupling caused by the filter bank; and (iii) fully decoupled features, $\text{DF}_{\text{MF}}$, using as original features marginal moments at the outputs of the filter bank. 
The three sets were obtained from order 2 to 6. 
The first 20 textures from the Broadtz database \cite{brodatz} were taken. 
The upper left quarter of every texture was normalized and divided into 25 disjoint, $64 \times 64$ pixel, patches. 
The problem consisted of classifying these patches into one of the 20 texture classes. 
To extract the features that characterized every patch we first applied the TILs representation, using the 9 subbands (thus including the low-pass residual, not used in previous section). Then, for every subband, the above described sets of 45 features (5 moments $\times$ 9 subbands) were calculated.
We used two classifiers: a Na\"\i ve Bayes and a parameter-optimized Support Vector Machine (SVM) using Gaussian kernels \cite{Burges1998}.
We applied cross validation with 10 folds, and averaged 200 runs for each result.
Table \ref{tab:Results_class} shows the mean $\pm$ standard deviation test classification error across runs when using the MSM (1st column), the partially decoupled $\text{DF}_{\text{MSM}}$ \cite{Portilla:ICIP:2018} (2nd column), and the fully decoupled features $\text{DF}_{\text{MF}}$ 
(3rd column). First row shows results using  Na\"\i ve Bayes, and second using SVM.
\begin{table}[t]
	\begin{center}
	\caption{\label{tab:Results_class} Texture patch classification, 2nd experiment. Test classification error (\%).}
	\setlength{\tabcolsep}{5pt}
  \scalebox{1}{
	\begin{tabular}{|c|c|c|c|}
	\hline 
						           & MSM & $\text{DF}_{\text{MSM}}$  \cite{Portilla:ICIP:2018}  &$\text{DF}_{\text{MF}}$ \\
  \hline 
  Na\"\i ve Bayes   	   & 9.1 $\pm$  0.4 &  7.4 $\pm$ 0.3   & 4.0 $\pm$ 0.3        \\         
  \hline 
	SVM         &  5.2 $\pm$ 0.5 & 4.5 $\pm$ 0.4   & 3.1 $\pm$ 0.4   \\ 

	\hline

	\end{tabular}}
	\end{center}
	\vspace{-0.7cm}	
\end{table}
We observed that the probability of error is greatly reduced when using $\text{DF}_{\text{MSM}}$ in comparison with MSM. In addition, fully decoupled features $\text{DF}_{\text{MF}}$ also significantly improved the results of the partially decoupled set~\cite{Portilla:ICIP:2018}. 
Using the Na\"\i ve Bayes classifier, we see that the error obtained by the fully decoupled set $\text{DF}_{\text{MF}}$ decreased the error in a factor of 0.54 and 0.44 w.r.t. $\text{DF}_{\text{MSM}}$ and original MSM sets, respectively.
We also did the experiment of incrementally increasing the highest order in the set of features, from just the variance to including all orders from 2 to 6.
Significantly, the error was minimized when including high order $\text{DF}_{\text{MF}}$, up to 6, while for the MSM case the minimum error was obtained when including features up to the fourth order (kurtosis).
This result is consistent with the one reported in Fig.\ref{fig:results}. 

It is especially relevant, at this point, to recall that the MF set (for orders above 2) does not fulfill Frobenius condition (as mentioned in Section~\ref{sec:study_cases}). This implies that decoupling solutions do not even exist within the theoretical framework presented here. Nevertheless, Algorithm~\ref{alg:nested_normalizations_narrow_path} (NeN, narrow path) provides a well-defined transformation that, as empirically shown in Section~\ref{subsubsection:empirical-gradientsVF} for the case of the second-order moments at the output of a filter bank, greatly reduces the amount of coupling between the features' gradients. This translates here into a substantial performance boost when applied to texture classification.

\section{Conclusions and Future Work}
\label{sec:conclusion}
We have presented a new mathematical and algorithmic framework for, given a set of differentiable functions acting as global data descriptors, obtaining a closely related set such that its gradients are mutually orthogonal. We have set the conditions under which this can be done hierarchically and progressively, adding a new feature at a time, and devised a new family of algorithms based on nested normalization operations. We have also studied the need of adding small perturbations in some cases, and devised some specific methods for that.

We have shown, first, that the proposed method allows for locally decorrelating features of statistical distributions, and why this has a positive impact in discriminating close values of statistical parameters. 
We have also tested empirically (with both real textured image patches and pseudo-random numbers, for some distributions) the degree of accomplishment of mutual decoupling (gradient orthogonality) for different global features, obtaining results that are both practically interesting and consistent with the theory.

We have applied our decoupling methods to marginal moments, both in the pixel domain and at the output of a filter bank. Using the new decoupled features as descriptors with state-of-the-art machine learning methods we have obtained a dramatic decrease in error in regression problems (over simulated random variables, under three different distributions) and classification (over real textured image patches), as compared to using classical standardized moments. It is worth noting that we obtained substantial improvements in classification accuracy even when theoretical conditions for perfect decoupling were not met. 

It is noteworthy that, applying our decoupling method to the first three raw moments results into their standardized counterparts. In addition, we obtained for the first time, a (quasi-analytic) normalized fourth-order moment that is decoupled from mean, variance, and skewness, which we have termed {\em orthokurtosis}. For higher orders, we obtained other new well-defined features, numerically computable, which present a much-decreased amount of coupling. Finally, although not fully developed here, natural application fields for our method go beyond analysis and include promising style transfer and synthesis techniques~\cite{MMSP2020,SIIMS2022}.

In future work we want to explore the extension of the deterministic decoupling methods for non-hierarchical sets, as well as its applicability for improving economy and robustness in ANNs, by decoupling their features, either at training phase or {\em a posteriori}. To conclude, a promising research area is that of exploring actual biological mechanisms of adaptation involving decoupling/cross-invariance, in perceptual neural science. 

\appendices

\section{Equilibrium points of gradient systems}
\label{appendix:grad-systems}

In this section we study equilibrium points for the gradient system \eqref{grad-system}.

If $f$ is continuous we must have $f(\overline\Omega)=[a,b]$ (here $a,b$ maybe $\pm\infty$ if the set is not bounded, but the analysis can be performed similarly). We would like to see if, given $x_0\in \Omega$, we can reach all values in the range of $f$ by moving in the direction of the gradient.
We look at stationary points in $\overline \Omega$ that are not a global maximum or minimum points.  The basin (or domain) of attraction of an equilibrium point $\bar \xv$ is the set of all initial conditions with
solutions that tend to it. Let us consider all possible cases according to the behavior of the linearized equation  $\xv=-D^2f (\bar\xv)\xv$ (see, for instance, the introduction in \cite{ODEbook}). Assume that $D^2f $ has no zero eigenvalues:
\begin{itemize}
\item If the eigenvalues of the Hessian $D^2f$ are all strictly positive then $\bar \xv$ is a sink.
\item Similarly for max.
\item If the Hessian has a negative real eigenvalue then the equilibrium is unstable. If all eigenvalues are nonzero, then the dimension of the unstable manifold is equal to the number of negative eigenvalues counting multiplicity. The dimension of the unstable manifold is the number of negative eigenvalues counting multiplicity. The tangent spaces of these manifolds are the spans of the corresponding eigenspaces so are orthogonal at the equilibrium point.
\end{itemize}
If some eigenvalue is zero, then the picture is much more complicated and we will make additional assumptions in order to avoid technicalities:
\begin{itemize}
\item If  $\bar x$  is a local minimum of $f$ (not necessarily strict), then $\bar x$ is a Lyapunov-stable equilibrium point of the gradient flow of a real-analytic function $f$ \cite[Section 3]{Simon}. Thus we 
pose the additional constraint
that all local maxima and minima are also global.

\item If 
$\xv$ is a degenerate saddle, we assume that its basin of attraction has a lower dimension and thus, saddles can be avoided by introducing a suitable perturbation. We will actually assume that the union of all basins of attraction, denoted by $\Lambda$, is also lower dimensional.
\end{itemize}

\section{Frobenius Theorem}\label{appendix:Frobenius}

Let  $\S= \{f_i:\Omega \rightarrow \R, i = 1\dots M \}$ be a set of features.
Fix $p$ a positive integer, $p\leq M$. A $p$-dimensional \emph{distribution} $\mathcal D$ in $\Omega$ is a (smooth) choice of a $p$-dimensional subspace of the linear span of the $\{\nabla f_i(\xv)\}_{i=1}^M$
 for every point in $\xv\in\Omega$. We denote the plane at $\xv$ by $\mathcal D(\xv)$.  We say that the set of features  $\mathcal S$ is of \emph{maximal rank} in $\Omega$ if the gradients are linearly independent at every point $\xv$ so that the distribution spanned by the gradients has $p=M$. This implies, in particular, that we must choose $\Omega$ not containing critical points of the $f_i$.

We say that $\mathcal D$ satisfies the \emph{Frobenius condition} if for every $f,g\in\mathcal S$, writing the gradient vectors as
\begin{equation}\label{Frobenius-XY}
X=\sum_\ell\partial_{x_\ell}f\frac{\partial}{\partial x_\ell},\quad Y=\sum_j\partial_{x_j}g\frac{\partial}{\partial x_j},
\end{equation}
we have that the combination
\begin{equation}\label{Frobenius-check}
XY-YX=\sum_{j,\ell}\left(\partial_{x_\ell}f \partial_{x_\ell x_j}g-\partial_{x_\ell}g\partial_{x_j x_\ell}f\right)\frac{\partial}{\partial x_j}
\end{equation}
also belongs to $\mathcal D$.

The classical Frobenius theorem \cite[Chapter 6]{Spivak} states that  a distribution $\mathcal D$ that satisfies the Frobenius condition in $\Omega$ is integrable. 
That is, one can find a submanifold $\I(\xv_0,\S)$ passing by $\xv_0$ whose tangent hyperplane at each location $\xv$ is the linear span of $\{\nabla f_i(\xv), i=1\dots M\}$. In addition, this submanifold is $p$-dimensional, and $\Omega$ is foliated by these submanifolds. A different presentation of Frobenius theorem from the dynamical systems point of view can be found in \cite[Chapter VI]{Hartman}.\\

\subsection{Back to Proposition \ref{prop:iff}}\label{appendix:iff}
We give now the proof of the {\em only if} statement  in Proposition \ref{prop:iff}. Assume that, given set of features $\S$, there exists an invariant mapping $\hat \xv_\S$ exists, and denote  by $J:=J_{\hat \xv_\S}$ its Jacobian matrix. We will show that the gradients of those features in $\S$ satisfy the Frobenius condition C2.  

For this, it is enough to pick two any two features in $\S$, say $f$ and $g$. By the definition of invariant mapping, we must have
\begin{equation*}
\begin{split}
&J \cdot \nabla f=\textbf 0,\\
&J \cdot \nabla g=\textbf 0,
\end{split} 
\end{equation*}
at each point in the domain (which is not written in to simplify the notation).
Differentiating both equations w.r.t. the $\ell$-th coordinate, component by component, we obtain
\begin{equation*}
\begin{split}
&\partial_{x_\ell} J \cdot \nabla f +J\cdot\partial_{x_\ell}\nabla f=\textbf 0,\\
&\partial_{x_\ell} J \cdot \nabla g +J\cdot\partial_{x_\ell}\nabla g=\textbf 0,
\end{split} 
\end{equation*}
from where we can isolate the second derivatives of each feature
\begin{equation*}
\begin{split}
&\partial_{x_\ell}\nabla f= -J^{-1}\cdot\partial_{x_\ell} J \cdot \nabla f,\\
&\partial_{x_\ell}\nabla g= -J^{-1}\cdot\partial_{x_\ell} J \cdot \nabla g.
\end{split} 
\end{equation*}
Now we extract the $j$-th component of both vectors above, to have a formula for $\partial_{x_\ell x_j} f$ and $\partial_{x_\ell x_j} g$. This is,
\begin{equation*}
\begin{split}
&\partial_{x_\ell x_j} f= -\sum_{k,s}(J^{-1})_{jk}\cdot(\partial_{x_\ell} J)_{ks} \cdot \partial_{x_s} f,\\
&\partial_{x_\ell x_j} g= -\sum_{k,s}(J^{-1})_{jk}\cdot(\partial_{x_\ell} J)_{ks} \cdot \partial_{x_s} g.
\end{split} 
\end{equation*}
Next, in order to check Frobenius condition  \eqref{Frobenius-check} we need to calculate
\begin{equation*}
\begin{split}
   \sum_\ell&( \partial_{x_\ell}f \partial_{x_\ell x_j}g-\partial_{x_\ell}g\partial_{x_j x_\ell}f)\\
   &=
-\sum_{\ell,k,s}\Big[\partial_{x_\ell}f\cdot(J^{-1})_{jk}\cdot(\partial_{x_\ell} J)_{ks} \cdot \partial_{x_s} g\\
&\qquad\quad\quad-\partial_{x_\ell}g\cdot(J^{-1})_{jk}\cdot(\partial_{x_\ell} J)_{ks} \cdot \partial_{x_s} f\Big].
\end{split}   
\end{equation*}
If one is able to interchange the indexes $s$ and $\ell$, the above quantity is identically zero. And this is possible since, for a Jacobian matrix, 
\begin{equation*}
\partial_{x_\ell} J_{ks}=\partial_{x_\ell} \partial_{x_s} (\hat \xv_\S)_k=\partial_{x_s} \partial_{x_\ell} (\hat \xv_\S)_k=
\partial_{x_s} J_{k\ell}.
\end{equation*}
We have shown that Frobenius condition \eqref{Frobenius-check} holds for any two features in the set $\S$, which completes the proof.

\subsection{Examples}
\subsubsection{Average of scalar functions}
If the set $\mathcal S$ consists on  features of the form
\[
f_j(\xv) =  \frac{1}{N}\sum_{n=1}^{N}{m_j(x_n)},\quad j=1..M,
\]
for scalar functions $m_j:\mathbb R\to\mathbb R$,
then it clearly satisfies the Frobenius condition. This follows by simple inspection of \eqref{Frobenius-check}, because the Hessians $\partial_{x_\ell x_j}f_i$ are multiples of the identity.

\subsubsection{Second-order moments at the output of a set of filters}
Another example of features that satisfy Frobenius is a set of second-order moments measured at the output of a set of filters. To fix notation, we take two functions
\begin{equation*}
f(\xv)= \frac{1}{N}(\xv*h)^{\odot2},
\quad g(\xv)=\frac{1}{N}(\xv*h')^{\odot2}.
\end{equation*}
Calculate
\begin{equation*}
\begin{split}
&\partial_{x_i}f=
\sum_k (\xv *h)_k h_{k-i}=\frac{2}{N}(\xv *(  h*\tilde h))_i=\frac{2}{N}(\xv * \hat h)_i,\\
&\partial_{x_i x_j}f=\hat h_{i-j},
\end{split}
\end{equation*}
where we have denoted $(\tilde h)_{s}=h_{i-s}$, $\hat h=h*\tilde h$, and similarly for $h'$.
Then, in the notation of \eqref{Frobenius-XY}, 
dropping the multiplicative constants,
\begin{equation*}
\begin{split}
XY-YX&=\sum_{i,j,s} x_s \hat h_{i-s} \hat h'_{i-j}\frac{\partial}{\partial x_j}
-\sum_{i,j,s} x_s \hat h'_{i-s} \hat h_{i-j}\frac{\partial}{\partial x_j}\\
&=
\sum_j \left\{[(\xv * \hat h)*\hat h']_j-[(\xv * \hat h')*\hat h]_j\right\}\frac{\partial}{\partial x_j}
=0,
\end{split}
\end{equation*}
as desired.\\

\section{Adding perturbations}
\label{appendix:crit-perturb}

\subsection{Marginal moments}
\label{subsec:crit_perturb_moments}

\subsubsection{Critical points of the decoupled moments and their basins of attraction}
\label{subsec:critical_points_moments}

In order to apply the methods proposed in Section~\ref{sec:NeNs}, conditions B1 (no local non-global extrema) and B2 (lower dimensional basins for saddles) must hold.
Here we test this, and search for explicit formulas for all critical points of the
new decoupled features.
In particular, this is essential to understand the structure of basins of attraction and  discuss the effectiveness of perturbations.

From the method in Subsection~\ref{subsec:critical_points}, obtaining the critical points of the decoupled features is sort of dual to finding the decoupled gradients within their corresponding reference manifolds. Whereas for the latter we imposed orthogonality on each new gradient with respect to the previous ones (see Subsection~\ref{subsec:higher_order_moments}), for finding the critical points we impose co-linearity on each gradient with respect to the previous ones, see Eq.~\eqref{eq:crit_point_condition_NeNs}.
A particular solution corresponds to finding common  critical points (through order $k$) of the original features, where all gradients for $j\leq k$ vanish. In this case there are no local (non-global) extrema, as gradients are made of monomials, which either have a single minimum at zero, for even orders, or a saddle point at zero and no extrema in $\R$ (odd orders). Therefore, the only solution coming from the original features' gradients vanishing corresponds to $\xv_0^* = {\mathbf 0}$, i.e., at zero all decoupled moments have a critical point, same as the original moments.

For the rest of solutions, we start by solving for the critical points of $\g_2$, denoted by $\xv_1^*$~\footnote{Although for notation simplicity critical points here are denoted as a vector $\xv_j^*$, they actually represent sets of vectors.}, in the equation $\nabla f_2(\xv_1^*) = \lambda_{1,2} f_1(\xv_1^*)$, which, substituting its corresponding expressions, gives us $\xv_1^* = c \v1,\,c\in\R$. No surprisingly, constant signals provide the minimal (zero) variance, the only extreme value of this feature. Moreover, as explained in Subsection~\ref{subsec:critical_points}, decoupled features are not defined at $\xv_1^*$ (being the sample variance null it can not be normalized to one). This manifold only intersects $\RM_1$ in $\xv = {\bf 0}$, and has no intersections with $\RM_2$ and subsequent.
Next, the calculation of the critical points of $\hat{f}_3$ comes from solving for $\xv_2^*$ in $f_3(\xv_2^*) = \lambda_{1,3} f_1(\xv_2^*) + \lambda_{2,3} f_2(\xv_2^*)$. This is a quadratic equation, whose solution is a vector $\xv_2^*$ made of only two arbitrary values, repeated in arbitrary proportions ($p$ and $1-p$), for all the coefficients. If we impose, in addition, $f_1(\xv_2^*)=0$ and $f_2(\xv_2^*)=1$ (conditions of $\xv_2^*\in\RM_2$), we obtain, after some operations, that the only possible values of the coefficients, for $\xv_2^*$ in $\RM_2$, are:
\begin{eqnarray}
x_{2,1}^* & = & \sqrt{p}/(1-p) \nonumber \\
x_{2,2}^* & = & -\sqrt{1-p}/p,
\end{eqnarray}
where $0<p<1$ denotes the proportion of the first value. From this we can easily compute the skewness of $\xv_2^*$, that only depends on $p$, not on the particular values: $\hat{f}_3(\xv_2^*) = (2p-1)/\sqrt{p(1-p)}$. Although the previous expression has no extrema for a continuous $p$, given the discrete nature of $p$ for discrete signals (from $1/N$ to $(N-1)/N$) we obtain that the maximal and minimal skewness are produced when $p$ is either $1/N$ or $(N-1)/N$, that is, when the vector is constant except for one coefficient. The rest of the values of $p$ produce the vast majority of the critical points of $\g_3$, which are saddles. There are no local (non-global) extrema (condition B1).
These critical points have subsequent undefined decoupled moments (see Subsection~\ref{subsec:critical_points}), with the only exception of the case $p=1/2$ (possible just when $N$ is even), the only intersection of $\xv_2^*$ with $\RM_3$, which, having already zero skewness, its (minimal) kurtosis is also its orthokurtosis. One can see that the skewness of a bivaluated vector can not be adjusted by applying a coefficient-wise reversible (monotonous) non-linearity, because the skewness only depends on $p$. 

For $k>3$ we apply the same procedure for obtaining the  critical points of $\g_k$ in $\RM_{k-1}$, namely, solving an algebraic equation of degree $k-1$, and obtaining any combination of distinct $k-1$ solutions for the values in the vector. Same as for the previous case only the $k-2$ distinct proportions ($k-1$ in total, adding up to 1) of each of these $k-1$ different values matter for the computation of $\g_k$. And, again, the vast majority of these critical points are saddles, presenting no local (non-global) extrema (condition B1). For instancce, in the case of $k=4$, critical points have three different distinct values or less, the minimal orthokurtosis point being a degenerate case (having two single values with $p=1/2$, as mentioned above), whereas the maximal orthokurtosis is produced for $p_1 = (N-2)/N$, $p_2 = 1/N$, corresponding to all pixels having the same value, except for two pixels, now having each of these two a different value from the dominant and from each other. These vectors can be adjusted to have the desired mean, variance and skewness, but, having only three distinct values, we run out of degrees of freedom to adjust their orthokurtosis too. Therefore, except for the case when the orthokurthosis is already three (its reference value), $\hat{\xv}_4(\xv_3^*)$ does not exist, and, as a consequence, the fifth order and subsequent decoupled moment are not defined at these points.
Note also that the range of the decoupled features generally change with respect to their original counterparts.
Whereas the second-order moment range does not change when decoupled (it is still $\R^+$), the fourth-order (skewness) is constrained from $\R$  to $[-\frac{N-2}{\sqrt{N-1}},\frac{N-2}{\sqrt{N-1}}]$, and the fourth order (orthokurtosis) from $\R^+$ to $[1,N/2]$. Finally, we note that the existence of higher-order decoupled moments also depends on having a large enough number of samples, $N$, In particular, the last result implies that, for $N<6$, there are no decoupled moment of higher-than-four orders, as the reference value for normalizing the orthokurtosis, 3, is not reachable within its valid range. This limitation also comes from the already explained requirement of having enough distinct values for the samples.

Whereas a typical vector will have more than a few distinct quantization values, and, thus, it will not produce critical points for the first few decoupled moments, a different situation is created by the basins of attraction of the saddles, which, as pointed out above, constitute the vast majority of the critical points.
It is easy to informally check that any vector having coefficients with repeated maximal (or minimal, for odd orders) values lies within the basin of attraction of a saddle.
To illustrate this, let us consider a gray-level image in the range $[0,255]$, with two pixels having the $255$ value. As we increase its fourth-order decoupled moment (the skewness) staying in $\RM_2$, these two values are going to grow at exactly the same pace, relatively to the rest of the coefficients, that, due to the normalization, will get relatively lower and progressively closer to each other. If we keep on increasing the skewness we would approach, in the limit, to an image made of all pixels sharing the same value, except for two pixels, both sharing another value. This mental experiment shows how our original image, having more than one pixel with the maximal value, lies in $\Lambda$, i.e., in the basin of attraction of a saddle. This situation, unless avoided by adding a proper perturbation, provokes the algorithm to eventually get stuck in the saddle, thus not letting the vector to be adjusted along its full range (as shown before, the maximal/minimal skewness is achieved when all samples except for one have the same value). The desired effect of a perturbation is to ``break the tie'', thus allowing the gradient to further advance towards the absolute extrema of the feature.
An analogous reasoning can be applied for higher-order moments.

Then, how likely is that a vector belongs to $\Lambda$? The answer, dealing with vectors having regular density distributions in $\R^N$ is: zero probability (as the basins of attraction of these saddles are lower-than-$N$ dimensional, condition B1). However, the answer for digital (discrete quantized samples) is totally different: potentially {\em very} likely.
Let us assume a vector having $N$ samples quantized in $Q$ levels. Then, the probability of that a particular value $v$ (e.g., the highest) is repeated, assuming a uniform and independent distribution for each of the vector coefficients is given by a binomial distribution:\footnote{This does not pretend to be an accurate estimation of the probability of the repetition of the maximal/minimal value $v$ in a typical signal (e.g., an image). However, it does provide a useful reference value.} $P(n(v)>1) = 1 - (1 - 1/Q)^N - N/Q(1 - 1/Q)^{N-1}$, $n(v)$ being the number of occurrences of $v$. For instance, for a very small size image of $64\times64$ pixels ($N=4096$ - for larger images it gets worse) with pixels ranging from 0 to 255 ($Q=256$), the probability of a given quantization level appearing more than once is extremely high: $1-1.86\times10^{-6}$.
This example shows the importance of adding a proper perturbation to our digital signal $\xv$, as the one proposed in Subsection~\ref{subsec:perturbation}.

\subsubsection{Adding a perturbation}
\label{subsec:perturbation_moments}

In this subsection we propose a perturbation method that
not only ensures that all perturbed coefficient values $x'_i = x_i+\epsilon_i,\,i=1\dots N$ are different (high entropy) but it also maximizes the minimal possible difference between them.   
The latter feature is achieved by using for the perturbation $N$ uniform extra levels within each single quantizing step and assigning a different level to each $\epsilon_i$.
On the other hand, noticeable local oscillations are avoided by choosing a very low frequency pattern for the perturbation, which minimizes its perceptual impact.

Algorithm~\ref{alg:perturbation} explains the process.
First, a random angle tangent is chosen for generating a ramp for the image grid (same method can be straightforwardly generalized to $n$-dimensional grids).
The tangent value must not be a rational $p/q$ number with small $p,q\in\N$, because that would create repeated values on the ramp.
Then, after checking that there are no repeated values on the generated ramp, its cells are sorted in a ranking according to their value, corresponding the number $1$ to the lowest value and $N$ the highest. Finally, these integers are normalized to the interval $[-1/2,1/2)$ and returned as the perturbation. The result is a (slightly curved, sigmoid) ramp made of all different values, having an exactly uniform distribution.

\begin{algorithm}
\begin{algorithmic}[1]
\REQUIRE An empty array of $N_x\times N_y = N$ pixels
\REPEAT
    \STATE Generate a pseudo-random number $r\in [0,1]$
    \STATE Compute a ramp $v(n_x,n_y) = n_x + r*n_y$
    \STATE Check for repeated values in $v$
\UNTIL there are no repeated values in $v$
\STATE Compute $o(n_x,n_y) = rank(v(n_x,n_y))\in \{1\dots N\}$
\STATE $\epsilon(n_x,n_y) = (o(n_x,n_y)-1-N/2)/N \in [-1/2,1/2)$
\RETURN $\epsilon$ ($N_x\times N_y$ array)
\end{algorithmic}
\caption{High-entropy, low-impact perturbation}
\label{alg:perturbation}
\end{algorithm}

\subsection{Second-order moments at the output of a filter bank}
\label{subsec:crit_perturb_filters}

\subsubsection{Critical points, active frequencies and perturbations}
\label{subsec:critical_etc_variance}

Being the sample second-order moment at the output of a filter a positive definite quadratic function, it has no saddle points.
Therefore, in this case, in contrast with Section~\ref{subsec:critical_points_moments}, we do not face the problem of their basins of attraction.

The critical points of the decoupled features are those vectors where the gradients are either zero (critical points inherited from the original features) or co-linear. Substituting Eq. \eqref{eq:gradient_variance} into Eq. \eqref{eq:crit_point_condition_NeNs} it yields:
\[
|H_k(\xiv)|^2 X_k^*(\xiv) = \sum_{j=1}^{k-1} \lambda_{j,k} |H_j(\xiv)|^2 X_k^*(\xiv), \quad\{\lambda_{j,k}\neq 0\}.
\]
This equation always admits the solution $\xv_0^* = {\mathbf 0}$.
In addition, in the case there are regions of the signal spectrum that are not covered by any filter (e.g., signals made of a constant value, for band-pass or high-pass filters), the corresponding vectors having only those frequencies will also be critical points. Apart from the previous solutions, all corresponding to the ``null space'' of the filters' output (which provide the absolute minima of the original features), the equation may only hold if the squared filters are themselves co-linear. Thus, the latter possibility must be prevented -  otherwise {\em all} points in $\bar\Omega$ would be critical!  

We see that, although in this case the decoupled features do not introduce additional critical points, zeros in the frequency domain, both of the signal and of the filters, act as partial ``gradient killers'': signal will not change at those frequencies where kernels are zero, and, similarly, the signal will neither change at those frequencies where the signal itself vanishes.

Whereas in some applications it is normal to ignore some regions of the spectrum that are not useful for a given task, and thus they are left uncovered by the filters bank, it seems advisable, nevertheless: i) to filter out those frequencies (having a support $D$) also in $\xv$ (otherwise those spectral components will remain unchanged in $\xv$), and ii) to introduce a small perturbation $\epsilon$ in $\xv$, such that $\xv+\epsilon$ will not be zero or too small at any frequency in the above defined support $D$. 
Then, a first reasonable concrete choice for the perturbation may be
\begin{equation}
\epsilon = \arg \min_{\zv} \|\zv\|\,s.t.\,|X(\xiv) + Z(\xiv)|\geq \theta, \forall \xiv\in D \nonumber
\end{equation}
which yields, in the Fourier domain, for $\xiv\in D$:
\begin{equation}
E(\xiv) = 
\left\{
\begin{tabular}{ll}
$0$, & \text{if} $|X(\xiv)|>\theta$  \\
$(\theta - |X(\xiv)|)e^{i2\pi\arg(X(\xiv))}$, & \text{if} $\theta>|X(\xiv)|>0$ \\
$\theta e^{i2\pi r}$, & \text{if} $|X(\xiv)|=0$,
\end{tabular}
\right.
\end{equation}
and 0 for $\xiv\not\in D$, where $E(\xiv)={\cal F}(\epsilon)$.
Here $r$ is a uniform random value in $[0,1]$.
The so defined $\epsilon$ is a perturbation of maximal entropy amongst all minimal Euclidean norm ensuring a spectral content above a threshold $\theta$ in $D$.\footnote{A perceptually-based perturbation may be easily obtained from here by considering perceptual metrics/threshold instead.}
The threshold $\theta$ can be chosen as the supreme of the set $\{ \theta: q\left(\xv+\epsilon(\xv,\theta)\right) = q(\xv)\}$, or a similar criterion. Furthermore, depending on the complete set of features, different types of perturbations (see Subsection \ref{subsec:perturbation}) can be fused into a single one fulfilling all the requirements.

\section{Regression experiments}
\label{appendix:reg-exp}
In this section we give some specific details about the regression experiments.
The probability density function of the Generalized Gaussian Distribution (GGD) is given by:

\begin{equation*}
f(x)=\frac{\beta}{2\alpha\Gamma(1/\beta)}e^{-(\left|x-\mu\right|/\alpha)^{\beta}},
\end{equation*}
where $\Gamma$ denotes the gamma function, $\beta$ is the shape parameter, $\alpha$ the scale parameter and $\mu$ the location parameter.
The probability density function of the Gamma Distribution (GMD) is given by:
\begin{equation*}
f(x)=\frac{1}{\Gamma(\beta)\theta^{\beta}}x^{\beta-1}e^{-x/\theta},
\end{equation*}
where $\beta$ is the shape parameter and $\theta$ the scale parameter.

The probability density function of the absolute value of a Normal distribution raised to $\beta$ (GND), $X=|T|^{\beta}$, $T\sim \mathcal{N}(0,\,1)$, is given by, for every positive $\beta$:
\begin{equation*}
f(x)=\frac{2}{\beta \sqrt{2\pi}} x^{(1/\beta)-1}e^{-\frac{1}{2}x^{2/\beta}},
\end{equation*}
where $\beta$ is the shape parameter. Note that, for $\beta=2$, this distribution leads to the chi-squared distribution of one  degree of freedom, $\chi^{2}(1)$.

The RMSE results for the different regression methods and sample sizes $N$ are shown in Tables \ref{tbl:GMD_results} (Generalized Gaussian Distribution distribution, GGD), \ref{tbl:GG_results} (Gamma distribution, GMD) and \ref{tbl:GN_results} (absolute value of a Normal distribution raised to $\beta$, GND). 
The regression methods compared are: linear regression models (LRM), regression trees (RT), support vector regression (SVR), Gaussian process regression (GPR), ensembles of trees (ET) and neural networks (NNR). 
MSM stands for the set of classical marginal standardized moments, and $\text{DF}_{\text{MSM}}$ for its corresponding decoupled set. Highlighted in bold, the regression method that minimizes the RMSE for each $N$ value.

\begin{table}[t]
	\begin{center}
	\caption{\label{tbl:GMD_results} RMSE Results for the different regression methods, Generalized Gaussian distribution (GGD).}

	\begin{tabular}{|c|c||c|c|c|c|c|c|}
	\hline 
	&	$N$				           & LRM & RT & SVR & ET & GPR & NNR\\

\hline
\hline
\multirow{6}{*}{MSM}& 64	   & 0.71 &  0.59  & 0.54 & 0.55 & 0.52 &  \textbf{0.51}   \\         
 \cline{2-8}
& 128    & 0.75 & 0.46 & 0.46 & 0.44 & 0.41 &   \textbf{0.41}       \\ 
 \cline{2-8}
& 256    & 0.85 & 0.36 & 0.44 & 0.35 & 0.34 &   \textbf{0.32}  \\
 \cline{2-8}
& 512    & 0.95 & 0.30 & 0.47 & 0.29 & 0.29 &   \textbf{0.27}  \\
 \cline{2-8}
& 1024    & 1.02 & 0.26 & 0.55 & 0.25 & 0.25 &   \textbf{0.24}  \\
 \cline{2-8}
& 2048    & 1.13 & 0.22 & 0.63 & 0.22 & 0.22 &   \textbf{0.21}  \\
  \hline 
  \hline 
\multirow{6}{*}{$\text{DF}_{\text{MSM}}$}& 64	   & 0.79 &  0.48  & 0.51 & 0.45 & 0.47 &  \textbf{0.44}   \\         
 \cline{2-8}
& 128    & 0.67 & 0.34 & 0.44 & 0.32 & 0.38 &   \textbf{0.31}       \\ 
 \cline{2-8}
& 256    & 0.85 & 0.25 & 0.38 & 0.23 & 0.31 &   \textbf{0.22}  \\
 \cline{2-8}
& 512    & 0.51 & 0.18 & 0.31 & 0.17 & 0.26 &   \textbf{0.17}  \\
 \cline{2-8}
& 1024    & 0.45 & 0.13 & 0.29 & 0.13 & 0.26 &   \textbf{0.13}  \\
 \cline{2-8}
& 2048    & 0.41 & \textbf{0.10} & 0.32 & 0.10 & 0.31 &  0.10 \\
\hline
	\end{tabular}
	\end{center}
	\label{tab:RMSE_GG}
\end{table}
\begin{table}[t]
	\begin{center}
	\caption{\label{tbl:GG_results} RMSE Results for the different regression methods, Gamma distribution (GMD).}

	\begin{tabular}{|c|c||c|c|c|c|c|c|}
	\hline 
	&	$N$				           & LRM & RT & SVR & ET & GPR & NNR\\

\hline
\hline
\multirow{6}{*}{MSM}& 64	   & 0.58 &  0.76  & 0.59 & 0.71 & \textbf{0.50} &  0.52   \\         
 \cline{2-8}
& 128    & 0.45 & 0.59 & 0.46 & 0.56 & \textbf{0.39} &  0.40       \\ 
 \cline{2-8}
& 256    & 0.37 & 0.48 & 0.38 & 0.46 & \textbf{0.31} &  0.31  \\
 \cline{2-8}
& 512    & 0.27 & 0.39 & 0.30 & 0.36 & \textbf{0.23} &   0.24  \\
 \cline{2-8}
& 1024    & 0.21 & 0.29 & 0.22 & 0.28 & 0.21 &   \textbf{0.19}  \\
 \cline{2-8}
& 2048    & 0.18 & 0.22 & 0.18 & 0.21 & 0.15 &   \textbf{0.13}  \\
  \hline 
  \hline 
\multirow{6}{*}{$\text{DF}_{\text{MSM}}$}& 64	   & 0.45 &  0.45  & 0.39 & 0.40 & 0.39 &  \textbf{0.38}   \\         
 \cline{2-8}
& 128    & 0.32 & 0.32 & 0.28 & 0.29 & \textbf{0.27} & 0.28       \\ 
 \cline{2-8}
& 256    & 0.25 & 0.24 & 0.21 & 0.22 & 0.20 &   \textbf{0.20}  \\
 \cline{2-8}
& 512    & 0.19 & 0.18 & 0.16 & 0.15 & \textbf{0.14} &  0.14  \\
 \cline{2-8}
& 1024    & 0.15 & 0.14 & 0.12 & 0.12 & \textbf{0.10} &  0.10  \\
 \cline{2-8}
& 2048    & 0.12 & 0.10 & 0.11 & 0.08 & 0.07 &  \textbf{0.07} \\
\hline
	\end{tabular}
	\end{center}
	\label{tab:RMSE_GammaD}
\end{table}

\begin{table}[t]
	\begin{center}
	\caption{\label{tbl:GN_results} RMSE Results for the different regression methods, absolute value of a Normal raised to $\beta$ (GND).}

	\begin{tabular}{|c|c||c|c|c|c|c|c|}
	\hline 
	&	$N$				           & LRM & RT & SVR & ET & GPR & NNR\\

\hline
\hline
\multirow{6}{*}{MSM}& 64	   & 0.64 &  0.77  & 0.68 & 0.73 & \textbf{0.55} &  0.59   \\         
 \cline{2-8}
& 128    & 0.62 & 0.74 & 0.65 & 0.71 & \textbf{0.53} &  0.56       \\ 
 \cline{2-8}
& 256    & 0.53 & 0.65 & 0.56 & 0.61 & \textbf{0.43} &  0.47  \\
 \cline{2-8}
& 512    & 0.46 & 0.56 & 0.47 & 0.52 & \textbf{0.36} &   0.40  \\
 \cline{2-8}
& 1024    & 0.40 & 0.49 & 0.41 & 0.45 & \textbf{0.31} &   0.34  \\
 \cline{2-8}
& 2048    & 0.36 & 0.41 & 0.35 & 0.38 & 0.34 &   \textbf{0.29}  \\
  \hline 
  \hline 
\multirow{6}{*}{$\text{DF}_{\text{MSM}}$}& 64	   & 0.41 &  0.45  & 0.38 & 0.39 & \textbf{0.37} &  0.38   \\         
 \cline{2-8}
& 128    & 0.37 & 0.40 & 0.34 & 0.35 & \textbf{0.34} & 0.34       \\ 
 \cline{2-8}
& 256    & 0.27 & 0.31 & 0.25 & 0.26 & \textbf{0.25} &  0.25  \\
 \cline{2-8}
& 512    & 0.20 & 0.21 & 0.18 & 0.18 & \textbf{0.17} &  0.18  \\
 \cline{2-8}
& 1024    & 0.16 & 0.16 & 0.14 & 0.14 & \textbf{0.13} &  0.13  \\
 \cline{2-8}
& 2048    & 0.13 & 0.12 & 0.11 & 0.10 & \textbf{0.10} &  0.10 \\
\hline
	\end{tabular}
	\end{center}
	\label{tab:RMSE_GaussN}
\end{table}

\ifCLASSOPTIONcompsoc
  \section*{Acknowledgments}
\else
  \section*{Acknowledgment}
\fi

The authors would like to thank Matteo Bonforte, Rafael Molina, Ivan Selesnick, Gustau Camps-Valls and Eero Simoncelli for fruitful discussions.

\ifCLASSOPTIONcaptionsoff
  \newpage
\fi

\bibliographystyle{IEEEtran}
\bibliography{Decoupled_2022}

\end{document}